\pdfoutput=1
\documentclass[11pt]{article}
\usepackage{times}
\usepackage[letterpaper, left=1in, right=1in, top=1in, bottom=1in]{geometry}

\usepackage{amsmath, amsthm, amssymb, amsfonts, mathtools, graphicx, enumerate}

\usepackage[utf8]{inputenc} 
\usepackage[T1]{fontenc}    
\usepackage{url}            
\usepackage{booktabs}
\usepackage{nicefrac}       
\usepackage{microtype}      
\usepackage{mkolar_definitions}
\usepackage{xspace}
\usepackage{algorithm,algorithmic}
\usepackage{color}
\usepackage{enumitem}
\usepackage{comment}
\usepackage{bm}
\usepackage{apptools}
\usepackage[page, header]{appendix}
\usepackage{titletoc}
\usepackage{array}
\usepackage{multirow}

\usepackage[scaled=.9]{helvet}

\usepackage{url}
\usepackage{caption}
\usepackage[table,dvipsnames]{xcolor}
\usepackage[colorlinks=true, linkcolor=blue, citecolor=blue]{hyperref}

\usepackage{authblk}
\title{Model-free Representation Learning and Exploration in Low-rank MDPs}
\date{}
\author{Aditya Modi$^{\dagger,1}$}
\author{Jinglin Chen$^{\dagger,2}$}
\author{Akshay Krishnamurthy$^3$}
\author{Nan Jiang$^2$}
\author{Alekh Agarwal$^4$\thanks{$^\dagger$equal contribution.$^\ddagger$Part of this work was done while AA was at Microsoft Research.\\admodi@umich.edu, jinglinc@illinois.edu, akshaykr@microsoft.com, nanjiang@illinois.edu, alekha@microsoft.com}}
\affil{$^1$Microsoft, $^2$University of Illinois at Urbana-Champaign, $^3$Microsoft Research, $^4$Google Research$^\ddagger$}

\usepackage{natbib}
\bibpunct{(}{)}{;}{a}{,}{,}

\newcommand{\alg}{\textsc{Moffle}\xspace}
\newcommand{\flambe}{\textsc{Flambe}\xspace}
\newcommand{\olive}{\textsc{Olive}\xspace}
\newcommand{\expl}{\textsc{Explore}\xspace}
\newcommand{\fqi}{\textsc{FQI}\xspace}
\newcommand{\fqe}{\textsc{FQE}\xspace}
\newcommand{\R}{\mathbb{R}}

\newcommand{\Ndim}{\mathrm{Ndim}}
\newcommand{\Pdim}{\mathrm{Pdim}}
\newcommand{\VCdim}{\mathrm{VCdim}}

\newcommand{\veps}{\varepsilon}

\newcommand{\unif}{\mathrm{unif}}
\newcommand{\etamin}{\eta_{\mathrm{min}}}

\newcommand{\be}[1]{\mathrm{b\_err}\rbr{#1}}

\newcommand{\phih}{\phi_h(x_h,a_h)}
\newcommand{\hphih}{\hat{\phi}_h(x_h,a_h)}
\newcommand{\sphih}{\phi^*_h(x_h,a_h)}

\newcommand{\hphitilH}{\hat \phi_{\tilde H-1}(x_{\tilde H-1},a_{\tilde H-1})}

\newcommand{\rhoxa}{\rho_{h-3}^{+3}} 
\newcommand{\rhox}{\rho_{h-3}^{+2}} 
\newcommand{\rhoxnxt}{\rho_{h-2}^{+2}} 

\newcommand{\poly}{\mathrm{poly}}

\newcommand{\defeq}{:=}

\newcommand{\ca}{c_1}
\newcommand{\cb}{c_2}
\newcommand{\cc}{c_3}
\newcommand{\cd}{c_4}
\newcommand{\ce}{c_5}
\newcommand{\cf}{c_6}

\newcommand{\repbar}{\bar \phi}
\newcommand{\rephat}{\hat \phi}
\newcommand{\downstream}{\mathrm{plan}}

\newtheorem*{theorem*}{Theorem}
\newtheorem*{lemma*}{Lemma}
\newtheorem*{corollary*}{Corollary}

\usepackage[capitalise,nameinlink]{cleveref}
\Crefname{equation}{Eq.}{Eqs.}
\Crefname{assumption}{Assumption}{Assumptions}
\Crefname{condition}{Condition}{Conditions}

\newcount\Comments  
\definecolor{darkred}{rgb}{0.7,0,0}
\definecolor{darkgreen}{rgb}{0,0.5,0}
\definecolor{orange}{rgb}{0.7,0.4,0}
\definecolor{purple}{rgb}{0.8,0.0,0.8}
\newcommand{\kibitz}[2]{\ifnum\Comments=1{\textcolor{#1}{\textsf{\footnotesize #2}}}\fi}

\begin{document}
\maketitle

\begin{abstract}
    The low-rank MDP has emerged as an important model for studying representation learning and exploration in reinforcement learning.  With a known representation, several model-free exploration strategies exist. In contrast, all algorithms for the unknown representation setting are model-based, thereby requiring the ability to model the full dynamics. In this work, we present the first model-free representation learning algorithms for low-rank MDPs. The key algorithmic contribution is a new minimax representation learning objective, for which we provide variants with differing tradeoffs in their statistical and computational properties. We interleave this representation learning step with an exploration strategy to cover the state space in a reward-free manner. The resulting algorithms are provably sample efficient and can accommodate general function approximation to scale to complex environments.
\end{abstract}

\section{Introduction}

A key driver of recent empirical successes in machine learning is the
use of rich function classes for discovering transformations of
complex data, a sub-task referred to as \emph{representation
  learning}. For example, when working with images or text, it is
standard to train extremely large neural networks in a self-supervised
fashion on large datasets, and then fine-tune the network on
supervised tasks of interest. The representation learned in the first
stage is essential for sample-efficient generalization on the
supervised tasks. Can we endow Reinforcement Learning (RL) agents with
a similar capability to discover representations that provably enable
sample efficient learning in downstream tasks?

In the empirical RL literature, representation learning often occurs
implicitly simply through the use of deep neural networks, for example
in DQN~\citep{mnih2015human}. Recent work has also
considered more explicit representation learning via auxiliary losses
like inverse dynamics~\citep{pathak2017curiosity}, the use of explicit
latent state space
models~\citep{hafner2019learning,sekar2020planning}, and via
bisimulation
metrics~\citep{gelada2019deepmdp,zhang2020learning}. Crucially, these
explicit representations are again often trained in a way that they
can be reused across a variety of related tasks, such as domains
sharing the same (latent state) dynamics but differing in reward
functions.

While these works demonstrate the value of representation
learning in RL, theoretical understanding of such approaches is
limited. Indeed obtaining sample complexity guarantees is quite subtle
as recent lower bounds demonstrate that various representations are
not useful or not
learnable~\citep{modi2019sample,du2019good,van2019comments,lattimore2019learning,hao2021online}.
Despite these lower bounds, some prior theoretical works do provide
sample complexity guarantees for non-linear function approximation~\citep{jiang2017contextual,sun2018model,osband2014model,wang2020provably,yang2020bridging},
but these approaches do not obviously enable generalization to related tasks. 
More direct representation learning approaches were recently
studied
in~\citet{du2019provably,misra2020kinematic,agarwal2020flambe}, who
develop algorithms that provably enable sample efficient learning in any downstream task that
shares the same dynamics.

Our work builds on the most general of the direct representation
learning approaches, namely the \flambe algorithm
of~\citet{agarwal2020flambe}, that finds features under which the
transition dynamics are nearly linear.  The main limitation of \flambe
is the assumption that the dynamics can be described in a parametric
fashion. 
In contrast, we take a model-free approach to this problem, thereby
accommodating much richer 
dynamics. 

Concretely, we study the low-rank MDP,  
in which the transition operator $T: (x,a)
\to \Delta(\Xcal)$ admits a low-rank factorization as $T(x' \mid x,a)
= \inner{\phi^*(x,a)}{\mu^*(x')}$ for feature maps $\phi^*,\mu^*$. For
model-free representation learning, we assume access to a function
class $\Phi$ containing the underlying feature map $\phi^*$. This is a
much weaker inductive bias than prior work in the ``known
features'' setting where $\phi^*$ is known in
advance~\citep{jin2019provably,yang2019reinforcement,agarwal2020pc}
and the model-based setting~\citep{agarwal2020flambe} that assumes
realizability for both $\mu^*$ and $\phi^*$.

While our model-free setting captures richer MDP models, addressing
the intertwined goals of representation learning and exploration is
much more challenging. In particular, the forward and inverse dynamics
prediction problems used in prior works are no longer admissible under
our weak assumptions. 
Instead, we address these challenges with a new representation learning
procedure based on the following insight: for any function $f:\Xcal
\to \RR$, the Bellman backup of $f$ is a linear function in the
feature map $\phi^*$. This leads to a natural minimax objective, where
we search for a representation $\hat{\phi}$ that can linearly
approximate the Bellman backup of all functions in some
``discriminator'' class $\Fcal$. Importantly, the discriminator class
$\Fcal$ is induced directly by the class $\Phi$, so no additional
realizability assumptions are required.  We also provide an
incremental approach for expanding the discriminator set, which leads
to a more computationally practical variant of our algorithm. 
The two algorithms reduce to minimax optimization problems over non-linear function classes.  While such problems can be solved empirically with modern deep learning libraries, they do not come with rigorous computational guarantees. To this end, we further show that when $\Phi$ is efficiently enumerable, our optimization problems can be reduced to eigenvector computations, which leads to provable computational efficiency.\footnote{For the enumerable case, our algorithm collects an exploratory dataset, which is sufficient for downstream planning but requires the entire feature class $\Phi$. See \pref{sec:enumerable} for more details.}  

\paragraph{Summary of contributions} 
Our main contributions and organization of the rest of this paper is summarized below:
\begin{itemize}[leftmargin=*, itemsep=0pt]
    \item In \pref{sec:problem_setting}, we formally describe the problem setting of this paper. The related work and comparison with existing literature is discussed in \pref{sec:related_literature}.
    \item In \pref{sec:algorithm}, we present our main algorithm \alg which interleaves the exploration and representation learning components. We further describe how the learned representation can be used for planning in downstream tasks by using a standard offline planning algorithm, namely, \fqi.
    \item In \pref{sec:min_max_min_rep_learn}, we present our novel representation learning objective for low-rank MDPs, a min-max-min optimization problem defined using the feature class $\Phi$. A sample complexity result is then presented for \alg under a min-max-min computational oracle assumption.
    \item To address the computational tractability of our representation learning objective, we propose a computationally friendly iterative greedy approach in \pref{sec:greedy_selection}. We state a formal guarantee on the iteration complexity of this approach and show that the resulting instance of \alg is provably sample efficient.
    \item In \pref{sec:enumerable}, we show that for the special case of enumerable feature classes, \alg can be used for sample efficient reward-free exploration and that the main representation learning objective can be reduced to a computationally tractable eigenvector computation problem.
    \item In \pref{sec:proof_main}, we give a proof outline of our main results along with the complete proofs for each instantiation of \alg and, finally, conclude in \pref{sec:conclusion}.
    \item The following supporting results are delegated to the appendix thereafter: (i) details and guarantees for an elliptical planning algorithm (\pref{app:FQI_ellip_planning}), (ii) deviation bounds for our main results (\pref{app:supp_dev_bounds_main}), (iii) sample complexity results for \fqi planning and \fqe methods (\pref{app:fqi} and \pref{app:fqe}), and (iv) auxiliary results e.g., deviation bounds for regression with squared loss (\pref{app:aux_lemmas}). 
\end{itemize}

\section{Problem Setting}
\label{sec:problem_setting}
We consider an episodic MDP $\Mcal$ with a state space $\Xcal$, a finite action space $\Acal = \{1,\ldots,K\}$ and horizon $H$. In each episode, an agent generates a trajectory $\tau = (x_0,a_0,x_1,\ldots, x_{H-1},$ $a_{H-1},x_H)$, where (i) $x_0$ is a starting state drawn from some initial distribution, (ii) $x_{h+1} \sim T_h(\cdot \mid x_h,a_h)$, and (iii) the actions are chosen by the agent according to some non-stationary policy $a_h \sim \pi(\cdot \mid x_h)$. Here, $T_h$ denotes the (possibly non-stationary) transition dynamics $T_h: \Xcal \times \Acal \rightarrow \Delta(\Xcal)$ for each timestep. For notation, $\pi_h$ denotes an $h$-step policy that chooses actions $a_0,\ldots,a_h$. We also use $\EE_{\pi}[\cdot]$ and $\PP_{\pi}[\cdot]$ to denote the expectations over states and actions and probability of an event respectively, when using policy $\pi$ in $\Mcal$. Further, we use $[H]$ to denote $\{0,1,\ldots, H-1\}$. 

We consider learning in a low-rank MDP defined as:

\begin{definition}
\label{def:lowrank}
An operator $T: \Xcal \times \Acal \rightarrow \Delta(\Xcal)$ admits a low-rank decomposition of dimension $d$ if there exists functions $\phi^*: \Xcal \times \Acal \rightarrow \R^d$ and $\mu^*: \Xcal \rightarrow \R^d$ such that: $\forall x, x' \in \Xcal , a \in \Acal: \, T(x'\mid x,a) = \inner{ \phi^*(x,a)}{\mu^*(x')}$,
and additionally $\|\phi^*(x,a)\|_2 \le 1$ and for all $g: \Xcal \rightarrow [0,1]$, $\left\|\int g(x) \mu^*(x) dx\right\|_2 \le \sqrt{d}$.
We assume that $\Mcal$ is low-rank with embedding dimension $d$, i.e., for each $h \in [H]$, the transition operator $T_h$ admits a rank-$d$ decomposition. 
\end{definition}

We denote the embedding for $T_h$ by $\phi^*_h$ and $\mu^*_h$. In addition to the low-rank representation, we also consider a latent variable representation of $\Mcal$, as defined in~\citet{agarwal2020flambe}, as follows: 
\begin{definition}
\label{def:latent_var}
The latent variable representation of a transition operator $T: \Xcal \times \Acal \rightarrow \Delta(\Xcal)$ is a latent space $\Zcal$ along with functions $\psi: \Xcal \times \Acal \rightarrow \Delta(\Zcal)$ and $\nu: \Zcal \rightarrow \Delta(\Xcal)$, such that $T(\cdot \mid x,a) = \int \nu(\cdot \mid z)\psi(z \mid x,a)dz$. The latent variable dimension of $T$, denoted $d_{\mathrm{LV}}$ is the cardinality of smallest latent space $\Zcal$ for which $T$ admits a latent variable representation. In other words, this representation gives a non-negative factorization of  $T$.
\end{definition}
When state space $\Xcal$ is finite, all transition operators $T_h(\cdot \mid x,a)$ admit a trivial latent variable representation. More generally, the latent variable representation enables us to augment the trajectory $\tau$ as: $\tau = \{x_0, a_0, z_1, x_1, \ldots, z_{H-1}, x_{H-1}, a_{H-1}, z_H, x_H\}$, where $z_{h+1} \sim \psi_h(\cdot\mid x_{h}, a_{h})$ and $x_{h+1} \sim \nu_{h}(\cdot \mid z_{h+1})$. In general we neither assume access
to nor do we learn this representation, and it is solely used to reason about the following reachability assumption:
\begin{assum}[Reachability]
\label{assum:reachability}
There exists a constant $\etamin> 0$, such that $\forall h \in [H],  z \in \Zcal_{h+1}:\, \max_\pi \PP_\pi \sbr{z_{h+1} = z} \ge \etamin$.
\end{assum}
\pref{assum:reachability} posits that in MDP $\Mcal$, for each factor (latent variable) at any level $h$, there exists a policy which reaches it with a non-trivial probability. This generalizes the reachability of latent states assumption from prior block MDP results~\citep{du2019provably,misra2020kinematic}.
Note that, exploring all latent states is still non-trivial, as a policy which chooses actions uniformly at random may hit these latent states with an exponentially small probability.

\paragraph{Representation learning in low-rank MDPs} 
We consider MDPs where the state space $\Xcal$ is large and the agent must employ function approximation to enable efficient learning. Given the low-rank MDP assumption, we grant the agent access to a class of representation functions mapping a state-action pair $(x,a)$ to a $d$-dimensional embedding. Specifically, the feature class is $\Phi=\bigcup_{h\in[H]}\Phi_h$, where each mapping $\phi_h \in \Phi_h$ is a function $\phi_h: \Xcal \times \Acal \rightarrow \R^d$. The feature class can now be used to learn $\phi^*$ and exploit the low-rank decomposition for efficient learning\footnote{Sometimes we drop $h$ in the subscript for brevity.}. We assume that our feature class $\Phi$ is rich enough:
\begin{assum}[Realizability]
\label{assum:realizability}
For each $h \in [H]$, we have $\phi^*_h \in \Phi_h$. Further, we assume that $\forall \phi_h \in \Phi_h, \forall (x,a) \in \Xcal \times \Acal$, $\|\phi_h(x,a)\|_2 \le 1$.
\end{assum}
\paragraph{Learning goal}
We focus on the problem of representation learning \citep{agarwal2020flambe} in low-rank MDPs where the agent tries to learn good enough features and collect a suitable dataset 
that enables offline optimization of any given reward in downstream tasks instead of optimizing a fixed and explicit reward signal. We consider a model-free setting and we provide this \emph{reward-free} learning guarantee for any reward function $R=R_{0:H-1}$ ($R_{0:H-1}\coloneqq \{R_0,\ldots R_{H-1}\}$ with $R_h:\Xcal\times\Acal\rightarrow [0,1],\forall h\in[H]$) in a bounded reward class $\Rcal$.\footnote{The subscript $i$:$j$ for any $i\le j$ is also used similarly in other variables.} Specifically, for such a bounded reward function $R$, the learned features $\{\bar \phi_h\}_{h \in [H]}$ and the collected data should allow the agent to compute a near-optimal policy $\pi_R$, such that \mbox{$v_R^{\pi_R} \ge v_R^* - \veps$}, where $v_R^\pi := \EE_\pi\sbr{\sum_{h=0}^{H-1} R_h(x_h,a_h)}$ is the expected return of policy $\pi$ under reward function $R$, and $v_R^*:= \max_\pi v_R^\pi$ is the optimal expected return for $R$. 
We desire (w.p.~$\ge 1-\delta$) sample complexity bounds which are \[\poly (d,H,K,1/\etamin, 1/\veps, \log (|\Phi|), \log (|\Rcal|), \log (1/\delta)).\] 

For the simplicity of presentation, we consider finite $\Phi$ and $\Rcal$ classes. The results can be straightforwardly extended to infinite function classes by applying standard tools in statistical learning theory \citep{natarajan1989learning,pollard2012convergence,devroye2013probabilistic}.

\section{Related Literature}
\label{sec:related_literature}
Much recent attention has been devoted to
linear function
approximation~\citep[c.f.,][]{jin2019provably,yang2019reinforcement}. These
results provide important building blocks for our work. In particular,
the low-rank MDP model we study is from~\citet{jin2019provably} who
assume that the feature map $\phi^*$ is known in advance. However, as we
are focused on nonlinear function approximation, it is more apt to
compare to related nonlinear approaches, which can be categorized in
terms of their dependence on the 
size of the
function class:

\paragraph{Polynomial in $|\Phi|$ approaches}
Many approaches, while not designed explicitly for our setting, can
yield sample complexity scaling polynomially with $|\Phi|$ in our setup. Note, however, that
polynomial-in-$|\Phi|$ scaling can be straightforwardly obtained by
concatenating all of candidate feature maps and running the algorithm
of~\citet{jin2019provably}. Further, this is the only obvious way to apply Eluder dimension results here~\citep{osband2014model,wang2020provably,ayoub2020model}, 
and it also pertains to
work on model selection~\citep{pacchiano2020regret,lee2021online}.  Indeed, the key observation that enables a
logarithmic-in-$|\Phi|$ sample complexity is that all value function
are in fact represented as \emph{sparse} linear functions of this
concatenated feature map.

However, exploiting sparsity in RL (and in contextual bandits) is
quite subtle. In both settings, it is not possible to obtain results
scaling logarithmically in both the ambient dimension \emph{and} the
number of actions~\citep{lattimore2020bandit,hao2021online}. That said,
it is possible to obtain results scaling polynomially with the number
of actions and logarithmically with the ambient dimension, as we do
here.

\paragraph{Logarithmic in $|\Phi|$ approaches} 
For logarithmic-in-$|\Phi|$ approaches, the assumptions and results
vary considerably. Several results focus on the block MDP
setting~\citep{du2019provably,misra2020kinematic,foster2020instance},
where the dynamics are governed by a discrete latent state space, 
which is decodable from the observations. This setting is a special
case of our low-rank MDP setting. 
Additionally, these works make stronger function approximation assumptions than we do.
As such our work can be seen as
generalizing and relaxing assumptions, when compared with existing
block MDP results.

Most closely related to our work are \olive \citep{jiang2017contextual}, \textsc{Witness Rank}\xspace\citep{sun2019model}, \flambe \citep{agarwal2020flambe}, and \textsc{BLin-Ucb}\xspace\citep{du2021bilinear}
algorithms. 
\olive is a
model-free RL algorithm that can be instantiated to produce a
logarithmic-in-$|\Phi|$ sample complexity guarantee in the reward-aware (single reward) low-rank MDP
setting (it also applies more generally). However, it is not
computationally efficient even in tabular
settings~\citep{dann2018oracle}. Like \olive, \textsc{Witness Rank}\xspace and \textsc{BLin-Ucb}\xspace are also statistically efficient in more general settings but computationally intractable in similar ways. \textsc{Witness Rank}\xspace is a model-based algorithm and can handle our setting given a stronger function approximation assumption.  \textsc{BLin-Ucb}\xspace works with a general hypothesis class, which can be either model-free or model-based and generalize the previous two approaches. 
All these three algorithms are restricted to the reward-aware setting, but do not require the reachability assumption as we do. 

\flambe is computationally efficient with Maximum Likelihood Estimation (MLE) and sampling oracles, but it is model-based, so the
function approximation assumptions are stronger than ours. Thus the
key challenge, as well as the main advancement, is our weaker model-free function approximation
assumption which does not allow modeling $\mu^*$ (and hence the MDP dynamics) whatsoever. On the other hand,
\flambe does not require the reachability assumption when two computational oracles are available. However, in the absence of the sampling oracle for the estimated model, it does require the reachability assumption to establish theoretical guarantees. We
leave fully eliminating the reachability assumption with a computationally efficient model-free approach
as an important open problem. 

Our proposed algorithms address the reward-free learning goal with differing tradeoffs in their statistical and computational properties. \alg is computationally efficient with the min-max-min oracle (\Cref{eq:emp_minmax_obj}) or the oracles for squared loss minimization and a saddle-point formulation (\pref{alg:greedy_selection}). For the special case of enumerable feature class, 
our optimization problems further reduce to a fully computationally tractable eigenvector computation problem (\Cref{eq:ev-enum}); note that even in this special case, the computation of \textsc{Olive}\xspace, \textsc{Witness rank}\xspace, and \textsc{BLin-Ucb}\xspace is still inefficient as they need to enumerate over infinite function classes (see \pref{app:comp} for more details). In addition, follow-up empirical evaluations \citep{zhang2022efficient} have also confirmed the practical feasibility of the optimization oracle (\pref{alg:greedy_selection}) required by our algorithm.

In \Cref{table:comparison_table}, we present a more detailed comparison between our algorithms and the closely related ones in the low-rank MDP setting. For comparisons among algorithms that tackle the more restricted block MDPs setting, we refer the reader to \citet{zhang2022efficient}. 
The details of how we instantiate \textsc{Olive}\xspace, \textsc{Witness rank}\xspace, and \textsc{BLin-Ucb}\xspace in our setting can be found in \pref{app:comp}.

\begin{table}[htb]
\centering\resizebox{\columnwidth}{!}{
 \begin{tabular}{|c|c|c|c|c|}
 \hline
 Algorithm & R-F? & Realizability & Sample Complexity & Computation \\  
  \hline\hline
 \textsc{Olive}\xspace  & \multirow{2}*{No} & \multirow{2}*{$\phi^*\in\Phi$} &\multirow{2}*{$\frac{d^3 KH^5\log(|\Phi|/\delta)}{ \varepsilon^2  }$}  & \multirow{2}*{Enumeration over the value class}  \\ 
 \citep{jiang2017contextual} & ~ &~ & ~ & ~\\
 \hline
  \textsc{Witness rank}\xspace  & \multirow{2}*{No} & $\phi^*\in\Phi$ & \multirow{2}*{$\frac{d^3 KH^5\log(|\Phi||\Upsilon|/\delta)}{ \varepsilon^2 }$}  & \multirow{2}*{Enumeration over the model class}  \\ 
 \citep{sun2019model} & ~ & $\mu^*\in\Upsilon$ & ~ & ~\\
 \hline
\textsc{BLin-Ucb}\xspace  & \multirow{2}*{No}  & \multirow{2}*{$\phi^*\in\Phi$} & \multirow{2}*{$\frac{d^3 KH^7\log(|\Phi|/\delta)}{ \varepsilon^2}$}  & \multirow{2}*{Enumeration over the hypothesis class}  \\ 
\citep{du2021bilinear} & ~ & ~ & ~ & ~\\
 \hline
\flambe\xspace  & \multirow{2}*{Yes} & $\phi^*\in\Phi$ & \multirow{2}*{$\frac{d^7 K^9H^{22}\log(|\Phi||\Upsilon|/\delta)}{\varepsilon^{10}}$}  & \multirow{2}*{MLE oracle + sampling oracle}  \\ 
\citep{agarwal2020flambe}& ~ & $\mu^*\in\Upsilon$ & ~ & ~\\
 \hline
 \textsc{Rep-Ucb}\xspace  & \multirow{2}*{No} & $\phi^*\in\Phi$ & \multirow{2}*{$\frac{d^4 K^2H^5\log(|\Phi||\Upsilon|/\delta)}{\varepsilon^2 }$}  & \multirow{2}*{MLE oracle + sampling oracle}  \\ 
\citep{uehara2021representation}& ~ & $\mu^*\in\Upsilon$ & ~ & ~\\
\hline
\rowcolor{Gray!10} ~ & ~ & ~ & ~ & ~  \\
 \rowcolor{Gray!10} \multirow{-2}*{\alg{} (Ours)} & \multirow{-2}*{Yes} & \multirow{-2}*{$\phi^*\in\Phi$} & \multirow{-2}*{$\frac{d^{11} K^{14}H^7\log(|\Phi||\Rcal|/\delta) }{\min\{\varepsilon^2\etamin,\etamin^5\}}$}  & \multirow{-2}*{Min-max-min oracle (\Cref{eq:emp_minmax_obj})} \\
\hline
\rowcolor{Gray!10} ~ & ~ & ~ & ~ & Squared loss minimization + \\
\rowcolor{Gray!10} \multirow{-2}*{\alg{} (Ours)} & \multirow{-2}*{Yes} & \multirow{-2}*{$\phi^*\in\Phi$} & \multirow{-2}*{$\frac{d^{19} K^{32}H^{19}\log(|\Phi||\Rcal|/\delta) }{\min\{\varepsilon^6\etamin^3,\etamin^{11}\}}$}  & saddle-point formulation (\pref{alg:greedy_selection}) \\
\hline
\rowcolor{Gray!10} \expl{} (Ours) + & ~ & ~ & ~ & Enumeration over $\Phi$ +  \\
\rowcolor{Gray!10} \fqi & \multirow{-2}*{Yes} & \multirow{-2}*{$\phi^*\in\Phi$} & \multirow{-2}*{$\frac{d^{25}K^{50}H^7 \log^5 (|\Phi||\Rcal|/\delta)}{\min\{\veps^2\etamin,\etamin^{17}\}}$}  & eigenvector computation (\Cref{eq:ev-enum})\\
\hline
\end{tabular}}
\caption{Comparisons among algorithms for low-rank MDPs (with unknown features). R-F column refers to whether the algorithm can handle reward-free learning. $\Upsilon$ is the additional candidate feature class used in \citet{sun2019model}, \citet{agarwal2020flambe} and \citet{uehara2021representation} to capture the model-based relizability. For the sample complexity, we only show the orders and hide $\text{polylog}$ terms (i.e., using $\tilde O(\cdot)$ notation). Since the sample complexity bounds for our proposed algorithms are too long, we only show their simplified versions here. We convert the $1/(1-\gamma)$ horizon dependence in \textsc{Rep-Ucb} \citep{uehara2021representation} to $H$. See the text for more discussions on realizability and computation.}
 \label{table:comparison_table}
\end{table}

\paragraph{Related algorithmic approaches}
Central to our approach is the idea of embedding plausible futures into a ``discriminator'' class and using this class to guide the learning process. \citet{bellemare2019avf} also propose a min-max representation learning objective using a class of \emph{adversarial value functions}, but their work only empirically demonstrates its usefulness as an auxiliary task during learning and does not study exploration. 
Similar ideas of using a discriminator class have been deployed in model-based RL \citep{farahmand2017value, sun2018model, modi2019sample, ayoub2020model}, but the application to model-free representation learning and exploration is novel to our knowledge. 

\paragraph{Subsequent works} 
After the initial version of our work was released, there have been several follow-up papers that also investigate representation learning. Similar to \flambe, \citet{uehara2021representation} develop a model-based representation-learning algorithm in low-rank MDPs, which is computationally efficient with MLE and sampling oracles. Their algorithm \textsc{Rep-Ucb}\xspace requires stronger function approximation assumptions (model realizability) and does not come with reward-free learning guarantees. On the other hand, \textsc{Rep-Ucb}\xspace does not require the reachability assumption and has a sharper rate. \citet{ren2021free} also proposes a model-based representation-learning algorithm, which exploits the noise assumption in the stochastic control model. Their setting has a low-rank structure but is different from our low-rank MDP setting.

Another recent work \citep{zhang2022efficient} builds on our representation learning oracle and analysis. They present various experimental results, which can be regarded as empirical evidence and support that the representation learning oracle proposed in our work is empirically tractable and effective. As for theoretical guarantees, their results are restricted to the reward-dependent block MDP setting---which is more restrictive than low-rank MDPs---but they do not need the reachability assumption. 
It is still an open question whether such an assumption can be removed in the computationally tractable model-free reward-free setting (see the discussion on a similar setting in \citet{zanette2020provably}). 

Orthogonal to our work, \citet{zhang2021provably} and \citet{papini2021reinforcement} assume a candidate set of ``correct'' representations is given and propose algorithms to select the ``good'' one (in a certain technical sense) from this candidate set. In contrast, our function-approximation assumption is much weaker: we only require \emph{one} ``correct'' representation (in their terminology) to lie in the candidate set.

Finally, the recent work \citep{huang2021towards} studies deployment-efficient RL in linear MDPs, where the goal is to minimize the number of policy changes (``deployment complexity'') for real-world deployment considerations. 
Our algorithm \alg only requires $H$ deployments, thus matching the optimal $\tilde \Omega(H)$ deployment complexity up to polylog terms and in a strictly more general setting.\footnote{Our earlier version of \alg employs the online ``elliptical planner'', which needs $\tilde O(Hd^3 K^4/\etamin^2)$ deployments overall. Inspired by the techniques in \citet{huang2021towards}, we integrate the more advanced offline ``elliptical planner'' to \alg and achieve the optimal deployment complexity.} 

\section{Main Algorithmic Framework}
\label{sec:algorithm}
In this section, we describe the overall algorithmic framework that we propose for representation learning for low-rank MDPs. A key component in this general framework, which specifies how to learn a good representation once exploratory data (at previous levels) has been acquired, is left unspecified in this section and instantiated with two different choices in the subsequent sections. We also present sample complexity guarantees for each choice in the corresponding sections. 

At the core of our model-free approach is the following well-known property of a low-rank MDP due to~\citet{jin2019provably}. We provide a proof in \pref{app:pf_lem1} for completeness. 
\begin{lemma}[\citet{jin2019provably}]
\label{lem:linMDP_expn}
For a low-rank MDP $\Mcal$ with embedding dimension $d$, for any $f : \Xcal \rightarrow [0,1]$, we have: $\EE \sbr{f(x_{h+1})\mid x_h,a_h } = \langle \sphih,\theta^*_f\rangle$, where $\theta^*_f \in \R^d$ and $\|\theta^*_f\|_2 \le \sqrt{d}$.
\end{lemma}
We turn this property into an algorithm by finding a feature map in the candidate class $\Phi$ which can certify this condition for a sufficiently rich class of functions $\Fcal$. The key insight in our algorithm is that this property depends solely on $\phi^*$, so we
do not require additional modeling assumptions.

Before turning to the algorithm description, we clarify a useful notation for $h$ step policies. An $h$-step policy\footnote{We use $\pi_h$ to denote a standard (single) policy and $\rho_h$ to denote an (exploratory) mixture policy, i.e., uniformly sampling from a set of policies.} $\rho_h$ chooses actions $a_0,\ldots,a_h$, consequently inducing a distribution over $(x_h,a_h,x_{h+1})$. We routinely append several random actions to such a policy, and we use $\rho_h^{+i}$ to denote the policy that chooses $a_{0:h}$ according to $\rho_h$ and then takes actions uniformly at random for $i$ steps, inducing a distribution over $(x_{h+i}, a_{h+i}, x_{h+i+1})$. As an edge case, for $i \geq j \geq 0$, $\rho_{-j}^{+i}$ takes actions $a_0,\ldots,a_{i-j}$ uniformly. The mnemonic is that the last action taken by $\rho_{j}^{+i}$ is $a_{i+j}$.

Our algorithm, Model-Free Feature Learning and Exploration (\alg) shown in \pref{alg:main_alg}, takes as input a feature set $\Phi$, a reward class $\Rcal$, the reachability coefficient $\etamin$ (\pref{assum:reachability}), the sub-optimality parameter $\varepsilon$, and the high-probability parameter $\delta$. It outputs feature maps $\bar \phi_{0:H-1}$ and a dataset $\Dcal=(\Dcal_{0:H-1})$ such that Fitted Q-Iteration (FQI), using linear functions of the returned features, can be run with the returned dataset to obtain a $\veps$-optimal policy for any reward function in  $\Rcal$. The algorithm runs in the following two stages.

\begin{algorithm}
\caption{\alg($\Rcal, \Phi, \etamin, \veps, \delta$): Model-Free Feature Learning and Exploration}
\label{alg:main_alg}
\begin{algorithmic}[1]
\STATE Set 
$\beta \leftarrow  \tilde O\rbr{\frac{\etamin^2}{dK^4B^2}}$
, $\kappa \leftarrow  \frac{64dK^4\log \rbr{1+8/\beta}}{\etamin}$, and  $\veps_{\mathrm{apx}} \leftarrow \frac{\veps^2}{16H^4 \kappa K}$.
\STATE Compute the exploratory policy cover: $\cbr{\rhoxa}_{h\in[H]} \leftarrow \textsc{Explore}\rbr{\Phi, \etamin, \delta}$. \label{line:mf_explore} 
\FOR{$h \in [H]$}
\STATE Collect dataset $\Dcal^{\repbar}_h$ of size $n_{\repbar}$ using $\rhoxa$. \label{line:rep_bar_data} 
\STATE Learn representation $\bar \phi_h$ by solving \Cref{eq:emp_minmax_obj} (or calling \pref{alg:greedy_selection}) with feature class $\Phi_h$, discriminator class $\Vcal=\Gcal_{h+1}$, dataset $\Dcal^{\repbar}_h$ and tolerance $\veps_{\mathrm{apx}}$. \label{line:mf_learn_phi}
\STATE Collect dataset $\Dcal_h$ of size $n_{\downstream}$ using $\rhoxa$. \label{line:mf_collect}
\STATE Set $\Dcal \leftarrow \Dcal \bigcup \{\Dcal_{h}\}$.
\ENDFOR
\STATE \textbf{return} $\Dcal, \bar \phi_{0:H-1}$.
\end{algorithmic}
\end{algorithm}

\paragraph{Exploration} In \pref{line:mf_explore} of \alg, we use the \textsc{Explore} sub-routine (\pref{alg:explore}) to compute exploratory policies $\rhoxa$ for each timestep $h \in [H]$. It takes as input a feature set $\Phi$, the reachability coefficient $\etamin$, and the high-probability parameter $\delta$. Intuitively, \pref{alg:explore} returns a set of exploratory policies $\rhoxa$ for each timestep $h \in [H]$ such that $\rhoxa$ hits each latent state in the set $\Zcal_h$ with a large enough probability. 

\begin{algorithm}[htb]
\caption{\textsc{Explore} ($\Phi, \etamin, \delta$)}
\label{alg:explore}
\begin{algorithmic}[1]
\STATE Set 
$\beta \leftarrow \tilde O\rbr{\frac{\etamin^2}{dK^4B^2}}$, and $\veps_{\mathrm{reg}} \leftarrow \tilde{\Theta}\rbr{\frac{\etamin^3}{d^2K^9 \log^2 (1+8/\beta)}}$. 
\FOR{$h=0,\ldots,H-1$}
\STATE Set exploratory policy for step $h$ to $\rhoxa$. \label{line:explore_policy}
\STATE Collect dataset $\Dcal^{\rephat}_h$ of size $n_{\rephat}$ using $\rhoxa$.\label{line:rep_hat_data} 
\STATE Learn representation $\hat \phi_h$ for timestep $h$ by solving \Cref{eq:emp_minmax_obj} (or calling \pref{alg:greedy_selection}) with feature class $\Phi_h$, discriminator class $\Fcal_{h+1}$, dataset $\Dcal^{\rephat}_h$ and tolerance $\veps_{\mathrm{reg}}$. \label{line:learn_phi}
\STATE Collect dataset $\Dcal^{\mathrm{ell}}_h$ of size $n_{\mathrm{ell}}$ using $\rhoxa$. \label{line:collect} \vspace{-0.1cm}
\STATE Call offline elliptical planner (\pref{alg:FQI_ell_planner}) with features $\hat \phi$, dataset $\Dcal^{\mathrm{ell}}_{0:h}$ and $\beta$ to obtain policy $\rho_{h}$. \label{line:plan}
\ENDFOR
\STATE \textbf{return} Exploratory policies $\cbr{\rhoxa}_{h\in[H]}$.
\end{algorithmic}
\end{algorithm}

\pref{alg:explore} uses a step-wise forward exploration scheme similar to \flambe \citep{agarwal2020flambe}. The algorithm proceeds in stages. For each level $h$, we first collect a dataset $\Dcal^{\rephat}_h$ of size $n_{\rephat}$ (the superscript $\hat\phi$ implies that the dataset will be used to learn $\hat \phi$) with the exploratory policy $\rhoxa$ constructed in the previous level (\pref{line:rep_hat_data}) and then use such dataset to learn feature $\hat \phi_h$ (\pref{line:learn_phi}) by calling a feature learning sub-routine (discussed in the sequel).
The feature $\hat \phi_h$ is computed to approximate the property in \pref{lem:linMDP_expn} for the discriminator function class $\Fcal_{h+1} \subseteq (\Xcal \to[0, 1])$ defined as 
\begin{align}
\label{eq:def_Fcal_clipped}
\Fcal_{h+1} \hspace{-.15em}:= \hspace{-.15em}\cbr{ \hspace{-.15em} \mathrm{clip}_{[0,1]}(\EE_{\unif(\Acal)} \hspace{-.15em}\inner {\phi_{h+1}(x_{h+1}, a)}{\theta})\hspace{-.15 em}: \hspace{-.15em} \phi_{h+1} \in \Phi_{h+1},\|\theta\|_2 \le B\hspace{-.15em}} \hspace{-.15em}, \text{ where } B \ge \sqrt{d}.
\end{align}

Using the policy $\rhoxa$, we also collect the exploratory dataset $\Dcal^{\mathrm{ell}}_h$ of size $n_{\mathrm{ell}}$ (``ell'' stands for elliptical) 
for step $h$ (\pref{line:collect}). With the collected datasets\footnote{For a dataset $\Dcal_h$, subscript $h$ denotes that it is a collection of tuples $(x_h,a_h,x_{h+1})$. $\Dcal_{0:h}^{\mathrm{ell}}$ denotes $\{\Dcal_{0}^{\mathrm{ell}},\ldots,\Dcal_{h}^{\mathrm{ell}}\}$.} $\Dcal^{\mathrm{ell}}_{0:h}$ and features $\hat \phi_h$ we call an \emph{offline} ``elliptical'' planning subroutine (\pref{alg:FQI_ell_planner}) to compute the policy $\rho_h$. This planning algorithm is inspired by techniques used in reward-free exploration \citep{wang2020on, zanette2020provably,huang2021towards}. Most elliptical planning algorithms in the literature require online interactions with the environment, where in each round the agent sets the reward appropriately to collect data from a previously  unexplored direction. In contrast, our offline elliptical planner only uses the offline data, which is more complicated and built on \citet{huang2021towards}. The detailed description with the pseudocode is deferred to \pref{app:FQI_ellip_planning}.

In \pref{alg:explore}, parameters $\beta,\veps_{\mathrm{reg}}$ are set according to \pref{thm:explore_general_feat}, which is deferred to \pref{sec:proof_sketch}. The missing values $B, n_{\rephat}$, and $n_{\mathrm{ell}}$ are specifically assigned for different instantiations, and we will present them in detail when later stating the formal theoretical guarantees.

\paragraph{Representation learning} In \alg, we subsequently learn a feature $\bar \phi_h$ for each level---again by invoking the representation learning subroutine---that allows us to use \fqi to plan for any reward $R \in \Rcal$ afterwards. Similar as \textsc{Explore} sub-routine (\pref{alg:explore}), we collect a dataset $\Dcal^{\repbar}_h$ of size $n_{\repbar}$ (the superscript $\bar\phi$ implies that the dataset will be used to learn $\bar \phi$) in \pref{line:rep_bar_data}, which is then used to learn features $\bar\phi_h$ in \pref{line:mf_learn_phi}. Additionally, we use the exploratory mixture policy $\rhoxa$ to collect a dataset $\Dcal_h$ of size $n_{\downstream}$ for the downstream planning (\pref{line:mf_collect}). Here for learning feature $\bar\phi_h$, we use a discriminator function class $\Gcal_{h+1}  \subseteq (\Xcal \rightarrow[0,H])$ defined as
\begin{align}
\label{eq:def_Gcal_main}
\Gcal_{h+1}:=\Big\{\mathrm{clip}_{[0,H]}\Big(\max_{a}&(R_{h+1}(x_{h+1},a)+\inner{\phi_{h+1}(x_{h+1},a)}{\theta})\Big): \nonumber
\\
&R  \in \Rcal, \phi_{h+1} \in  \Phi_{h+1}, \|\theta\|_2  \le  B \Big\} \text{ where } B \ge H\sqrt{d}.
\end{align}
Note that the class $\Gcal_{h+1}$, while still derived from $\Phi$, is quite different from the class $\Fcal_{h+1}$ (\Cref{eq:def_Fcal_clipped}) used to learn features inside the exploration module. Recall that in $\Fcal_{h+1}$, we clip the functions to $[0,1]$, set $R_{h+1}(x_{h+1},a)=0$ and, take expectation with respect to $a \sim \unif(\Acal)$ instead of a maximum. 
Finally, \alg returns the computed features $\bar \phi_{0:H-1}$ and the exploratory dataset $\Dcal_{0:H-1}$.

In \pref{alg:main_alg}, $\beta$ and $\kappa$ are set according to \pref{thm:explore_general_feat}, and $\veps_{\mathrm{apx}}$ is set according to \pref{thm:mf_fqi}. The missing values $B, n_{\repbar}$, and $n_{\downstream}$ are again specifically chosen for different instantiations, and we discuss them in detail later.

Careful readers may have noticed that we use the same exploratory mixture policy $\rhoxa$ to collect different datasets in several places. In the practical implementation, we can equivalently collect a large enough dataset $\Dcal_h$ once and then use it to unify current $\Dcal^{\rephat}_h,\Dcal^{\repbar}_h,\Dcal^{\mathrm{ell}}_h,$ and $\Dcal_h$. Therefore, in our subsequent discussions, sometimes we will simply use $\Dcal_h$ to refer to such a dataset and do not differentiate them. 

\paragraph{Planning in downstream tasks} For downstream planning with any reward $R \in \Rcal$, we use FQI \citep{ernst2005tree,antos2007fitted,antos2008learning,munos2008finite,szepesvari2010algorithms,chen2019information} with the following Q-function class defined using the features $\bar \phi_{0:H-1}$:
\begin{gather}
\label{eq:linear_Q}
\Qcal(\bar\phi,R):=\bigcup_{h\in[H]}
\Qcal_h(\bar\phi_h,R_h),
\\
\Qcal_h(\bar{\phi}_h,R_h) := \Big\{\mathrm{clip}_{[0,H]}(R_h(x_h,a_h) + \langle \bar \phi_h(x_h,a_h), w \rangle ) : \|w\|_2 \leq B\Big\}. \nonumber
\end{gather}
Note that features $\bar \phi$ are computed to approximate the backup of candidate functions of this form (class $\Gcal_{h+1}$), and thus, satisfy the conditions stated in \citet{chen2019information} for using \fqi.

\section{Min-Max-Min Representation Learning}
\label{sec:min_max_min_rep_learn}
In this section, we describe our novel representation learning objective for finding $\hat \phi$ and $\bar \phi$. The key insight is that the low-rank property of the MDP $\Mcal$ can be used to learn a feature map $\hat \phi$ which can approximate the Bellman backup of all linear functions under feature maps $\phi \in \Phi$, and that approximating the backups of these functions enables subsequent near-optimal planning.

We present the algorithm with an abstract discriminator class $\Vcal \subseteq (\Xcal \to [0,L])$ that is instantiated either with $\Fcal_{h+1}$ (with $L=1$) or $\Gcal_{h+1}$ (with $L=H$) defined in \Cref{eq:def_Fcal_clipped} and \Cref{eq:def_Gcal_main}, respectively. In order to describe our objective, it is helpful to introduce the shorthand
\begin{align} \label{eq:be}
    \be{\pi_h, \phi_h, v; B} = \min_{\|w\|_2 \le B} \EE_{\pi_h} \sbr{\rbr{ \inner{\phih}{w} - \EE \sbr{ v(x_{h+1}) \mid x_h,a_h}}^2}
\end{align}
for any policy $\pi_h$, feature $\phi_h$, function $v$, and constant $B$, which we set so that $B \geq L\sqrt{d}$. This is the error in approximating the conditional expectation of $v(x_{h+1})$ using linear functions in the features $\phih$. For approximating backups of all functions $v \in \Vcal$, we seek feature $\hat \phi_h$ which minimizes $\max_{v \in \Vcal} \be{\rhoxa,\hat{\phi}_h,v;B}$ up to an error of $\veps_{\mathrm{tol}}$.

Unfortunately, the quantity $\be{\cdot}$ contains a conditional expectation inside the square loss, so we cannot estimate it from samples $(x_h,a_h,x_{h+1})\sim\rhoxa$. This is an instance of the well-known double sampling issue \citep{baird1995residual, antos2008learning}. Instead, we introduce the loss function
\begin{align*}
    \Lcal_{\rhoxa}(\phi_h,w,v) = \EE_{\rhoxa}\sbr{\rbr{ \inner{\phih}{w} - v(x_{h+1})}^2},
\end{align*}
which is amenable to estimation from samples. However this loss function contains an undesirable conditional variance term, since via the bias-variance decomposition we have
\[
\Lcal_{\rhoxa}(\phi_h,w,v) = \be{\rhoxa,\phi_h,v;B} + \EE_{\rhoxa}\sbr{\VV\sbr{v(x_{h+1})\mid x_h,a_h}}.
\]

The excess variance term can lead the agent to erroneously select a bad feature $\hat{\phi}_h$. 
However, via \pref{lem:linMDP_expn}, we can rewrite the conditional variance as $\Lcal_{\rhoxa}(\phi_h^*,\theta_v^*,v)$ for some $\nbr{\theta_v^*}_2\leq L\sqrt{d}$. Therefore, we can instead optimize the following \emph{variance-corrected} objective which includes a correction term:
\begin{align}
    \argmin_{\phi_h \in \Phi_h} & \max_{v \in \Vcal} \Big\{  \min_{\|w\|_2 \le B} \Lcal_{\Dcal_h}(\phi_h, w, v) - \min_{\tilde{\phi}_h \in \Phi_h, \|\tilde{w}\|_2 \le L\sqrt{d}} \Lcal_{\Dcal_h}(\tilde{\phi}_h, \tilde{w}, v) \Big\}.
    \label{eq:emp_minmax_obj}
\end{align}
Here we set the constant $B \ge L\sqrt{d}$, and $\Lcal_{\Dcal_h}(\cdot)$ is the empirical estimate of $\Lcal_{\rhoxa}(\cdot)$ using dataset $\Dcal_h=\cbr{\rbr{x^{(i)}_h,a^{(i)}_h,x^{(i)}_{h+1}}}_{i=1}^n$, which is defined as
\[
\Lcal_{\Dcal_h}(\phi_h,w,v):=\sum_{i=1}^n \rbr{ \inner{\phi_h\rbr{x^{(i)}_h,a^{(i)}_h}}{w} - v\rbr{x_{h+1}^{(i)}}}^2.
\]

We now state our first result, which is an information-theoretic guarantee and assumes that an oracle solver for the objective in \Cref{eq:emp_minmax_obj} is available when we run \alg. 
A complete proof will be given in \pref{sec:flo_analysis}.
\begin{theorem}
\label{thm:flo_result}
Fix $\delta \in (0,1)$ and consider an MDP $\Mcal$ that satisfies \pref{def:lowrank} and \pref{assum:reachability}, \pref{assum:realizability} hold. If an oracle solution to \Cref{eq:emp_minmax_obj} is available, then by setting
\begin{gather*}
B=\sqrt d, \quad n_{\rephat}  = \tilde{O}\rbr{ \frac{d^4 K^9 \log(|\Phi|/\delta)}{\etamin^3} }, \quad	n_{\mathrm{ell}} = \tilde{O}\rbr{\frac{H^5d^{11}K^{14} \log (|\Phi|/\delta)}{\etamin^5}}   ,
\\
n_{\repbar} =\tilde{O}\rbr{ \frac{H^6 d^3 K^5 \log \rbr{|\Phi||\Rcal|/\delta}}{\veps^2 \etamin} }, \quad	n_{\downstream} =  \tilde{O}\rbr{\frac{H^6 d^2 K^5 \log (|\Phi||\Rcal|/\delta)}{\veps^2 \etamin}},
\end{gather*}
with probability at least $1-\delta$, \alg returns an exploratory dataset $\Dcal$ s.t. for any $R \in \Rcal$, running \fqi with value function class $\Qcal(\bar \phi, R)$ defined in \Cref{eq:linear_Q} returns an $\veps$-optimal policy for MDP $\Mcal$. The total number of episodes used by the algorithm is
\begin{align*}
    \tilde{O}\rbr{\frac{H^6d^{11}K^{14} \log (|\Phi|/\delta)}{\etamin^5} + \frac{H^7 d^3 K^5 \log (|\Phi||\Rcal|/\delta)}{\veps^2 \etamin}}.
\end{align*}
\end{theorem}

When compared with the most related work \flambe, our sample complexity bound has worse dependence on $K,d$ and better dependence on $\varepsilon,H$. We additionally pay $1/\etamin$ and $\log(|\Rcal|)$ since our algorithm is reward-free model-free and requires the reachability assumption. On the other side, \flambe has a $\log(|\Upsilon|)$ dependence, where $\Upsilon$ is another function class that captures the second component of the low-rank decomposition. We refer the reader to \Cref{table:comparison_table} for more detailed comparisons.

\section{Iterative Greedy Representation Learning}
\label{sec:greedy_selection}

The min-max-min objective (\Cref{eq:emp_minmax_obj}) in the previous section is not provably computationally tractable for non-enumerable and non-linear function classes. However, recent empirical work~\citep{linMAMRL} has considered a heuristic approach for solving similar min-max-min objectives by alternating between updating the outer min and inner max-min components.
In this section, we show that a similar iterative approach that alternates between a \emph{squared loss minimization problem} and a \emph{max-min objective} in each iteration can be used to provably solve our representation learning problem. 

This iterative procedure is displayed in~\pref{alg:greedy_selection}. Given the discriminator class $\Vcal$ (instantiated with $\Fcal_{h+1}$ as defined in \Cref{eq:def_Fcal_clipped} for learning $\hat \phi$ or $\Gcal_{h+1}$ as defined in \Cref{eq:def_Gcal_main} for learning $\bar \phi$), the algorithm grows finite subsets $\Vcal^1,\Vcal^2,\ldots\subseteq\Vcal$ in an incremental and greedy fashion with $\Vcal^1 = \{v_1\}$ initialized arbitrarily. In the $t^{\textrm{th}}$ iteration, we have the discriminator class $\Vcal^t$ and we estimate a feature $\hat{\phi}_{t,h}$ which has a low total squared loss with respect to all functions in $\Vcal^t$ (\pref{line:min}). Importantly, the total \emph{squared loss} (sum) avoids the double sampling issue that arises with the worst case loss over class $\Vcal^t$ (max), so no correction term is required. More specifically, it is easy to see from \Cref{eq:emp_minmax_obj} and \Cref{eq:fit_feature} that once all $v_i$ are fixed, the conditional variance terms and their sum is also fixed. Thus it can be dropped when we minimize over $\phi_h$ and $W$. 

Next, we try to certify that $\hat{\phi}_{t,h}$ is a good representation by searching for a \emph{witness} function $v_{t+1} \in \Vcal$ for which $\hat{\phi}_{t,h}$ has large excess square loss (\pref{line:maxmin}). The  optimization problem in~\Cref{eq:adv_func} does require a correction term to address double sampling, but since $\hat{\phi}_{t,h}$ is fixed, it can be written as a simpler \emph{max-min} program when compared to the previous oracle approach. If the objective value $l$ (\pref{line:obj_l_val}) is smaller than the threshold $\veps_{\mathrm{tol}}$ (instantiated with $\veps_{\mathrm{reg}}$ for learning $\hat \phi$ or $\veps_{\mathrm{apx}}$ for learning $\bar \phi$), then our certification successfully verifies that $\hat{\phi}_{t,h}$ can approximate the Bellman backup of all functions in $\Vcal$, so we terminate and output $\hat{\phi}_{t,h}$. On the other hand, if the objective is large, we add the witness $v_{t+1}$ to our growing discriminator class and advance to the next iteration. 

\begin{algorithm}[!ht]
\caption{Feature Selection via Greedy Improvement}
\label{alg:greedy_selection}
\begin{algorithmic}[1]
\STATE \textbf{input:} Feature class $\Phi_h$, discriminator class $\Vcal$, dataset $\Dcal$ and tolerance $\veps_{\mathrm{tol}}$.
\STATE Set $\Vcal^0 \leftarrow \emptyset$ and choose $v_1 \in \Vcal$ arbitrarily. 
\STATE Set $\veps_0 \leftarrow \veps_{\mathrm{tol}}/52d^2$, $t \leftarrow 1$ and $l \leftarrow \infty$.
\REPEAT
\STATE Set $\Vcal^{t} \leftarrow \Vcal^{t-1} \bigcup \{v_t\}$. \label{line:increment}
\STATE \textbf{\textcolor{blue}{(Fit feature)}} Compute $\hat \phi_{t,h}$ as: \vspace{-0.2cm}
\begin{align}\hat \phi_{t,h}, W_t = \argmin_{\phi_h \in \Phi_h, W \in \R^{d\times t}, \|W\|_{2,\infty} \le L\sqrt{d}\;}  \sum_{i=1}^t \Lcal_{\Dcal_h}(\phi_h, W^i, v_i).
\label{eq:fit_feature}
\end{align}\label{line:min} \vspace{-0.2cm}
\STATE \textbf{\textcolor{blue}{(Find witness)}} Find test witness function:\label{line:maxmin}
\vspace{-0.2cm}
\begin{align}
    v_{t+1} = \argmax_{v \in \Vcal} \max_{\substack{\tilde{\phi}_h \in \Phi_h, \|\tilde{w}\|_2 \le L\sqrt{d}}} \rbr{ \min_{\|w\|_2 \le \frac{L\sqrt{dt}}{2}}  \Lcal_{\Dcal_h}(\hat \phi_{t,h},w,v) - \Lcal_{\Dcal_h}( \tilde{\phi}_h, \tilde{w}, v)}. \label{eq:adv_func}
\end{align} \vspace{-0.2cm}
\STATE Set test loss $l$ to the objective value in~\Cref{eq:adv_func}. \label{line:obj_l_val}
\UNTIL{$l < 24d^2\veps_0 + \veps_0^2$}. 
\STATE \textbf{return} Feature $\hat \phi_{T,h}$ from last iteration $T$.
\end{algorithmic}
\end{algorithm}

One technical point worth noting is that in~\Cref{eq:adv_func} we relax the norm constraint on $w$ to allow it to grow with $\sqrt{t}$ (\Cref{eq:sqrt_t_grow} in the proof of \pref{lem:dev_bound_greedy_final}). This is required by our iteration complexity analysis which we summarize in the following lemma:
\begin{lemma} (Informal)
\label{lem:greedy_selection}
Fix $\delta \in (0,1)$. If the dataset $\Dcal$ is sufficiently large, then with probability at least $1-\delta$, \pref{alg:greedy_selection} terminates after $T = \frac{52L^2d^2}{\veps_{\mathrm{tol}}}$ iterations and returns a feature $\hat \phi_h$ such that:
\begin{align}
\label{eq:greedy_error}
    \max_{v \in \Vcal} \be{\rhoxa, \hat \phi_h, v; \sqrt{\frac{13L^4d^3}{\veps_{\mathrm{tol}}}}} \le \veps_{\mathrm{tol}}.
\end{align}
\end{lemma}
The size of $\Dcal$ scales polynomially with the relevant parameters, e.g., for $\Vcal = \Fcal_{h+1}$, we set $n = \tilde{O}\rbr{ \frac{d^7 \log \rbr{|\Phi|/\delta}}{\veps^3_{\mathrm{tol}}}}$. 

A formal statement (\pref{lem:dev_bound_greedy_final}) along with its complete proof is provided in \pref{sec:greedy_analysis}. \cref{eq:greedy_error} in  \pref{lem:greedy_selection} shows that the learned feature $\hat \phi_h$ does have a small Bellman backup error (as defined in \cref{eq:be}) for all discriminator functions in the class $\Fcal_{h+1}$. Notice that the norm bound in \cref{eq:greedy_error} scales with the accuracy parameter $\veps_{\mathrm{tol}}$ and degrades the overall sample complexity when compared with using the oracle approach (\Cref{eq:emp_minmax_obj}). However, it can be used in \alg leading to a more computationally viable algorithm. In the following we state sample complexity result for \alg when \pref{alg:greedy_selection} is used as the feature learning sub-routine. The detailed analysis for the result can be found in \pref{sec:greedy_analysis}. 
\begin{theorem}
\label{thm:greedy_result}
Fix $\delta \in (0,1)$ and consider an MDP $\Mcal$ that satisfies \pref{def:lowrank} and \pref{assum:reachability}, \pref{assum:realizability} hold. If \Cref{eq:emp_minmax_obj} is solved via the iterative greedy approach (\pref{alg:greedy_selection}), then by setting
\begin{gather*}
B =\tilde O\rbr{\sqrt{\frac{d^5K^9}{\etamin^3}}},\; n_{\rephat}  = \tilde{O}\rbr{ \frac{d^{13}K^{27} \log (|\Phi|/\delta)}{\etamin^9}}, \;	n_{\mathrm{ell}} =  \tilde{O}\rbr{\frac{H^5d^{19}K^{32} \log (|\Phi|/\delta)}{\etamin^{11}}},
\\
n_{\repbar}  = \tilde{O}\rbr{ \frac{H^{18}d^{10}K^{15} \log (|\Phi||\Rcal|/\delta)}{\veps^6 \etamin^3}},\quad n_{\downstream} =  \tilde{O}\rbr{\frac{H^6 d^2 K^5 \log (|\Phi||\Rcal|/\delta)}{\veps^2 \etamin}},
\end{gather*}
with probability at least $1-\delta$, \alg returns an exploratory dataset $\Dcal$ s.t. for any $R \in \Rcal$, running \fqi with value function class $\Qcal(\bar \phi, R)$ defined in \Cref{eq:linear_Q} returns an $\veps$-optimal policy for MDP $\Mcal$. The total number of episodes used by the algorithm is
\begin{align*}
    \tilde{O}\rbr{\frac{H^6d^{19}K^{32} \log (|\Phi|/\delta)}{\etamin^{11}} + \frac{H^{19}d^{10}K^{15} \log (|\Phi||\Rcal|/\delta)}{\veps^6 \etamin^3}}.
\end{align*}
\end{theorem}

Applying \pref{alg:greedy_selection} leads to a more computationally viable algorithm by breaking the feature learning objective in \Cref{eq:emp_minmax_obj} into familiar computational primitives: \emph{squared loss minimization} (\cref{eq:fit_feature}) and a \emph{saddle point problem} (\cref{eq:adv_func}). Apart from the squared loss minimization which is considered tractable, saddle point problems have also become common in recent off-policy RL methods like \citet{dai2018sbeed,zhang2019gendice} where scalable optimization heuristics are presented as well. For the practical implementation of \pref{alg:greedy_selection}, we can choose parameterized $\Phi$ and $\Vcal$ classes, which allow us to use gradient descent based methods for learning features and finding witness functions (discriminators) in a scalable manner. We want to highlight that the subsequent work of \citet{zhang2022efficient}, which builds on the same representation learning oracle as ours (with minor differences in the discriminator class), provides empirical evidence that our iterative greedy representation learning oracle (\pref{alg:greedy_selection}) is indeed implementable and more computationally viable.

\section{Provably Computationally-Tractable Reward-Free RL with an Enumerable Feature Class}
\label{sec:enumerable}

A critical component of \alg (\pref{alg:main_alg}) is the stage of collecting exploratory data. The problem of reward-free exploration only asks the agent to collect a dataset with good coverage over the state space and does not require the agent to learn a representation. This problem has been studied in recent literature for tabular \citep{jin2020reward,kaufmann2021adaptive,menard2021fast,zhang2021near}, block MDPs \citep{misra2020kinematic}, and linear-MDP/low inherent Bellman error/linear-mixture setting with known features \citep{wang2020on,zanette2020provably,zhang2021reward,huang2021towards,wagenmaker2022reward}.\footnote{We only provide an incomplete list here.} In this section, we describe a special case where the min-max-min objective (\Cref{eq:emp_minmax_obj}) from \alg results in a provably computationally-tractable reward-free exploration scheme. 

In particular, we show that when $\Phi$ is efficiently enumerable, we can use \pref{alg:explore} to compute an exploratory policy cover in a computationally tractable manner. We can learn $\hat \phi$ using a slightly different min-max-min objective
\begin{align}
\label{eq:ridge_objective}
    \argmin_{\phi \in \Phi_h}\max_{f \in \Fcal_{h+1},\tilde{\phi} \in \Phi_{h},\|\tilde{w}\|_2 \le B  } \cbr{ \min_{\|w\|_2 \le B}  \Lcal_{\Dcal_h}(\phi, w, f) - \Lcal_{\Dcal_h}(\tilde{\phi}, \tilde{w}, f)},
\end{align}
where with a slight abuse of notation $\Fcal_{h+1}$ (different from that in \Cref{eq:def_Fcal_clipped}) now is the discriminator class that contains all \emph{unclipped} functions $f$ in form of 
\begin{align}
\label{eq:def_Fcal_unclipped}
\Fcal_{h+1}:=\cbr{\EE_{\unif(\Acal)}\sbr{\inner {\phi_{h+1}(x_{h+1}, a)}{\theta}}:\phi_{h+1} \in \Phi_{h+1},\|\theta\|_2 \le \sqrt{d}}.
\end{align}

Consider the min-max-min objective \Cref{eq:ridge_objective} and fix $\phi, \tilde{\phi} \in \Phi_h$. We show that it can be can be reduced to
\begin{align}
\label{eq:max_f_ridge}
    \max_{f \in \Fcal_{h+1}} f(\Dcal_h)^\top \rbr{A(\phi)^{\top} A(\phi) -  A(\tilde{\phi})^{\top} A(\tilde{\phi})} f(\Dcal_h)
\end{align}
where $A(\phi) = I_{n\times n} - X \rbr{\tfrac{1}{n}X^\top X + \lambda I_{d\times d}}^{-1}\rbr{\tfrac{1}{n} X^\top}$, $A(\tilde \phi) = I_{n\times n} - \tilde X \rbr{\tfrac{1}{n}\tilde X^\top \tilde X + \lambda I_{d\times d}}^{-1}$ $\rbr{\tfrac{1}{n} \tilde X^\top}$ for a parameter $\lambda$. Additionally, $X, \tilde{X} \in \R^{n \times d}$ are the sample covariate matrices for features $\phi, \tilde{\phi}$ respectively, and we overload the notation and use $f(\Dcal_h) \in \R^n$ to denote the value of any $f\in\Fcal_{h+1}$ on the $n$ samples. The objective in \Cref{eq:max_f_ridge} is obtained by using a ridge regression solution for $w, \tilde{w}$ in \Cref{eq:ridge_objective} and the details are deferred to \pref{sec:enumerable_analysis}. 

Finally, for any fixed feature $\phi'$ in the definition of $f = X'\theta \in \Fcal_{h+1}$, we can rewrite \Cref{eq:max_f_ridge} as
\begin{align*}
    \max_{\|\theta\|_2 \le \sqrt{d}} \theta^\top X'^\top \rbr{A(\phi)^{\top} A(\phi) -  A(\tilde{\phi})^{\top} A(\tilde{\phi})} X' \theta
\end{align*}
where $X' \in \R^{n \times d}$ is again a sample matrix defined using $\phi' \in \Phi_{h+1}$.

Thus, for a fixed tuple of $(\phi, \tilde{\phi}, \phi')$, the maximization problem reduces to a tractable eigenvector computation problem. As a result, we can efficiently solve the min-max-min objective in \Cref{eq:ridge_objective} by enumerating over each candidate feature in $(\phi, \tilde{\phi}, \phi')$ to solve
\begin{equation}
    \argmin_{\phi\in\Phi_h}\max_{\tilde{\phi} \in \Phi_h,\phi' \in \Phi_{h+1},\|\theta\|_2 \le \sqrt{d}} \theta^\top X'^\top \rbr{A(\phi)^{\top} A(\phi) -  A(\tilde{\phi})^{\top} A(\tilde{\phi})} X' \theta.
    \label{eq:ev-enum}
\end{equation}

While the analysis is more technical, \Cref{eq:ev-enum} still allows us to plan using $\hat \phi_h$ in \expl (\pref{alg:explore}) to guarantee that the policies $\rhoxa$ are exploratory. With the exploratory data, we can subsequently call \fqi with the following Q-function class defined using the entire feature class for the downstream planning:
\begin{gather}
    \label{eq:full_Q}
    \Qcal(R):=\bigcup_{h\in[H]}\Qcal_h(R),
    \\
    \Qcal_h(R) := \Big\{\mathrm{clip}_{[0,H]}\rbr{R_h(x,a) + \langle \phi_h(x,a), w \rangle }: \|w\|_2 \leq B, \phi_h \in \Phi_h\Big\},\text{ where } B \ge H\sqrt{d}.\nonumber
\end{gather}

We summarize the overall result as follows and its proof can be found in \pref{sec:enumerable_analysis}.
\begin{theorem}
\label{thm:enumerable_result}
Fix $\delta \in (0,1)$ and consider an MDP $\Mcal$ that satisfies \pref{def:lowrank} and \pref{assum:reachability}, \pref{assum:realizability} hold. In \expl (\pref{alg:explore}), if $\hat \phi_h$ is learned using the eigenvector formulation \Cref{eq:ev-enum}, then by setting
\begin{gather*}
B =\tilde{\Theta}\rbr{\frac{d^4K^9\log\rbr{|\Phi|/\delta}}{\etamin^3}}, \quad n_{\rephat} = \tilde{O}\rbr{ \frac{d^{12} K^{27} \log^3(|\Phi|/\delta)}{\etamin^9} },
\\
n_{\mathrm{ell}}  = \tilde{O}\rbr{\frac{H^5d^{25}K^{50} \log^5 (|\Phi|/\delta)}{\etamin^{17}}}, \quad n_{\downstream}  = \tilde{O}\rbr{\frac{H^6d^3 K^5 \log(|\Phi||\Rcal|/\delta)}{\veps^2\etamin}},
\end{gather*}
\alg returns an exploratory dataset $\Dcal$ such that for any $R \in \Rcal$, running \fqi with the collected dataset and the value function class $\Qcal(R)$ defined in \Cref{eq:full_Q} returns an $\veps$-optimal policy with probability at least $1-\delta$. The total number of episodes used by the algorithm is 
\begin{align*}
    \tilde{O}\rbr{\frac{H^6d^{25}K^{50} \log^5 (|\Phi|/\delta)}{\etamin^{17}} + \frac{H^7d^3 K^5 \log (|\Phi||\Rcal|/\delta)}{\veps^2\etamin} }.
\end{align*}
\end{theorem}

From the result, we can see that although \expl with \Cref{eq:ev-enum} enumerates over the candidate feature class $\Phi$, the sample complexity is still logarithmic in $|\Phi|$. We believe there may be still room for further improving the bound, and we leave it to future work.

\section{Proofs}
\label{sec:proof_main}
In this section, we present detailed proofs. We start with the overall proof outline for \alg in \pref{sec:proof_sketch}, and then provide the proofs for the exploration and representation learning components in \pref{sec:explore_analysis} and \pref{sec:reward_free_planning} respectively. The concrete results for min-max-min oracle representation learning (\Cref{eq:emp_minmax_obj}), iterative greedy representation learning (\pref{alg:greedy_selection}), and enumerable case (\Cref{eq:ev-enum}) are shown in \pref{sec:flo_analysis}, \pref{sec:greedy_analysis}, and \pref{sec:enumerable_analysis} respectively. We defer the statements and proofs for the offline elliptical planner (\pref{alg:FQI_ell_planner}), \fqi (\pref{alg:linear_fqi}), \fqe (\pref{alg:linear_fqe}), and auxiliary results (e.g., concentration arguments and deviation bounds for regression with squared loss) to the appendix.
\subsection{Proof Outline}
\label{sec:proof_sketch}
We provide some intuition behind the design choices in \alg and give a sketch of the proof of the main results. 
We divide the proof sketch into four stages: 
(i) establishing the exploratory nature of the policies $\rho^{+3}$, 
(ii) representation learning guarantees for features $\bar\phi$ used for the downstream planning, (iii) concentration arguments for learned features $\hat \phi$ and $\bar\phi$, and (iv) final planning in downstream tasks.

\paragraph{Computing exploratory policies}
To understand the intuition behind \expl (\pref{alg:explore}), it is helpful to consider how we can discover a policy cover over the latent state space $\Zcal_{h+1}$. If we knew the mapping to latent states, we could create the reward functions $\one[z_{h+1}=z]$ for all $z\in \Zcal_{h+1}$ and compute policies to optimize such rewards, but here we do not have access to this mapping. Additionally, we do not have access to the true features $\phi^*$ to enable tractable planning even for the known rewards. \expl tackles both of these challenges. Note that in this section, we will establish the coverage over $\Zcal_{h+1}$ through learning feature $\hat \phi_{h-2}$ and calling the offline ``elliptical planner'' (\pref{alg:FQI_ell_planner}) to build an exploratory mixture policy $\rho_{h-2}$. This is for the simplicity of presentation. To connect it to the computation at level $h$ in \expl, we need to add all subscripts by 2, i.e., changing $\Zcal_{h+1}$ to $\Zcal_{h+3}$, $\hat \phi_{h-2}$ to $\hat \phi_h$, $\rho_{h-2}$ to $\rho_h$, etc.

For the first challenge, we note that by~\pref{def:latent_var} of the latent variables and~\pref{lem:linMDP_expn},  there always exists $f(x_h,a_h)=\PP[z_{h+1}=z \mid x_h,z_h]$ such that 
\begin{align*}
  &~\PP[z_{h+1}=z\mid  x_{h-1},a_{h-1}]= \EE[\one[z_{h+1}=z]\mid  x_{h-1},a_{h-1}] 
  \notag\\
  =&~ \EE[f(x_h, a_h)\mid  x_{h-1},a_{h-1}]= \langle \phi^*_{h-1}(x_{h-1},a_{h-1}), \theta^*_f\rangle.
\end{align*}
This essentially says the desired indicator function for reaching latent states at level $h+1$ can be written as some function $f$ at level $h$. Although such $f$ cannot be directly captured by the given the candidate feature class $\Phi$, after one Bellman backup, it can be represented by a linear function of the true feature $\phi_{h-1}^*$ (at level $h-1$). This inspires us to construct the discriminator class $\Fcal_{h-1}$ (the clipped version is defined in \Cref{eq:def_Fcal_clipped} and the unclipped version is defined in \Cref{eq:def_Fcal_unclipped}), which includes $\inner{ \phi^*_{h-1}(x_{h-1},a_{h-1})}{\theta^*_f}$. This overcomes the first challenge.

Now we discuss how to tackle the second challenge: finding good features ($\hat\phi_{h-2}$) to enable planning for the known rewards. Given the discriminator class, in the sequel, we specify the objective of our feature learning step. In \expl (\pref{alg:explore}), we learn $\hat \phi_{h-2}$ so that for any appropriately bounded $\theta$, there is a $w$ such that
\begin{align}
\EE[\langle \phi^*_{h-1}(x_{h-1},a_{h-1}), \theta\rangle \mid x_{h-2},a_{h-2}] \approx \langle\hat\phi_{h-2}(x_{h-2},a_{h-2}), w\rangle,\label{eq:alekh_backup}
\end{align}
More specifically, in the feature learning step (\pref{line:learn_phi} of \pref{alg:explore}), we learn $\hat \phi_{h-2}$ such that, for some fixed scalar $B$ and the discriminator class $\Fcal_{h-1}$ discussed above, it satisfies
\begin{align}
   \max_{f\in \Fcal_{h-1}} \be{\rho_{h-5}^{+3}, \hat \phi_{h-2}, f; B} \le \veps_{\mathrm{reg}} \label{eq:inv_phi},
\end{align}
where (recall that) $\be{\cdot}$ is defined in \Cref{eq:be}.

Intuitively, it is guaranteed that we can find such a good $\hat \phi_{h-2}\in\Phi_{h-2}$ because again from \pref{lem:linMDP_expn}, we know that true feature $\phi^*_{h-2}$ satisfies \Cref{eq:alekh_backup}. We defer the detailed discussion on how to find $\hat\phi_{h-2}$ that satisfies \Cref{eq:inv_phi} to the later part as we focus on computing exploratory policies here. 

For building an (exploratory) mixture policy $\rho_{h-2}$ that effectively covers all directions spanned by $\hat{\phi}_{h-2}$, we employ the \emph{offline} ``elliptical planner'' for reward-free exploration \citep{huang2021towards} and optimize reward functions that are quadratic in the learned features $\hat \phi_{h-2}$. 
To do so, in \pref{alg:FQI_ell_planner},
we repeatedly (i) update the elliptical reward, (ii) invoke \fqi subroutine (\pref{alg:linear_fqi}) to obtain a greedy policy w.r.t. the elliptical reward, and (iii) call \fqe (Fitted Q-Evaluation, \citet{le2019batch}) subroutine (\pref{alg:linear_fqe}) with the policy obtained from \fqi to estimate the covariance matrix (to update the elliptical reward in the next iteration) and estimate the expected return (to check the stopping criterion). For both \fqi and \fqe, we use a function class comprising of all reward-appended linear functions of $\phi \in \Phi$. 

Instead of using the offline ``elliptical planner'', we can also employ an online version, where we substitute the \fqe component with Monte Carlo rollout to estimate. However, this requires collecting additional samples and leads to worse sample complexity bound due to lack of data reuse. One appealing guarantee we can obtain by using our offline ``elliptical planner'' is the optimal deployment complexity formulated in \citet{huang2021towards} and it can be easily verified. On the technical side, our ``elliptical planner'' builds on \citet{huang2021towards}, but is established without discretization over the value function class and the reward functions. This advance makes the algorithm more computationally friendly. However, we need to apply the more involved concentration analysis (uniform Bernstein's inequality) for the infinite function class to achieve sharp rates. We adapt the tools and analysis from \citet{dong2019sqrt} and show a key concentration result in \pref{corr:uni_bern_conf}, which is then applied in the squared loss deviation result in \pref{app:dev_bound}. We provide a complete description and analysis for the ``elliptical planner'' in \pref{app:FQI_ellip_planning}, and its induced guarantee for proving the distribution shift argument in \pref{lem:planner_result}. The related \fqi (\pref{alg:linear_fqi}) and \fqe (\pref{alg:linear_fqe}) analyses used in the ``elliptical planner'' are presented in \pref{app:fqi_ellip} and \pref{app:fqe} respectively. 

With the help of \pref{lem:planner_result} and based on our earlier intuition for covering $\Zcal_{h+1}$ by translating the indicator reward at level $h+1$ to feature $\hat \phi_{h-2}$, we show that the policy $\rhoxnxt = \rho_{h-2} \circ \unif(\Acal) \circ \unif(\Acal)$ is exploratory and hits all latent states $z \in \Zcal_{h+1}$,
\begin{align}
\label{eq:explore_latent_state}
    \max_\pi \PP_{\pi} \sbr{z_{h+1} = z} \le \kappa \PP_{\rhoxnxt} \sbr{z_{h+1} = z},
\end{align}
where $\kappa > 0$ is a constant specified in \pref{thm:explore_general_feat}. 

The formal proof is established in an inductive way, i.e., we assume \Cref{eq:explore_latent_state} holds for all $h'\le h$ and then show it also holds for $h+1$. The main reason is that we need exploratory policies/datasets at the prior levels to be fed into the offline ``elliptical planner''. For the induction base ($h=0$), it is easy to verify that the null policy $\rho_{-2}^{+2}$ satisfies the exploration guarantee in \Cref{eq:explore_latent_state}.

One thing that eluded the objective \Cref{eq:inv_phi} for learning $\hat\phi$ is that we need exploratory policies. More specifically, we can see that \Cref{eq:inv_phi} is defined with an exploratory policy $\rho_{h-5}$. Therefore, the goals of representation learning and exploration are intertwined. However, this is indeed not a concern since we need $\hat \phi_{h-2}$ when exploring/covering $\Zcal_{h+1}$ (or building $\rho_{h-2}$) while learning $\hat \phi_{h-2}$ only requires $\rho_{h-5}$ (exploratory policies at the prior step). By our inductive proof, we can observe that $\rho_{h-5}$ has already been established at this stage. 

Also notice that we plan in the previously learned features $\hat\phi_{h-2}$ to obtain a cover over $\Zcal_{h+1}$. This way, planning trails feature learning like \flambe, but with an additional step of lag due to differences between model-free and model-based reasoning.

Based on the exploratory property in \Cref{eq:explore_latent_state}, taking another action $a_{h+1}$ uniformly at random further returns an exploratory policy $\rho_{h-2}^{+3}$ for state-action pairs $(x_{h+1},a_{h+1})$. 
In summary, we provide the following result for the policies returned by \pref{alg:explore} with details in \pref{sec:explore_analysis}:

\begin{theorem}
\label{thm:explore_general_feat}
Fix $\delta \in (0,1)$ and consider an MDP $\Mcal$ that satisfies \pref{def:lowrank} and \pref{assum:reachability}, \pref{assum:realizability} hold. If the features $\hat \phi_h$ learned in \pref{line:learn_phi} of \expl (\pref{alg:explore}) satisfy the condition in \Cref{eq:inv_phi} for $B \ge \sqrt{d}$, and $\veps_{\mathrm{reg}} = \tilde{\Theta}\rbr{\frac{\etamin^3}{d^2K^9 \log^2 (1+8/\beta)}}$, then with probability at least $1-\delta$, the sub-routine \textsc{Explore} collects an exploratory mixture policy $\rhoxa$ for each level $h$ such that
\begin{align}
\label{eq:dist_shift}
    \forall \pi, \forall f: \Xcal \times \Acal \rightarrow \R^+\text{, we have } \, \EE_{\pi} [f(x_h,a_h)] \le \kappa K \EE_{\rhoxa} [f(x_h,a_h)],
\end{align}
where $\kappa = \frac{64dK^4\log \rbr{1+8/\beta}}{\etamin}$. The total number of episodes used in \pref{line:plan} by \expl is
\begin{align*}
    \tilde{O} \rbr{\frac{H^5d^9K^{14}B^4 \log (|\Phi|/\delta)}{\etamin^5}},
\end{align*}
with $\beta$ chosen to satisfy $\beta \log \rbr{1+8/\beta} \le \frac{\etamin^2}{128dK^4B^2}$ and a sufficient one is $\beta =\tilde O\rbr{\frac{\etamin^2}{dK^4B^2}}$.

\end{theorem}
This result specifies the required size $n_{\mathrm{ell}}$ of dataset $\Dcal^{\mathrm{ell}}$ in \pref{alg:FQI_ell_planner}. We need to choose different values of $B$ to guarantee approximation error bound for $\hat \phi$ (\Cref{eq:inv_phi}) holds for different instantiations, and we discuss the details in the deviation bounds for learned features part.

The precise dependence on parameters $d, H, K$, and $\etamin$ is likely improvable. The exponent on $K$ arises from multiple importance sampling steps over the uniform action choice and can be improved when the features $\phi(x,a) \in \Delta(d)$ for all $x,a$, and $\phi\in\Phi$ (\pref{sec:explore_simplex}). Improving these dependencies further is an interesting avenue for future progress.

\paragraph{Representation learning for downstream tasks}
For showing planning guarantees using \fqi, we need to ensure that the following requirements stated in \citet{chen2019information} are satisfied: (i) (\emph{concentrability}) we have adequate coverage over the state space, (ii) (\emph{realizability}) we can express $Q^*$ (more specifically it is $Q^*_R$, i.e., $Q^*$ under the reward function $R$ that we consider) with our function class, and (iii) (\emph{completeness}) our class is closed under Bellman backups. Condition (i) is implied by \pref{thm:explore_general_feat}. For (ii) and (iii), we learn features $\bar \phi_{0:H-1} \in \Phi$ such that $\bar \phi_h$ satisfies 
\begin{align}
    \label{eq:fqi_rep_error}
    \max_{g \in \Gcal_{h+1}} \be{\rhoxa, \bar \phi_h, g; B} \le \veps_{\mathrm{apx}}.
\end{align}
Here $\Gcal_{h+1}$ as defined in \Cref{eq:def_Gcal_main} is the discriminator class containing all reward-appended linear candidate Q-value functions, which in turn includes the true $Q^*$ value function for all $R\in\Rcal$. The main conceptual difference over \Cref{eq:inv_phi} is that the discriminator class $\Gcal_{h+1}$ now incorporates reward information, which enables downstream planning. The objective for learning $\bar \phi$ (\Cref{eq:fqi_rep_error}) again includes the exploratory policies $\rho$, but they have been fully established in the exploration phase. 
Using \Cref{eq:fqi_rep_error} and the low-rank MDP properties, we show that this function class satisfies approximate realizability and approximate completeness, so we can invoke results for FQI and obtain the following representation learning guarantee. The details can be found in \pref{sec:reward_free_planning}.
\begin{theorem}
\label{thm:mf_fqi}
Fix $\delta \in (0,1)$ and consider an MDP $\Mcal$ that satisfies \pref{def:lowrank} and \pref{assum:reachability}, \pref{assum:realizability} hold. If the features $\bar \phi_{0:H-1}$ learned by \alg satisfy the condition in \Cref{eq:fqi_rep_error} for all $h$ with $\veps_{\mathrm{apx}} = \tilde{O}\rbr{\frac{\veps^2 \etamin}{dH^4 K^5}}$, then for any reward function $R \in \Rcal$, running \fqi with the value function class $\Qcal(\bar \phi, R)$ in \Cref{eq:linear_Q} and an exploratory dataset $\Dcal$, returns a policy $\hat \pi$, which satisfies $v^{\hat \pi}_R \ge v^*_R - \veps$ with probability at least $1-\delta$. The total number of episodes collected by \alg in \pref{line:mf_collect} is:
\begin{align*}
    \tilde{O}\rbr{\frac{H^7 d^2 K^5 \log (|\Phi||\Rcal|B/\delta) }{\veps^2 \etamin}}.
\end{align*}
\end{theorem}

\paragraph{Deviation bounds for learned features} 
One thing we skipped in the earlier discussion is how to establish guarantees for learning $\hat \phi$ that satisfies \Cref{eq:inv_phi} and $\bar\phi$ that satisfies \Cref{eq:fqi_rep_error}. The key technical component used in all proofs is the Bernstein's version of uniform concentration result (\pref{corr:uni_bern_conf}). With this careful concentration argument, in \pref{lem:dev_bound_flo}, we show that the min-max-min oracle subroutine (\Cref{eq:emp_minmax_obj}) can be used to achieve these goals with appropriate choices of $n_{\rephat}$ and $n_{\repbar}$. The corresponding guarantees of the parameters for iterative greedy representation learning (\pref{alg:greedy_selection}) is presented in \pref{lem:dev_bound_greedy_final}, where the analysis is more complicated and we additionally use a potential argument to give the bound on its iteration complexity.

For feature learning in the enumerable case (\pref{sec:enumerable}), we only provide the guarantee to learn $\hat \phi$ that satisfies \Cref{eq:inv_phi} with large enough $n_{\rephat}$ in \pref{lem:ridge_error}, as we only present the reward-free exploration result for this version. We invoke the Bernstein's type concentration result on the ridge regression estimator. Therefore, we need to treat the scale of the regularizer and the bias term carefully, which leads to a worse rate in the sample complexity bound. 

The detailed choices of $n_{\rephat}$ and $n_{\repbar}$ are calculated from the deviation bounds and the thresholds $\veps_{\mathrm{reg}}$ (\pref{thm:explore_general_feat}) and $\veps_{\mathrm{apx}}$ (\pref{thm:mf_fqi}). Recall that previously we skipped the choices of $B$ for setting $n_{\mathrm{ell}}$ in \pref{thm:mf_fqi}. The proper way to specify them are also discussed in the deviation results (\pref{lem:dev_bound_flo}, \pref{lem:dev_bound_greedy_final}, and \pref{lem:ridge_error}).

\paragraph{Planning in downstream tasks} 
We combine the representation learning guarantees with the sample complexity analysis for \fqi to set the size $n_{\downstream}$ of dataset $\Dcal$ for planning in downstream tasks. The specific values are set according to \pref{thm:mf_fqi}. For the enumerable feature instance, we integrate the reward-free exploration guarantee and the \fqi result for planning for a reward class with the full representation class (\pref{corr:fqi_full_class}), which also gives us the choice of $n_{\downstream}$.

\subsection{Proofs for Exploration and Sample Complexity Results for \pref{alg:explore}}
\label{sec:explore_analysis}
In this section, we present proofs for the exploration and sample complexity results for \expl (\pref{alg:explore}). We provide the result for the low-rank setting in \pref{sec:explore_low_rank} and an improved result for the simplex feature setting in \pref{sec:explore_simplex}.

\subsubsection{Proof of \pref{thm:explore_general_feat}}
\label{sec:explore_low_rank}
\paragraph{Proof of \pref{thm:explore_general_feat}} We will now prove the result assuming that the following condition from \Cref{eq:inv_phi} is satisfied by $\hat \phi_h$ for all $h \in [H]$ with probability at least $1-\delta/4$
\begin{align}
    \max_{f \in \Fcal_{h+1}} \min_{\|w\|_2 \le B} \EE_{\rhoxa} \sbr{\Big( \inner{\hphih}{w} - \EE \sbr{ f(x_{h+1})\mid x_h,a_h}\Big)^2}  \le \veps_{\mathrm{reg}}. \label{eq:app_inv_phi}
\end{align}

Now, let us turn to the inductive argument to show that the constructed policies $\rhoxa$ are exploratory for every $h$. We will establish the following inductive statement for each timestep $h$:
\begin{align}
\label{eq:inv_explore}
        \forall z \in \Zcal_{h+1}: \, \max_\pi \PP_{\pi}\sbr{z_{h+1} = z} \le \kappa \PP_{\rhoxnxt}\sbr{z_{h+1} = z}.
\end{align}
Assume that the exploration statement \Cref{eq:inv_explore} is true for all timesteps $h' \le h$. We first show an error guarantee similar to \Cref{eq:dist_shift} under distribution shift:
\begin{lemma}
\label{lem:dist_shift}
If the inductive assumption in \Cref{eq:inv_explore} is true for all $h' \le h$, then for all $v: \Xcal \times \Acal \rightarrow \R^+$ we have
\begin{align}
\label{eq:inv_explore_1}
    \forall \pi: \, \EE_{\pi} [v(x_h, a_h)] \le \kappa K \EE_{\rhoxa} [v(x_h, a_h)].
\end{align}
\end{lemma}
\begin{proof}
Consider any timestep $h$ and non-negative function $v$. 
Using the inductive assumption, we have
\begin{align*}
    \EE_{\pi} \sbr{v(x_h,a_h)} = {} & \sum_{z \in \Zcal_{h}} \PP_\pi[z_h = z] \cdot \int \EE_{\pi_h}[v(x_h, a_h)] \nu^*(x_h\mid z) d(x_h) \\
    \le {} & \kappa \sum_{z \in \Zcal_h} \PP_{\rhox}[z_h = z] \cdot \int \EE_{\pi_h}[v(x_h, a_h)] \nu^*(x_h\mid z) d(x_h) \\
    = {} & \kappa \EE_{\rhox}[\EE_{\pi_h}[v(x_h, a_h)]]
    \\
    \le {} & \kappa K \EE_{\rhoxa}[v(x_h,a_h)].
\end{align*}
Therefore, the result holds for any policy $\pi$, timestep $h' \le h$, and non-negative function $v$.
\end{proof}

Choosing $v(x_h,a_h) =  \rbr{\inner{\hphih}{w} - \EE\sbr{f(x_{h+1}) \mid x_h,a_h} }^2$ and using the feature learning guarantee in \Cref{eq:app_inv_phi} along with \Cref{eq:inv_explore_1}, we have
\begin{align}
\label{eq:apx_error_abr_dist}
\forall \pi, \forall f \in \Fcal_{h+1}: \, \min_{\|w\|_2 \le B} \EE_{\pi}\sbr{ \rbr{\inner{\hphih}{w} - \EE\sbr{f(x_{h+1}) \mid x_h,a_h} }^2 } \le \kappa K \veps_{\mathrm{reg}}.
\end{align}

We now outline our key argument to establish exploration: Fix a latent variable $z \in \Zcal_{h+1}$ and let $\pi \coloneqq \pi_h$ be the policy which maximizes $\PP_\pi [z_{h+1} = z]$. Thus, with 
\[f(x_h,a_h) = \PP_\pi[z_{h+1} = z \mid x_h, a_h]\]
we have
\begin{align}
    \EE_{\pi} \sbr{f(x_h, a_h)} \le {} & K^2 \EE_{\pi_{h-2} \circ \,\unif(\Acal) \circ \,\unif(\Acal)} \sbr{f(x_h, a_h)} \nonumber \\ 
    = {} & K^2 \EE_{\pi_{h-2}} \sbr{\EE_{\unif(\Acal)}\sbr{g(x_{h-1},a_{h-1})\mid x_{h-2},a_{h-2}}} \nonumber \\
    \le {} & K^2 \EE_{\pi_{h-2}} \sbr{\abr{\langle \hat \phi_{h-2}(x_{h-2},a_{h-2}), w_g\rangle}} + \sqrt{\kappa K^5 \veps_{\mathrm{reg}}}. \label{eq:ex_transfer} 
\end{align}
The first inequality follows by using importance weighting on timesteps $h-1$ and $h$, where we choose actions uniformly at random among $\Acal$. In the next step, we define 
\[g(x_{h-1}, a_{h-1}) = \EE_{\unif(\Acal)} [f(x_h,a_h)\mid x_{h-1}, a_{h-1}] = \inner{\phi^*_{h-1}(x_{h-1},a_{h-1})}{\theta^*_f}\]
with $\|\theta^*_f\|_2 \le \sqrt{d}$. For the last inequality, we first use the result from \pref{lem:dist_shift} that $\hat \phi_{h-2}$ has a small squared loss for the regression target specified by $g(\cdot)$ with a vector $w_g$ defined as
\[w_g:= \argmin_{\|w\|_2 \le B} \EE_{\rho_{h-5}^{+3}} \sbr{\Big( \inner{\hat \phi_{h-2}(x_{h-2},a_{h-2})}{w} - \EE_{\unif(\Acal)}\sbr{g(x_{h-1},a_{h-1})\mid x_{h-2},a_{h-2}}\Big)^2}.\]
We further use \Cref{eq:apx_error_abr_dist} to translate the error from  $\rho_{h-5}^{+3}$ to $\pi_{h-2}$ and apply the weighted RMS-AM inequality in the same step to bound the mean absolute error using the squared error bound. 

\begin{lemma}
\label{lem:planner_result}
If the offline elliptical planner (\pref{alg:FQI_ell_planner}) is called with a sample of size 
\[\tilde{O}\rbr{\frac{H^4d^6\kappa K^2 \log(|\Phi|/\delta)}{\beta^2}},\]
then with probability at least $1-\delta$, for all $h\in[H]$, we have
\begin{align*}
    \EE_{\pi_{h-2}} \sbr{\abr{\langle \hat \phi_{h-2}(x_{h-2},a_{h-2}), w_g\rangle}} \le {} & \frac{\alpha}{2}\EE_{\rho_{h-2}} \sbr{\rbr{ \langle \hat \phi_{h-2}(x_{h-2},a_{h-2}), w_g\rangle}^2} + \frac{T\beta}{2\alpha} \\
    {} & + \frac{\alpha \|w_g\|^2_2}{2T} + \frac{\alpha \beta \|w_g\|^2_2}{2}.
\end{align*}
\end{lemma}
\begin{proof}
Applying Cauchy-Schwarz inequality followed by AM-GM, for any matrix $\widehat{\Sigma}$, we have
\begin{align*}
    \EE_{\pi_{h-2}} \sbr{\abr{\langle \hat \phi_{h-2}(x_{h-2},a_{h-2}), w_g\rangle}} \le {} & \EE_{\pi_{h-2}} \sbr{\nbr{\hat \phi_{h-2}(x_{h-2},a_{h-2})}_{\widehat{\Sigma}^{-1}} \cdot \nbr{w_g}_{\widehat{\Sigma}}}\\
    \le {} & \frac{1}{2\alpha}\EE_{\pi_{h-2}} \sbr{\nbr{\hat \phi_{h-2}(x_{h-2},a_{h-2})}_{\widehat{\Sigma}^{-1}}^2} + \frac{\alpha}{2}\nbr{w_g}^2_{\widehat{\Sigma}}.
\end{align*}
Here, we choose $\widehat{\Sigma}$ to be the (normalized) matrix returned by the elliptic planner in \pref{alg:FQI_ell_planner}. As can be seen in the algorithm pseudocode, $\widehat \Sigma$ is obtained by summing up a (normalized) identity matrix and the empirical estimates of the population covariance matrix $\Sigma_{\pi_\tau} = \EE_{\pi_\tau} \hat \phi_{h-2}(x_{h-2}, a_{h-2})\hat \phi_{h-2}(x_{h-2}, a_{h-2})^\top$, where $\{\pi_\tau\}_{1\le\tau \le T}$ are the $T$ policies computed by the planner. Noting that $\rho_{h-2}$ is a mixture of these $T$ policies, we consider the following empirical and population quantities:
\begin{align*}
    \Sigma_{\rho_{h-2}} = \frac{1}{T} \sum_{t=1}^T \Sigma_{\pi_t}, \qquad \Sigma = \Sigma_{\rho_{h-2}} + \frac{1}{T}I_{d\times d}, \qquad \widehat \Sigma = \frac{1}{T} \Gamma_T  = \frac{1}{T}\sum_{i=1}^T \widehat \Sigma_{\pi_i} + \frac{1}{T} I_{d \times d}.
\end{align*}

Now, we use the termination conditions satisfied by the elliptic planner (shown in \pref{lem:FQI_ellip_planning}) in the following steps:
\begin{align}
    & \EE_{\pi_{h-2}}\sbr{\abr{\langle \hat \phi_{h-2}(x_{h-2},a_{h-2}), w_g\rangle}} \nonumber
    \\
    \le {} & \frac{1}{2\alpha} \EE_{\pi_{h-2}}\sbr{ \nbr{\hat \phi_{h-2}(x_{h-2},a_{h-2})}_{\widehat{\Sigma}^{-1}}^2} + \frac{\alpha}{2}\nbr{w_g}^2_{\widehat{\Sigma}}\label{eq:fqi_reward}
    \\
    \le {} & \frac{T\beta}{2\alpha} + \frac{\alpha}{2}\nbr{w_g}^2_{\widehat{\Sigma}} 
    \le {}  \frac{T\beta}{2\alpha} + \frac{\alpha}{2}\nbr{w_g}^2_{\Sigma} + \frac{\alpha}{2}\beta \|w_g\|^2_2 \label{eq:ex_planner}
    \\ 
    = {} & \frac{T\beta}{2\alpha} + \frac{\alpha}{2}\EE_{\rho_{h-2}} \sbr{\rbr{ \langle \hat \phi_{h-2}(x_{h-2},a_{h-2}), w_g\rangle}^2} + \frac{\alpha \|w_g\|^2_2}{2T} + \frac{\alpha \beta \|w_g\|^2_2}{2}. \notag
\end{align}
For the second inequality, note that $\tfrac{1}{T}\nbr{\hat \phi_{h-2}(x_{h-2},a_{h-2})}_{\widehat{\Sigma}^{-1}}^2$ is the reward function optimized by the offline elliptical planner in the last iteration. Let $v_T^\pi$ denote the expected return of the policy $\pi$ for this reward function and MDP $\Mcal$. From the termination condition and the results for the offline elliptical planner  in \pref{lem:FQI_ellip_planning}, we get
\begin{align*}
    \max_\pi v_T^{\pi} \le v_T^{\pi_T} + \beta/8 \le \hat v_T^{\pi_T} + \beta/4 \le \beta.
\end{align*}
Therefore, the first term on the RHS in \Cref{eq:fqi_reward} can be bounded by $T\beta/(2\alpha)$. In \Cref{eq:ex_planner}, we use the estimation guarantee for $\Sigma = \Gamma_T/T$ for the \fqi planner shown in \pref{lem:FQI_ellip_planning}. Then, in the last equality step, we expand the norm of $w_g$ using the definition of $\Sigma$ to arrive at the desired result.

Putting everything together, we now compute the number of samples used during elliptical planning for the required error tolerance. \pref{lem:FQI_ellip_planning} states that for a sample of size $n$, the computed policy is sub-optimal by a value difference of order upto $\tilde{O}\rbr{\sqrt{\frac{H^4d^6 \kappa K^2 \log( |\Phi|/\delta)}{n}}}$. Setting the failure probability of elliptical planning to be $\delta/(4H)$ for each level $h \in [H]$, and setting the planning error to $\beta/8$, we conclude that the total number of episodes used by \pref{alg:FQI_ell_planner} for each timestep $h$ is $\tilde{O}\rbr{\frac{H^4d^6\kappa K^2\log (|\Phi|/\delta)}{\beta^2}}$.
\end{proof}

Using \pref{lem:planner_result} in \Cref{eq:ex_transfer}, we get
\begin{align}
     \EE_{\pi} \sbr{f(x_h, a_h)} 
    \le {} & \frac{\alpha K^2}{2}\EE_{\rho_{h-2}} \sbr{\rbr{ \langle \hat \phi_{h-2}(x_{h-2},a_{h-2}), w_g\rangle}^2} + \frac{\beta K^2T}{2\alpha} + \frac{\alpha K^2\|w_g\|^2_2}{2T} \nonumber\\
    &~~~~+ \frac{\alpha \beta K^2\|w_g\|^2_2}{2} + \sqrt{\kappa K^5 \veps_{\mathrm{reg}}} \nonumber \\
    \le {} & \alpha K^2\EE_{\rho_{h-2}} \sbr{\rbr{ \langle \phi^*_{h-2}(x_{h-2},a_{h-2}), \theta^*_g\rangle}^2} + \alpha\kappa K^3 \veps_{\mathrm{reg}} + \frac{K^2T\beta}{2\alpha} + \frac{\alpha K^2\|w_g\|^2_2}{2T} \nonumber\\
    &~~~~+ \frac{\alpha \beta K^2\|w_g\|^2_2}{2} + \sqrt{\kappa K^5 \veps_{\mathrm{reg}}}. \label{eq:ex_transf2}
\end{align}
The second inequality uses the approximation guarantee for features $\hat \phi_{h-2}$ in \Cref{eq:apx_error_abr_dist} (derived from \Cref{eq:app_inv_phi}), the definition of $w_g$, and the inequality $(a+b)^2 \le 2a^2 + 2b^2$. 
Finally, we note that the inner product inside the expectation is always bounded between $[0,1]$ which allows use to use the fact that $f(x)^2 \le f(x)$ for $f: \Xcal \rightarrow [0,1]$. Substituting the upper bound for $\|w_g\|_2$, we get
\begin{align}
    &\EE_{\pi} \sbr{f(x_h, a_h)}\nonumber\\
    \le {} & \alpha K^2 \EE_{\rho_{h-2}} \sbr{\inner{ \phi_{h-2}^*(x_{h-2},a_{h-2})}{\theta^*_g}} + \alpha \kappa K^3 \veps_{\mathrm{reg}} + \sqrt{\kappa K^5 \veps_{\mathrm{reg}}}  \nonumber\\
    &~~~~+ \frac{\beta K^2T}{2\alpha} + \frac{\alpha \beta K^2 B^2}{2} + \frac{\alpha K^2 B^2}{2T} \nonumber \\
    = {} & \alpha K^2 \PP_{\rhoxnxt} \sbr{z_{h+1} = z} + \alpha \kappa K^3 \veps_{\mathrm{reg}} + \sqrt{\kappa K^5 \veps_{\mathrm{reg}}} + \frac{\beta K^2T}{2\alpha} + \frac{\alpha \beta K^2 B^2}{2} + \frac{\alpha K^2 B^2}{2T}. \label{eq:ex_latent}
\end{align}
\Cref{eq:ex_latent} follows by the definition of the function $g(\cdot)$.

We now set $\kappa \ge 2\alpha K^2$ in \Cref{eq:ex_latent}. Therefore, if we set the parameters $\alpha, \beta, \veps_{\mathrm{reg}}$ such that
\begin{align}
    \max \cbr{ \alpha \kappa K^3 \veps_{\mathrm{reg}} + \sqrt{\kappa K^5 \veps_{\mathrm{reg}}}, \frac{\beta K^2T}{2\alpha}, \frac{\alpha  \beta K^2 B^2 }{2} , \frac{\alpha  K^2 B^2}{2T }} \le \etamin/8,
    \label{eq:param_constraint_flo}
\end{align}
\Cref{eq:ex_latent} can be re-written as
\begin{align*}
    \max_{\pi} \PP_{\pi}\sbr{z_{h+1}=z} \le \frac{\kappa}{2} \PP_{\rhoxnxt} \sbr{z_{h+1} = z} + \frac{\etamin}{2} \le \kappa \PP_{\rhoxnxt} \sbr{z_{h+1} = z} 
\end{align*}
where in the last step, we use \pref{assum:reachability}. Hence, we prove the exploration guarantee in \pref{thm:explore_general_feat} by induction. 

To find the feasible values for the constants in \Cref{eq:param_constraint_flo}, we first note that $T \le 8d \log \rbr{1+8/\beta}/\beta$ (\pref{lem:FQI_ellip_planning}). We start by setting $\frac{\beta K^2 T}{2\alpha} = \etamin/8$ which gives $\alpha/T = \frac{4\beta K^2}{\etamin}$. Using the upper bound on $T$, we get $\alpha \le \frac{32dK^2\log \rbr{1+8/\beta}}{\etamin}$. Next, we set the term $\alpha \kappa K^3 \veps_{\mathrm{reg}} + \sqrt{\kappa K^5 \veps_{\mathrm{reg}}} \le \etamin/8$. Using the value of $\kappa = 2\alpha K^2$ we get
\begin{align*}
    2\alpha^2 K^5 \veps_{\mathrm{reg}} + \sqrt{2\alpha K^7 \veps_{\mathrm{reg}}} \le \etamin/8,
\end{align*}
which is satisfied by $\veps_{\mathrm{reg}} = \Theta \rbr{\frac{\etamin^3}{d^2 K^9 \log^2 \rbr{1+8/\beta}}}$.

Lastly, we will consider the term $\frac{\alpha \beta K^2 B^2}{2}$ and by setting it less than $\etamin/8$, we get
\begin{align*}
    \beta \log \rbr{ 1+8/\beta } \le \frac{\etamin^2}{128d B^2 K^4}.
\end{align*}

One can verify that under this condition we also have $\frac{\alpha K^2B^2}{2T}\le \etamin/8$, and setting $\beta = \tilde{O}\rbr{\frac{\etamin^2}{dB^2K^4}}$ satisfies the feasibility constraint for $\beta$. Here, we assume that $B$ only has a $\mathrm{polylog}$ dependence on $\beta$ and show later that this is true for all our feature selection methods. Notably, the only cases when $B$ depends on $\beta$ in our results is when $B = O\rbr{\frac{1}{\veps_{\mathrm{reg}}^c}}$ for a constant $c=\{1/2,1\}$ which has a $\log^2(1+8/\beta)$ term.  

Substituting the value of $\kappa$ and $\beta$ in \pref{lem:planner_result} with an additional factor of $H$ to account for all $h$ gives us the final sample complexity bound in \pref{thm:explore_general_feat}. The change of measure guarantee (\Cref{eq:dist_shift}) follows from the result in \pref{lem:dist_shift}. 

\subsubsection{Improved Sample Complexity Bound for Simplex Features} 
\label{sec:explore_simplex}
We can obtain more refined results when the agent instead has access to a latent variable feature class $\{\Psi_h\}_{h \in [H]}$
with $\psi_h : \Xcal \times \Acal \to \Delta(d_{\mathrm{LV}})$. We
call this the \emph{simplex features} setting \citep{agarwal2020flambe} and show the improved results in this section. For notation simplicity, we still use $\Phi_h$ and $\phi_h$ to represent the features. In order to achieve this improved result, we make two modifications to \textsc{Explore}: (i) We use a smaller discriminator function class $\Fcal_{h+1} \coloneqq \{f(x_{h+1}, a_{h+1}) = \EE_{\unif(\Acal)}[\phi_{h+1}(x_{h+1}, a_{h+1})[i]]: \phi_{h+1} \in \Phi_{h+1}, i \in [d_{\mathrm{LV}}] \}$ and (ii) in \textsc{Explore}, instead of calling the planner with learned features $\hat \phi_{h-2}$ and taking three uniform actions, we plan for the features $\hat \phi_{h-1}$ and add two uniform actions to collect data for feature learning in timestep $h$. The key idea here is that instead of estimating the expectation of any bounded function $f$, we only need to focus on the expectation of coordinates of $\phi^*$ as included in class $\Fcal_{h+1}$. Further, since $\phi^*_{h+1}[i]$ is already a linear function of the feature $\phi^*_{h+1}$, we take only one action at random at timestep $h$.

\begin{theorem}[Exploration with simplex features]
\label{thm:explore_simplex_feat}
Fix $\delta \in (0,1)$. Consider an MDP $\Mcal$ which admits a low-rank factorization with dimension $d$ in \pref{def:lowrank} and satisfies \pref{assum:reachability}. If \pref{assum:realizability} holds, the features $\hat \phi_h$ learned in \pref{line:learn_phi} in \pref{alg:explore} satisfy the condition in \Cref{eq:inv_phi} for $B \ge \sqrt{d}$, and $\veps_{\mathrm{reg}} = \tilde{\Theta}\rbr{\frac{\etamin^3}{d^2K^5 \log^2 (1+8/\beta)}}$, then with probability at least $1-\delta$, the sub-routine \textsc{Explore} collects an exploratory mixture policy $\rhoxa$ for each level $h$ such that
\begin{align}
\label{eq:app_dist_shift_simplex}
    \forall \pi: \, \EE_{\pi} [f(x_h,a_h)] \le \kappa K \EE_{\rhoxa} [f(x_h,a_h)]
\end{align}
for any $f: \Xcal \times \Acal \rightarrow \R^+$ and $\kappa = \frac{64dK^2\log \rbr{1+8/\beta}}{\etamin}$. The total number of episodes used in \pref{line:plan} by \pref{alg:explore} is
\begin{align*}
    \tilde{O} \rbr{\frac{H^5d^9K^{8}B^4 \log(|\Phi|/\delta)}{\etamin^5}}.
\end{align*}
$\beta$ is chosen such that $\beta \log \rbr{1+8/\beta} \le \frac{\etamin^2}{128dK^2B^2}$ and a sufficient one is $\beta =\tilde O\rbr{\frac{\etamin^2}{dK^4B^2}}$.
\end{theorem}
\begin{proof}
For simplex features, the key observation is that for any latent state $z \in \Zcal_{h+1}$, the function $f(x_h) = \EE_{\unif(\Acal)}\sbr{\PP[z_{h+1} = z|x_h,a_h]}$ is already a member of the discriminator function class $\Fcal_{h} \coloneqq \{f(x_h) = \EE_{\unif(\Acal)}\sbr{\phi_h(x_h, a_h)[i]}: \phi_h \in \Phi_h, i \in [d_{\mathrm{LV}}] \}$. Thus, when we rewrite the term $\EE_\pi[f(x_h,a_h)]$ as a linear function, we only need to backtrack one timestep to use the feature selection guarantee
\begin{align}
    \EE_{\pi} \sbr{f(x_h, a_h)} \le {} & K \EE_{\pi_{h-1} \circ \,\unif(\Acal)} \sbr{f(x_h, a_h)} = K \EE_{\pi_{h-1}} \sbr{g(x_{h-1},a_{h-1})} \nonumber \\
    \le {} & K \EE_{\pi_{h-1}} \sbr{\abr{\langle \hat \phi_{h-1}(x,a), w_g \rangle}} + \sqrt{\kappa K^3 \veps_{\mathrm{reg}}},
\end{align}
where we define $g(x_{h-1}, a_{h-1}) = \EE_{\unif(\Acal)} [f(x_h,a_h)| x_{h-1}, a_{h-1}].$
Therefore, the new value of $\kappa$ becomes $2\alpha K$ and by shaving off this $K$ factor in the chain of inequalities, we get the following constraint set for the parameters:
\begin{align}
    \max \cbr{ \alpha \kappa K^2 \veps_{\mathrm{reg}} + \sqrt{\kappa K^3 \veps_{\mathrm{reg}}}, \frac{\beta KT}{2\alpha}, \frac{\alpha \beta K B^2}{2} , \frac{\alpha K B^2}{2T }} \le \etamin/8.
\end{align}
Thus, the values of these parameters for the simplex features case are as follows:
\begin{align*}
    \frac{\alpha}{T} = \frac{4\beta K}{\etamin}, \quad \alpha \le \frac{32dK\log(1+8/\beta)}{\etamin}, \quad \veps_{\mathrm{reg}} = \tilde{\Theta}\rbr{\frac{\etamin^3}{d^2K^5 \log^2 (1+8/\beta)}}.
\end{align*}

Hence, the updated constraint for $\beta$ is
\begin{align*}
    \beta \log \rbr{1+8/\beta} \le \frac{\etamin^2}{64dB^2K^2}.
\end{align*}

Other than the values for these parameters, the algorithm remains the same. Therefore, substituting the new values of $\kappa$ and $\beta$ in the expression $\tilde{O} \rbr{\frac{H^4d^6\kappa K^2 \log( |\Phi|/\delta)}{\beta^2}}$ as before, we get the improved sample complexity result. 
\end{proof}

\subsection{Proofs for Representation Learning Guarantees for Downstream Tasks}
\label{sec:reward_free_planning}
We show that after obtaining the exploratory policies $\rhoxa$ for all $h \in [H]$ using \alg, we can collect a dataset $\Dcal$ to learn a feature $\bar \phi_h \in \Phi_h$ for all levels and use \fqi to plan for any reward function $R \in \Rcal$. Specifically, with min-max-min oracle or iterative greedy representation learning, we compute a feature $\bar \phi_h \in \Phi_h$ such that 
\begin{align}
    \label{eq:fqi_rep_error_app}
    \max_{g \in \Gcal_{h+1}} \min_{\|w\|_2 \le B} \EE_{\rhoxa} \sbr{\Big( \inner{\bar \phi_h(x_h,a_h)}{w} - \EE \sbr{ g(x_{h+1})\mid x_h,a_h}\Big)^2 }\le \veps_{\mathrm{apx}},
\end{align}
where $\Gcal_{h+1} \subseteq (\Xcal \rightarrow [0,H])$ is defined in \Cref{eq:def_Gcal_main}. For ease of discussion, we also present it here:
$\Gcal_{h+1}:=\Big\{\mathrm{clip}_{[0,H]}\Big(\max_{a}(R_{h+1}(x_{h+1},a)+\inner{\phi_{h+1}(x_{h+1},a)}{\theta})\Big):R  \in \Rcal, \phi_{h+1} \in  \Phi_{h+1}, \|\theta\|_2  \le  B \Big\}$, where $B \ge H\sqrt{d}.$
Also recall that $\Qcal(\bar \phi,R)$ as defined in \Cref{eq:linear_Q} is:
$\Qcal(\bar\phi,R):=\bigcup_{h\in[H]}
\Qcal_h(\bar\phi_h,R_h),\Qcal_h(\bar{\phi}_h,R_h) := \Big\{\mathrm{clip}_{[0,H]}(R_h(x_h,a_h) + \langle \bar \phi_h(x_h,a_h), w \rangle ) : \|w\|_2 \leq B\Big\}.$\\

The learned feature serves two purposes as discussed before:
\begin{itemize}
    \item (\emph{realizability}) The optimal value function for any timestep $h+1$ and reward $R_{h+1} \in \Rcal$, is defined as $V^*_{h+1}(x') = \max_a \big(R_{h+1}(x',a) + \EE[Q^*_{h+2}(\cdot)\mid x',a]\big) = \max_a \big(R_{h+1}(x',a) + \langle \phi^*_{h+1}, \theta^*_{h+1}\rangle\big)$. Thus, we have realizability as $V^*_{h+1} \in \Gcal_{h+1}$, which in turn implies that $\exists Q_h \in \Qcal_h(\bar\phi_h,R_h)$, s.t. $Q_h \approx R_h+\EE[V^*_{h+1}(\cdot)]$.
    \item (\emph{completeness}) For completeness, note that $\Gcal_{h+1}$ contains the Bellman backup of all possible $Q_{h+1}(\cdot)$ value functions we may encounter while running \fqi with $\Qcal(\bar \phi,R)$. Therefore, for any such $Q_{h+1}$, we have that $\exists Q_h \in \Qcal_h(\bar\phi_h,R_h)$, s.t. $Q_h \approx \mathcal{T}Q_{h+1}$.
\end{itemize}

\begin{proof}{\textbf{of \pref{thm:mf_fqi}}}
We run \fqi with the learned representation $\bar \phi_h$ using the value function class $\Qcal_h(\bar \phi_h, R_h)$ defined for each $h \in [H]$. \pref{lem:fqi_reward_class} shows that when \Cref{eq:fqi_rep_error_app} is satisfied with an error $\veps_{\mathrm{apx}}$, running \fqi using a total of $n_h =  \tilde{O}\rbr{\frac{H^6 d \kappa K \log(|\Rcal|B/\delta')}{\beta^2}}$ episodes collected from each exploratory policy $\{\rhoxa\}$ returns a policy $\hat \pi$ which satisfies
\begin{align*}
    \EE_{\hat{\pi}}\sbr{\sum_{h=0}^{H-1} R_h(x_h,a_h)} \geq \max_{\pi} \EE_\pi \sbr{\sum_{h=0}^{H-1}R_h(x_h,a_h)} - \beta - H^2\sqrt{\kappa K\veps_{\mathrm{apx}}}
\end{align*}
with probability at least $1-\delta'$. 

Then union bounding over all possible $\bar\phi$, and setting $\delta=\delta'/|\Phi|$, $\beta = \veps/2$, $\veps_{\mathrm{apx}} = \frac{\veps^2}{16H^4\kappa K}$, we get the final planning result with a value error of $\veps$ and probability at least $1-\delta$. Substituting $\kappa = \tilde{O}\rbr{\frac{32dK^4}{\etamin}}$, we get $n_h = \tilde{O} \rbr{\frac{H^6 d^2 K^5 \log (|\Phi||\Rcal|B/\delta )}{\veps^2 \etamin}}$. The final sample complexity is $\tilde{O} \rbr{\frac{H^7 d^2 K^5 \log(|\Phi||\Rcal|B/\delta)}{\veps^2 \etamin}}$, where we sum up the collected episodes across all levels.
\end{proof}

\subsection{Proofs for Oracle Representation Learning}\label{sec:flo_analysis}
In this section, we present the sample complexity result and the proof for \alg when a computational oracle \textsc{Flo} is available. Since we need to set $B \ge L\sqrt{d}$ in the min-max-min objective (\Cref{eq:emp_minmax_obj}), we assume \textsc{Flo} solves \Cref{eq:emp_minmax_obj} with $B=L\sqrt{d}$. The computational oracle is defined as follows:
\begin{definition}[Optimization oracle, \textsc{Flo}]
Given a feature class $\Phi_h$ and an abstract discriminator class $\Vcal \subseteq(\Xcal\rightarrow [0,L])$, we define the Feature Learning Oracle ($\textsc{Flo}$) as a subroutine that takes a dataset $\Dcal$ of tuples $(x_h,a_h,x_{h+1})$ and returns a solution to the following objective:
\begin{align}
\label{eq:flo_obj}
    \hat \phi_h = \argmin_{\phi \in \Phi_h} \max_{v \in \Vcal} \cbr{ \min_{\|w\|_2 \le L\sqrt{d}} \Lcal_{\Dcal}(\phi, w, v) - \min_{\tilde{\phi} \in \Phi_h, \|\tilde{w}\|_2 \le L\sqrt{d}} \Lcal_{\Dcal}(\tilde{\phi}, \tilde{w}, v)}.
\end{align}
\end{definition}

With this definition of \textsc{flo}, we will use the sample complexity result in \pref{lem:dev_bound_flo}, shown for the min-max-min objective (\Cref{eq:emp_minmax_obj}) against a general discriminator function class $\Vcal$ consisting of the set of functions 
\begin{align*}
v(x_{h+1}) = {} & \mathrm{clip}_{[0,L]} (\EE_{a_{h+1}
\sim \pi_{h+1}(x_{h+1})}[ R(x_{h+1},a_{h+1}) + \langle\phi_{h+1}(x_{h+1},a_{h+1}),\theta\rangle])
\end{align*}
where $\phi_{h+1} \in \Phi_{h+1}, \|\theta\|_2 \le L\sqrt{d}, R \in \Rcal$ for a prespecified policy $\pi_{h+1}$ over $x_{h+1}$. Note that, $\Fcal_{h}$ in the main text uses a singleton reward class $R(x_{h+1},a_{h+1}) = 0$ with $L=1$ and $\pi_{h+1} = \unif(\Acal)$. Similarly, $\Gcal_h$ uses $L=H$ with $\pi_{h+1}$ as the greedy arg-max policy. 

\paragraph{Sample Complexity of \alg with Min-Max-Min Oracle}
We now give a proof for the final sample complexity result for \alg as instantiated with the oracle \textsc{flo}.

\begin{proof}[\textbf{Proof of \pref{thm:flo_result}}]
Let us start with any fixed $h\in[H]$ and calculate the required number of samples per level. 

Firstly, we consider learning $\hat \phi_h$ that satisfies \Cref{eq:inv_phi}. We use the discriminator class $\Vcal=\Fcal$ as defined in \Cref{eq:def_Fcal_clipped} and set $B=\sqrt d$. Then applying \pref{lem:dev_bound_flo} with $L=1$, we know that condition \Cref{eq:inv_phi} holds with probability at least $1-\delta/(4H)$, if
\begin{align*}
    n \ge \frac{16\cc d^2 \log( 2n\sqrt d|\Phi_h||\Phi_{h+1}||\Rcal|/(\delta/4H))}{\veps_{\mathrm{reg}}},
\end{align*}
where $\cc$ is the constant in \pref{lem:fqi_fastrate_unclipped}.

Setting $\veps_{\mathrm{reg}} = \tilde{\Theta}\rbr{\frac{\etamin^3}{d^2K^9 \log^2 (1+8/\beta)}}$ and noting $\beta = \tilde O\rbr{\frac{\etamin^2}{dK^4B^2}}$, we get
\begin{align*}
    n_{\rephat} = \tilde{O}\rbr{\frac{d^2\log(|\Phi_h||\Phi_{h+1}||\Rcal|/\delta)}{\veps_{\mathrm{reg}}}} = \tilde{O}\rbr{ \frac{d^4 K^9 \log(|\Phi|/\delta)}{\etamin^3} }.
\end{align*}

Substituting the value $B = \sqrt{d}$ in \pref{thm:explore_general_feat}, we know that we can get an exploratory dataset with probability at least $1-\delta/(4H)$ and the corresponding sample complexity for the elliptic planner is
\begin{align*}
    n_{\mathrm{ell}} = \tilde{O} \rbr{\frac{H^5d^9K^{14}B^4 \log (|\Phi|/\delta)}{\etamin^5}} = \tilde{O}\rbr{\frac{H^5d^{11}K^{14} \log (|\Phi|/\delta)}{\etamin^5}}
\end{align*}

Then we consider learning $\bar\phi_h$ that satisfies \Cref{eq:fqi_rep_error}. We use the discriminator class $\Vcal=\Gcal$ as defined in \Cref{eq:def_Gcal_main} and set $B=\sqrt d$. Noticing that $\kappa = \frac{64dK^4\log \rbr{1+8/\beta}}{\etamin}$ and $\beta$ is a polynomial term, we know that $\frac{\veps^2}{16H^4 \kappa K}= \tilde{O}\rbr{\frac{\veps^2 \etamin}{dH^4 K^5}}$.  Setting $\veps_{\mathrm{apx}}=\frac{\veps^2}{16H^4 \kappa K}$ and applying \pref{lem:dev_bound_flo} with $L=H$, we have that condition \Cref{eq:fqi_rep_error} is satisfied with probability at least $1-\delta/(4H)$ if
\begin{align*}
    n_{\repbar} = \tilde{O}\rbr{ \frac{H^6 d^3 K^5 \log \rbr{|\Phi||\Rcal|/\delta}}{\veps^2 \etamin} }.
\end{align*}

Notice that \Cref{eq:fqi_rep_error} holds and we collect an exploratory dataset by applying \pref{thm:explore_general_feat}. Then \pref{thm:mf_fqi} implies the required sample complexity for offline FQI planning with $\bar \phi_{0:H-1}$ to learn an $\veps$-optimal policy with probability at least $1-\delta/(4H)$ is
\begin{align*}
       n_{\downstream} =  \tilde{O}\rbr{\frac{H^6 d^2 K^5 \log (|\Phi||\Rcal|/\delta)}{\veps^2 \etamin}}.
\end{align*}

Union bounding over $h\in[H]$, the final sample complexity is $H(n_{\rephat} + n_{\mathrm{ell}} + n_{\repbar}+ n_{\downstream})$, and the result holds with probability  $1-\delta$. Reorganizing terms completes the proof.
\end{proof}

\subsection{Proofs for Iterative Greedy Representation Learning Method}
\label{sec:greedy_analysis}
We start by showing the main iteration complexity result and a feature selection guarantee for \pref{alg:greedy_selection} below.

\begin{lemma}[Iteration complexity for \pref{alg:greedy_selection}]
\label{lem:dev_bound_greedy_final}
Fix $\delta \in (0,1)$. If the iterative greedy feature selection algorithm (\pref{alg:greedy_selection}) is run with a sample $\Dcal$ of size $n = \tilde{O}\rbr{\frac{L^6d^7 \log (|\Phi_{h}||\Phi_{h+1}||\Rcal|/\delta)}{\veps^3_{\mathrm{tol}}}}$, then with $B = \sqrt{\frac{13L^4d^3}{\veps_{\mathrm{tol}}}}$, it terminates after $T = \frac{52L^2d^2}{\veps_{\mathrm{tol}}}$ iterations and returns a feature $\hat \phi_h$ such that for $\Vcal \subseteq (\Xcal \rightarrow [0,L])$, $\Vcal\coloneqq \{v(x_{h+1}) = \mathrm{clip}_{[0,L]} (\EE_{a_{h+1}
\sim \pi_{h+1}(x_{h+1})}[R_{h+1}(x_{h+1},a_{h+1}) +$ $\langle\phi_{h+1}(x_{h+1},a_{h+1}),\theta\rangle]): \phi_{h+1} \in \Phi_{h+1}, \|\theta\|_2 \le L\sqrt{d}, R \in \Rcal \}$, where policy $\pi$ is the greedy policy or the uniform policy, we have
\begin{align*}
    \max_{v \in \Vcal} \be{\rhoxa, \hat \phi_h, v; B} \le \veps_{\mathrm{tol}}.
\end{align*}
\end{lemma}
\begin{proof}
For ease of notation, we will not use the subscript $\rhoxa$ in the expectations below ($\Lcal(\cdot) \coloneqq \Lcal_{\rhoxa}(\cdot)$). Similarly, we will use $\phi_t$ to denote feature $\phi_{t,h}(x_h,a_h)$ of iteration $t$ and $(x',a')$ for $(x_{h+1}, a_{h+1})$ unless required by context. Further, for any iteration $t$, let $W_t = [w_{t,1} \mid w_{t,2} \mid \ldots \mid w_{t,t}] \in \R^{d \times t}$ be the matrix with columns $W_t^i$ as the linear parameter $w_{t,i} = \argmin_{\|w\|_2 \le L\sqrt{d}} \Lcal_{\Dcal}(\hat \phi_{t,h}, w, v_i)$. Similarly, let $A_t = [\theta^*_1 \mid \theta^*_2 \mid \ldots \mid \theta^*_t]$.

In the proof, we assume that the total number of iterations $T$ does not exceed $\frac{52L^2d^2}{\veps_{\mathrm{tol}}}$ and set parameters accordingly. We later verify that this assumption holds. Further, let $\tilde{\veps} = \frac{\veps^2_{\mathrm{tol}}}{2704 L^2 d^3}$ and $\veps_0 = T_{\max} \cdot \tilde{\veps} = \frac{\veps_{\mathrm{tol}}}{52d}$.

To begin, based on the deviation bound in \pref{lem:dev_bound_greedy}, we note that if the sample $\Dcal$ in \pref{alg:greedy_selection} is of size $n = \tilde{O}\rbr{\frac{L^6d^7 \log (|\Phi_{h}||\Phi_{h+1}||\Rcal|/\delta)}{\veps^3_{\mathrm{tol}}}}$ and the termination loss cutoff is set to $3\veps_1/2 + \tilde{\veps}$ such that, with probability at least $1-\delta$, for all non-terminal iterations $t$ we have
\begin{gather}
    \sum_{v_i \in \Vcal^t} \EE \sbr{\rbr{ \hat \phi_t^\top W_t^i - \phi^{*\top} A_t^i}^2} \le {}  t \tilde{\veps} \le \veps_0, \label{eq:greedy_reg} \\
    \label{eq:testloss_lowbnd}
    \EE \sbr{\rbr{ \hat \phi_t^\top w - \phi^{*\top} \theta^*_{t+1}}^2} \ge \veps_1
\end{gather}
where $\tilde{\veps}$ is an error term dependent on the size of $\Dcal$ and $w$ is any vector with $\|w\|_2 \le B_t \le B_T\le B$. Further, when the algorithm does terminate, we get the loss upper bound to be $3\veps_1 + 4\tilde{\veps}$. 

Using ~\Cref{eq:greedy_reg} and \Cref{eq:testloss_lowbnd}, we will now show that the maximum iterations in \pref{alg:greedy_selection} is bounded.
At round $t$, for functions $v_1,\ldots,v_t \in \Vcal$ in \pref{alg:greedy_selection}, let $\theta^*_i = \theta^*_{v_i}$ as before and further let $\Sigma_t = A_tA_t^\top + \lambda I_{d \times d}$. Using the linear parameter $\theta^*_{t+1}$ of the adversarial test function $v_{t+1}$, define $\hat{w}_t = W_t A_t^\top \Sigma_t^{-1} \theta^*_{t+1}$. For this $\hat w_t$, we can bound its norm as
\begin{align}
\label{eq:sqrt_t_grow}
    \|W_t A_t^\top \Sigma_t^{-1} \theta^*_{t+1}\|_2 \le \|W_t\|_2 \|A_t^\top \Sigma_t^{-1}\|_2 \|\theta^*_{t+1}\|_2 \le L^2d\sqrt{\frac{t}{4\lambda}}.
\end{align}
Here $\|W_t\|_2 \le L\sqrt{dt}$ and $\|\theta^*_{t+1}\|_2 \le L\sqrt{d}$. Applying SVD decomposition and the property of matrix norm, $\|A_t^\top \Sigma_t^{-1}\|_2$ can be upper bounded by $\max_{i \le d} \frac{\sqrt{\lambda_i}}{\lambda_i + \lambda} \le \frac{1}{\sqrt {4\lambda}}$, where $\lambda_i$ are the eigenvalues of $A_t A_t^\top$. Then noticing AM-GM inequality, we get $\|A_t^\top \Sigma_t^{-1}\|_2\le\sqrt{1/4\lambda}$. 

Setting $B_t = L^2d\sqrt{\frac{t}{4\lambda}}$, from \Cref{eq:testloss_lowbnd}, we have
\begin{align*}
    \veps_1 \leq {} & \EE \sbr{\rbr{ \hat \phi_t^\top \hat{w}_t - \phi^{*\top} \theta^*_{t+1}}^2} 
    \\
    = {}& \EE \sbr{\rbr{\hat{\phi}_t^\top W_t A_t^\top \Sigma_t^{-1} \theta_{t+1}^* - \phi^*{}^\top \Sigma_t \Sigma_t^{-1}\theta_{t+1}^*}^2}
    \\
    \leq {} & \| \Sigma_t^{-1} \theta_{t+1}^*\|_2^2 \cdot \EE \sbr{\| \hat{\phi}_t^\top W_t A_t^\top  - \phi^*{}^\top \Sigma_t\|_2^2}
    \\
    \leq {} & 2\| \Sigma_t^{-1} \theta_{t+1}^*\|_2^2 \cdot  \EE \sbr{\| \hat{\phi}_t^\top W_t A_t^\top  - \phi_t^\top A_t A_t^\top \|_2^2 + \lambda^2 \|\phi^*{}^\top\|_2^2} \\
    \leq {} & 2 \| \Sigma_t^{-1} \theta^*_{t+1}\|_2^2 \cdot \left(\sigma_1^2(A_t) \EE\sbr{\|\hat{\phi}_t^\top W_t - \phi^{*}{}^{\top} A_t\|_2^2} + \lambda ^2\right)
    \leq {}  2 \|\Sigma_t^{-1} \theta^*_{t+1}\|_2^2 \cdot\left( L^2dt\veps_0 + \lambda^2\right).
\end{align*}
The second inequality uses Cauchy-Schwarz. The last inequality applies the upper bound $\sigma_1(A_t) \le L\sqrt{dt}$ and the guarantee from \Cref{eq:greedy_reg}. Using the fact that $t \leq T$, this implies that
\begin{align*}
    \|\Sigma_t^{-1}\theta^*_{t+1}\|_2 \geq \sqrt{ \frac{\veps_1}{2(L^2dT\veps_0 + \lambda^2)}}.
\end{align*}

We now use the generalized elliptic potential lemma from \citet{alex2020elliptical} to upper bound the total value of $\|\Sigma_t^{-1}\theta^*_{t+1}\|_2$. From \pref{lem:gen_elliptic_pot} in \pref{app:gen_elliptic_pot}, if $\lambda \ge L^2d$ and we do not terminate in $T$ rounds, then
\begin{align*}
    T\sqrt{ \frac{\veps_1}{2(L^2dT\veps_0 + \lambda^2)}} \leq \sum_{t=1}^T \|\Sigma_t^{-1}\theta^*_{t+1}\|_2 \leq 2 \sqrt{\frac{Td}{\lambda}}.
\end{align*}

From this chain of inequalities, we can deduce
 $T\veps_1 \leq 8 (d/\lambda)  \left( L^2dT\veps_0 + \lambda^2\right),$
therefore
$T \leq \frac{8d\lambda}{\veps_1 - 8 L^2d^2\veps_0/\lambda}.$ 
Now, if we set $\veps_1 = 16L^2d^2\veps_0/\lambda$ in the above inequality, we can deduce 
\begin{align*}
    T \leq \frac{\lambda^2}{L^2d\veps_0}.
\end{align*}

Putting everything together, for input parameter $\veps_{\mathrm{tol}}$, the termination threshold for the loss $l$ is set such that  $\frac{48L^2d^2\veps_0}{\lambda} + \frac{4L^2d\veps_0^2}{\lambda^2} \le \veps_{\mathrm{tol}}$ which is satisfied for $\veps_0 = \frac{\lambda \veps_{\mathrm{tol}}}{52L^2d^2}$. In addition, with $\lambda=L^2d$, we set the constants for \pref{alg:greedy_selection} as follows: 
\begin{align*}
    T \le \frac{52L^2d^2}{\veps_{\mathrm{tol}}}, \qquad \veps_0 = \frac{\veps_{\mathrm{tol}}}{52d}, \qquad B_t \coloneqq \sqrt{\frac{L^2dt}{4}}, \qquad B \coloneqq \sqrt{\frac{13L^4d^3}{\veps_{\mathrm{tol}}}}.
\end{align*}
Further, for \pref{lem:dev_bound_greedy}, we set $\tilde{\veps}$ to $\veps_0/T =  O\rbr{\frac{\veps^2_{\mathrm{tol}}}{L^2d^3}}$. Note that from \pref{lem:dev_bound_greedy}, the loss upper bound is $3\veps_1 + 4\tilde{\veps}$ when the algorithm terminates. By our choice of the parameters, we can verify that $3\veps_1 + 4\tilde{\veps}\le\veps_{\mathrm{tol}}$ and $T$ does not exceed $\frac{52L^2d^2}{\veps_{\mathrm{tol}}}$, which completes the proof.
\end{proof}

\paragraph{Sample Complexity of \alg with Iterative Greedy Representation Learning}
With the feature selection guarantee in \pref{lem:dev_bound_greedy_final}, we can now finish the proof for the final sample complexity result for the greedy iterative algorithm.

\begin{proof}[\textbf{Proof of \pref{thm:greedy_result}}]
Let us start with any fixed $h\in[H]$ and calculate the required number of samples per level. 

Firstly, we consider learning $\hat \phi_h$ that satisfies \Cref{eq:inv_phi}. We use the discriminator class $\Vcal=\Fcal$ as defined in \Cref{eq:def_Fcal_clipped} and set $B = \sqrt{\frac{13d^3}{\veps_{\mathrm{reg}}}}=\tilde O\rbr{\sqrt{\frac{d^5K^9}{\etamin^3}}}$ (since $\veps_{\mathrm{reg}} =  \tilde{\Theta}\rbr{\frac{\etamin^3}{d^2K^9 \log^2 (1+8/\beta)}}$ and $\beta$ is a polynomial term). Applying \pref{lem:dev_bound_greedy_final}, we know that for an approximation error of $\veps_{\mathrm{tol}}$, we need to set the sample size to $n = \tilde{O}\rbr{ \frac{L^6d^7 \log (|\Phi||\Rcal|/\delta)}{\veps^3_{\mathrm{tol}}}}$.

Setting the values of the parameter $\veps_{\mathrm{tol}}=\veps_{\mathrm{reg}} = \tilde{\Theta}\rbr{\frac{\etamin^3}{d^2K^9 \log^2 (1+8/\beta)}}$ (according to \pref{thm:explore_general_feat}) and $L=1$ in \pref{lem:dev_bound_greedy_final}, we get the number of episodes for learning $\hat \phi_h$ that satisfies \Cref{eq:inv_phi} with probability at least $1-\delta/(4H)$ is
\begin{align*}
    n_{\rephat} = \tilde{O}\rbr{ \frac{L^6d^7 \log (|\Phi||\Rcal|/\delta)}{\veps^3_{\mathrm{reg}}}} = \tilde{O}\rbr{ \frac{d^{13}K^{27} \log (|\Phi|/\delta)}{\etamin^9}}.
\end{align*}

Substituting the value $B = \sqrt{\frac{13d^3}{\veps_{\mathrm{reg}}}}=\tilde O\rbr{\sqrt{\frac{d^5K^9}{\etamin^3}}}$ in \pref{thm:explore_general_feat}, we know that we can get an exploratory dataset with probability at least $1-\delta/(4H)$ and the corresponding sample complexity for the elliptic planner is
\begin{align*}
    n_{\mathrm{ell}} = \tilde{O} \rbr{\frac{H^5d^9K^{14}B^4 \log (|\Phi|/\delta)}{\etamin^5}} = \tilde{O}\rbr{\frac{H^5d^{19}K^{32} \log (|\Phi|/\delta)}{\etamin^{11}}}.
\end{align*}

Next, we consider learning $\bar \phi_h$ that satisfies \Cref{eq:fqi_rep_error}. Noticing that $\kappa = \frac{64dK^4\log \rbr{1+8/\beta}}{\etamin}$ and $\beta$ is a polynomial term, we have that $\frac{\veps^2}{16H^4 \kappa K}= \tilde{O}\rbr{\frac{\veps^2 \etamin}{dH^4 K^5}}$. Setting $\veps_{\mathrm{tol}}=\veps_{\mathrm{apx}} =\frac{\veps^2}{16H^4 \kappa K}$ and applying \pref{lem:dev_bound_greedy_final} with $L=H$, we know that if
\begin{align*}
    n_{\repbar} = \tilde{O}\rbr{ \frac{L^6d^7 \log (|\Phi||\Rcal|/\delta)}{\veps^3_{\mathrm{apx}}}} = \tilde{O}\rbr{ \frac{H^{18}d^{10}K^{15} \log (|\Phi||\Rcal|/\delta)}{\veps^6 \etamin^3}},
\end{align*}
then condition \Cref{eq:fqi_rep_error} is satisfied with probability at least $1-\delta/(4H)$.

Notice that \Cref{eq:fqi_rep_error} holds and we collect an exploratory dataset by applying \pref{thm:explore_general_feat}. Then \pref{thm:mf_fqi} implies the required sample complexity for offline FQI planning with $\bar \phi_{0:H-1}$ to learn an $\veps$-optimal with probability at least $1-\delta/(4H)$ is
\begin{align*}
       n_{\downstream} =  \tilde{O}\rbr{\frac{H^6 d^2 K^5 \log (|\Phi||\Rcal|/\delta)}{\veps^2 \etamin}}.
\end{align*}

Union bounding over $h\in[H]$, the final sample complexity is $H(n_{\rephat} + n_{\mathrm{ell}} + n_{\repbar}+ n_{\downstream})$, and the result holds with probability $1-\delta$. Reorganizing terms completes the proof.
\end{proof}

\subsection{Proofs for Enumerable Representation Class}
\label{sec:enumerable_analysis}
We first derive the ridge regression based reduction of the min-max-min objective to eigenvector computation problems. Recall that for the enumerable feature class, we solve the following modified objective (\Cref{eq:ridge_objective}) in \pref{alg:explore}
\begin{align*}
    \argmin_{\phi \in \Phi_h}  \max_{\substack{f \in \Fcal_{h+1},\tilde{\phi} \in \Phi_{h},\|\tilde{w}\|_2 \le B  }}  \cbr{ \min_{\|w\|_2 \le B}  \Lcal_{\Dcal_h}(\phi, w, f) - \Lcal_{\Dcal_h}(\tilde{\phi}, \tilde{w}, f)}
\end{align*}
where $\Fcal_{h+1}$ is now the discriminator class that contains all \emph{unclipped} functions $f$ in form of
\begin{align*}
f(x_{h+1}) = \EE_{\unif(\Acal)} \sbr{\inner {\phi_{h+1}(x_{h+1}, a)}{\theta}}, \text{ for }  \phi_{h+1} \in \Phi_{h+1}, \|\theta\|_2 \le \sqrt{d}.
\end{align*}

Consider the min-max-min objective and fix $\phi, \tilde{\phi} \in \Phi_h$. Rewriting the objective for a sample of size $n$, we get the following updated objective:
\begin{align*}
    \max_{f \in \Fcal_{h+1}} \min_{\|w\|_2 \le \sqrt{d}} \|X w - f(\Dcal_h)\|_2^2 - \min_{\|\tilde{w}\|_2 \le \sqrt{d}} \|\tilde{X} \tilde{w} - f(\Dcal_h)\|_2^2
\end{align*}
where $X, \tilde{X} \in \R^{n \times d}$ are the covariate matrices for features $\phi$ and $\tilde{\phi}$ respectively. 

We overload the notation and use $f(\Dcal_h) \in \R^n$ to denote the value of any $f \in \Fcal_{h+1}$ on the $n$ samples. Now, instead of solving the constrained least squares problem, we use a ridge regression solution with regularization parameter $\lambda$. Thus, for any target $f$ in the min-max objective, for feature $\phi$, we get
\begin{gather*}
    w_f =  \rbr{\tfrac{1}{n}X^\top X + \lambda I_{d\times d}}^{-1}\rbr{\tfrac{1}{n} X^\top f(\Dcal_h)} \\
    \| X w - f(\Dcal_h) \|_2^2 = {}  \left\| X \rbr{\tfrac{1}{n}X^\top X + \lambda I_{d\times d}}^{-1}\rbr{\tfrac{1}{n} X^\top f(\Dcal_h)} - f(\Dcal_h)\right\|_2^2 = \|A(\phi) f(\Dcal_h)\|_2^2
\end{gather*}
where $A(\phi) = I_{n\times n} - X \rbr{\tfrac{1}{n}X^\top X + \lambda I_{d \times d}}^{-1}\rbr{\tfrac{1}{n} X^\top}$. 

Similarly, for the feature $\tilde \phi$, we have
\[\| \tilde X \tilde w - f(\Dcal_h) \|_2^2 = \|A(\tilde \phi) f(\Dcal_h)\|_2^2,
\]
where $A(\tilde \phi) = I_{n\times n} - \tilde X \rbr{\tfrac{1}{n}\tilde X^\top  \tilde X + \lambda I_{d \times d}}^{-1}\rbr{\tfrac{1}{n} \tilde X^\top}$.

In addition, any regression target $f$ can be rewritten as $f = X'\theta$ for a feature $\phi' \in \Phi_{h+1}$ and $\|\theta\|_2 \le \sqrt{d}$. Thus, for a fixed $\phi'$, $\phi$ and $\tilde{\phi}$, the maximization problem for $\Fcal_{h+1}$ is the same as
\begin{align}
    \max_{\|\theta\|_2 \le \sqrt{d}} \theta^\top X'^\top \rbr{A(\phi)^{\top} A(\phi) -  A(\tilde{\phi})^{\top} A(\tilde{\phi})} X' \theta.
\end{align}
where $X' \in \R^{n \times d}$ is again the sample matrix defined using $\phi' \in \Phi_{h+1}$. 

For each tuple of $(\phi, \tilde{\phi}, \phi')$, the maximization problem reduces to an eigenvector computation. As a result, we can efficiently solve the min-max-min objective in \Cref{eq:ridge_objective} by enumerating over each candidate feature in $(\phi, \tilde{\phi}, \phi')$ to solve 
\begin{equation}
    \argmin_{\phi\in\Phi_h}\max_{\substack{\tilde{\phi} \in \Phi_h, \phi' \in \Phi_{h+1}, \|\theta\|_2 \le \sqrt{d}}} \theta^\top X'^\top \rbr{A(\phi)^{\top} A(\phi) -  A(\tilde{\phi})^{\top} A(\tilde{\phi})} X' \theta.
    \label{eq:ev-enum_pf}
\end{equation}

\paragraph{Sample Complexity of \alg for the Enumerable Feature Class}
We now prove the sample complexity result.
\begin{proof}[\textbf{Proof of \pref{thm:enumerable_result}}]
Let us start with any fixed $h\in[H]$ and calculate the required number of samples per level.

Firstly, we consider learning $\hat \phi_h$ that satisfies \Cref{eq:inv_phi}. We use the discriminator class $\Vcal=\Fcal$ as defined in \Cref{eq:def_Fcal_unclipped} and the error threshold $\veps_{\mathrm{reg}}$. Setting the values of the parameter $\veps_{\mathrm{reg}} = \tilde{\Theta}\rbr{\frac{\etamin^3}{d^2K^9 \log^2 (1+8/\beta)}}$ (according to \pref{thm:explore_general_feat}) in \pref{lem:ridge_error} and noting $\beta $ is a polynomial term, we get the number of episodes for learning $\hat \phi_h$ that satisfies \Cref{eq:inv_phi} with probability at least $1-\delta/(3H)$ is
\begin{align*}
    n_{\rephat} = \tilde{O}\rbr{\frac{d^6\log^3(|\Phi_h||\Phi_{h+1}|/\delta)}{\veps^3_{\mathrm{reg}}}} = \tilde{O}\rbr{ \frac{d^{12} K^{27} \log^3(|\Phi|/\delta)}{\etamin^9} }.
\end{align*}

Now, substituting the value $B =1/\lambda= \tilde{\Theta}\rbr{n^{1/3}_{\rephat}}=\tilde{\Theta}\rbr{\frac{d^4K^9\log\rbr{|\Phi|/\delta}}{\etamin^3}}$ in \pref{thm:explore_general_feat}, we know that we can get an exploratory dataset with probability at least $1-\delta/(3H)$ and the corresponding sample complexity for the elliptic planner is
\begin{align*}
    n_{\mathrm{ell}} = \tilde{O} \rbr{\frac{H^5d^9K^{14}B^4 \log (|\Phi|/\delta)}{\etamin^5}} = \tilde{O}\rbr{\frac{H^5d^{25}K^{50} \log^5 (|\Phi|/\delta)}{\etamin^{17}}}.
\end{align*}

Finally, using \pref{corr:fqi_full_class} from \pref{app:fqi_full_class} and noticing $\kappa = \frac{64dK^4\log \rbr{1+8/\beta}}{\etamin}$, the number of episodes collected for running \fqi with $\Qcal(R)$ to learn an $\veps$-optimal policy with probability at least $1-\delta/(3H)$ can be bounded by
\begin{align*}
    n_{\downstream} = \tilde{O}\rbr{\frac{H^6d^2\kappa K \log(|\Phi||\Rcal|/\delta)}{\veps^2}} = \tilde{O}\rbr{\frac{H^6d^3 K^5 \log(|\Phi||\Rcal|/\delta)}{\veps^2\etamin}}.
\end{align*}

Union bounding over $h\in[H]$, the final sample complexity is  $H(n_{\rephat} + n_{\mathrm{ell}} + n_{\downstream})$, and the result holds with probability at least $1-\delta$. Reorganizing terms completes the proof.
\end{proof}

\section{Conclusion}
\label{sec:conclusion}
In this paper, we present \alg, a new model-free algorithm, for representation learning and exploration in low-rank MDPs. We develop several representation learning schemes that vary in their computational and statistical properties, each yielding a different instantiation of the overall algorithm. Importantly \alg can leverage a general function class $\Phi$ for representation learning, which provides it with the expressiveness and flexibility to scale to rich observation environments in a provably sample-efficient manner. 

\section*{Acknowledgements}
Part of this work was done while AM was at University of Michigan and was supported in part by a grant from the Open Philanthropy Project to the Center for Human-Compatible AI, and in part by NSF grant CAREER IIS-1452099. JC would like to thank Kefan Dong for helpful discussions related to Bernstein's version of uniform deviation bounds. NJ acknowledges funding support from ARL Cooperative Agreement W911NF-17-2-0196, NSF IIS-2112471, NSF CAREER IIS-2141781, and Adobe Data Science Research Award.

\clearpage
\bibliography{arxiv-jmlr.bib}

\clearpage
\tableofcontents
\clearpage
\appendix

\section{Comparisons Among the Closely Related Works}
\label{app:comp}
In this section, we provide more details about comparisons with 
\textsc{Olive}\xspace, \textsc{Witness rank}\xspace, and \textsc{BLin-Ucb}\xspace. They are statistically efficient for more general settings beyond low-rank MDPs. 
Strictly speaking, their realizability assumptions and sample complexity terms are different from what we present in \Cref{table:comparison_table}. \textsc{Olive}\xspace requires the realizability of the value function class $Q^*\in\Fcal^{\textsc{class}}$ and has $\log(|\Fcal^{\textsc{class}}|)$ dependence. \textsc{Witness rank}\xspace requires the realizability of the model class $\Mcal^*\in\Mcal^{\textsc{class}}$ and an induced value function $\Fcal^{\textsc{class}}$ class from $\Mcal^{\textsc{class}}$, and consequently pays $\log(|\Fcal^{\textsc{class}}||\Mcal^{\textsc{class}}|)$. \textsc{BLin-Ucb}\xspace makes a more complicated realizability assumption on a hypothesis function class $\Hcal^{\textsc{class}}$, whose complexity $\log(|\Hcal^{\textsc{class}}|)$ shows up on the bound (please refer to \citet{du2021bilinear} for more details). Here we add the superscript $\textsc{class}$ to function classes to differentiate them from the notations in other parts of the paper.

For the purpose of comparison, we instantiate their sample complexity bounds in our setting. We design function classes $\Hcal^{\textsc{class}}=\Fcal^{\textsc{class}}=\Fcal_0^{\textsc{class}}\times\ldots \times\Fcal_{H-1}^{\textsc{class}} $, where $\Fcal_h^{\textsc{class}} = \big\{ f_h(x_h,a_h) = R_h(x_h,a_h)+\inner{\phi_h(x_h,a_h)}{\theta_h} : \phi_h \in \Phi_h, \|\theta_h\|_2 \le \sqrt{d} \big\}$. For \textsc{Witness rank}\xspace, we additionally construct $\Mcal^{\textsc{class}}=\{\langle\phi,\mu\rangle:\phi\in\Phi,\mu\in\Upsilon\}$. The sample complexities are then obtained by calculating the complexity of these function classes and multiplying an $H^2$ factor to translate the results from the bounded total reward setting ($0\le \sum_{h=0}^{H-1} r_h\le 1$) in \citet{jiang2018open} to our uniformly bounded reward setting ($r_h\in[0,1],\forall h\in[H]$). 

\section{The Analysis of Elliptical Planner} 
\label{app:FQI_ellip_planning}
In this section, we show the iteration and sample complexities and the estimation guarantee for \emph{offline} ``elliptical planner'' (\pref{alg:FQI_ell_planner}). The algorithm and analysis follows a similar approach as Algorithm 2 in \citet{agarwal2020flambe}, while the major difference here is that we call \fqi for the policy optimization step because we do not have the model. In addition, we cannot directly estimate the covariance matrix by sampling data from the estimated model as in \citet{agarwal2020flambe}.

Inspired by \citet{huang2021towards}, we perform Fitted Q-Evaluation (FQE) with the exploratory data in the prior levels to substitute the Monte Carlo estimation counterpart used in the online ``elliptical planner''. Different from \citet{huang2021towards}, we no longer run our algorithm on the discretized value functions and rewards. In contrast, we directly use the original elliptical reward and perform \fqi and \fqe on the original continuous function class. This makes the algorithm more computationally handy. 

The detailed algorithm is shown in \pref{alg:FQI_ell_planner}. The algorithm proceeds in iterations. In each round, we first use the current covariance matrix $\Gamma_{t-1}$ to set the elliptical reward \pref{line:set_elliptical_reward}. Then in \pref{line:fqi} we call \fqi (\pref{alg:linear_fqi}) to get policy $\pi_t$ that explores the uncovered direction set by the elliptical reward. Next, we call \fqe (\pref{alg:linear_fqe}) to estimate the covariance matrix $\widehat \Sigma_{\pi_t}$ (more specifically, each $(i,j)$-th coordinate, $i,j\in \{1,\ldots,d\}$ in the covariance matrix respectively) for all policy $\pi_t$ (\pref{line:est_cov_mat}-12). The covariance matrix $\Gamma_t$ is updated in \pref{line:update_cov}. \fqe is also used to the expected return (\pref{line:fqe_expected_return}) to check the stopping condition.

\begin{algorithm}[htb]
\begin{algorithmic}[1]
	\STATE \textbf{input:} Features $\hat\phi$, exploratory dataset $\Dcal \coloneqq \Dcal_{0:\tilde{H}}$ with size $n$ at each level $h\in[\tilde H]$, and threshold $\beta > 0$.
	\STATE Initialize $\Gamma_0 = I_{d \times d}$.
	\FOR{$t = 1,2,\ldots,$}
	\STATE Define the elliptical reward $R^{\fqi,t}$ as $R_{\tilde H}^{\fqi,t}=\left\|\hat \phi_{\tilde{H}}\right\|^2_{\Gamma^{-1}_{t-1}}$ and $R_h^{\fqi,t}=\zero,\forall h\in[\tilde H]$.\label{line:set_elliptical_reward}
	\STATE Using \pref{alg:linear_fqi}, compute 
	\[\pi_t = \textsc{FQI-ELLIPTICAL}\rbr{ \Dcal, R^{\fqi,t}}.\] \label{line:fqi}
	\vspace{-1em}
	\STATE Estimate feature covariance matrix $\widehat{\Sigma}_{ \pi_t}$ as\label{line:est_cov_mat}
	\FOR{$i=1,\ldots, d$}
    \FOR{$j=1,\ldots, d$}
    \STATE Define reward function $R^{\fqe,ij}$ as $R^{\fqe,ij}_{\tilde H}(\cdot,\cdot)=\frac{1+\hat\phi_{\tilde H}(\cdot,\cdot)[i]\hat\phi_{\tilde H}(\cdot,\cdot)[j]}{2}$ and $R^{\fqe,ij}_{h}=\zero,\forall h\in[\tilde H]$.
    \STATE 	Estimate the $(i,j)$-th coordinate of $\widehat \Sigma_{\pi_t}$ using \fqe (\pref{alg:linear_fqe})
    \[\widehat \Sigma_{\pi_t}[i,j]:=2\sbr{\textsc{FQE}\rbr{\Dcal,R^{\fqe,ij},\pi_t}}-1.\] 
    \vspace{-1em}
	\ENDFOR
	\ENDFOR
	\STATE Update $\Gamma_t \gets \Gamma_{t-1} + \widehat{\Sigma}_{\pi_t}$.\label{line:update_cov}
	\STATE Estimate the expected return of $\pi_t$ under the elliptical reward $R^t$ as 
	\[\hat v_t^{\pi_t}:=\textsc{FQE}\rbr{\Dcal,R^{\fqi,t},\pi_t}.\]\label{line:fqe_expected_return}
	\vspace{-1em}
	\STATE If the estimated objective $\hat v_t^{\pi_t}\le\frac{3\beta}{4}$, halt and output $\rho \coloneqq \unif(\{\pi_\tau\}_{1\le \tau \le t})$. \label{line:ell_planner_output}
	\ENDFOR
\end{algorithmic}
\caption{Elliptical Planner with \fqi and \fqe}
\label{alg:FQI_ell_planner}
\end{algorithm}

On the technical side, because the number of arbitrary policies (the policy set of $\pi_t$) and the value function class are exponentially large, we need to invoke covering argument on the infinite function class. We apply the more involved concentration analysis (uniform Bernstein’s inequality) for the infinite function class to achieve sharp rates. By adapting the tools and analysis from \citet{dong2019sqrt}, we show a key concentration result in \pref{corr:uni_bern_conf}, which is then applied in the squared loss deviation result in \pref{app:dev_bound}. As discussed in \pref{sec:proof_sketch}, \fqe procedure is not the only solution for the ``elliptical planner''. Instead of running \fqe in the offline ``elliptical planner'', in each iteration we can also collect new data according to policy $\pi_t$ and use Monte Carlo evaluation to estimate the covariance matrix, which yields the online ``elliptical planner''. However, it leads to a worse rate due to additional collection and inefficient usage of prior data. Another advantage of using \fqe is that it matches the optimal $\tilde \Omega(H)$ deployment complexity as discussed in \citet{huang2021towards}.

Now we state and prove the theoretical guarantee in \pref{lem:FQI_ellip_planning}.

\begin{lemma}[Estimation and iteration guarantees for  \pref{alg:FQI_ell_planner}]
\label{lem:FQI_ellip_planning}
If \pref{alg:FQI_ell_planner} is run with a dataset of size $n\ge \tilde O(\frac{H^4d^6 \kappa K^2 \log(|\Phi|/\delta)}{\beta^2})$ for a fix $\beta > 0,\delta\in(0,1)$, then upon termination, it outputs a matrix $\Gamma_T$ and a policy $\rho$ that with probability at least $1-\delta$
\begin{align}
\label{eq:fqi_term_1}
    \forall \pi: \EE_\pi\sbr{ \hphitilH^\top\rbr{\Gamma_T}^{-1} \hphitilH} \leq O(\beta),
\end{align}
\begin{align}
\label{eq:fqi_term_2}
    \nbr{\frac{\Gamma_T}{T}  - \rbr{\EE_\rho \sbr{ \hphitilH\hphitilH^\top}+\frac{I_{d\times d}}{T}}}_{\mathrm{op}} \leq O(\nicefrac{\beta}{d}).
\end{align}
Further, the iteration complexity is also bounded $T\le \frac{8d}{\beta}\log\rbr{1+\frac{8}{\beta}}$.
\end{lemma}

\begin{proof}
As notation, we use $v_t^{\pi}$ to denote the expected return of any policy $\pi$ under the elliptical reward $R^{\fqi,t}$ ($R^{\fqi,t}$ is defined as $R_{\tilde H}^{\fqi,t}=\|\hat \phi_{\tilde{H}}\|^2_{\Gamma^{-1}_{t-1}}$ and $R_h^{\fqi,t}=\zero,\forall h\in[\tilde H]$). 

We start with showing
\begin{align}\label{eq:fqi_sample_2}
\max_{t \in [T]}\max\cbr{ d \cdot \nbr{\widehat{\Sigma}_{ \pi_t} - \Sigma_{\pi_t}}_{\mathrm{op}}, \abr{\hat{v}_t^{ \pi_t} - v_t^{\pi_t}}, \max_\pi v_t^\pi - v_t^{\pi_t}} \leq \beta/8.
\end{align}

For the second term in \Cref{eq:fqi_sample_2}, from 
\pref{corr:fqe_entire_class} we know that if $n\ge \tilde O(\frac{H^4d^5\kappa K^2\log(|\Phi|/\delta)}{\beta^2})$, then with probability at least $1-\delta/3$, we have for $\abr{\hat{v}_t^{ \pi_t} - v_t^{\pi_t}}\le\beta/8$ any $t\in[T]$. 

For the third term in \Cref{eq:fqi_sample_2}, \pref{lem:planning_sparse_reward} tells us that if $n\ge \tilde O(\frac{H^4d^3\kappa K\log(|\Phi|/\delta)}{\beta^2})$, then with probability at least $1-\delta/3$, we have $\max_\pi v_t^\pi - v_t^{\pi_t}$ for any $t\in[T]$.

Then we consider the more complicated first term in \Cref{eq:fqi_sample_2}. We will show that for any policy $\pi_t$, such quantity can be upper bound by some \fqe errors.

For any fixed policy $\pi_t$, noticing the spectral norm result from Lemma 5.4 of \citet{vershynin2010introduction}, we get 
$$\nbr{\widehat{\Sigma}_{\pi_t} - \Sigma_{\pi_t}}_{\mathrm{op}}\le(1-2\gamma')^{-1}\sup_{v\in\mathcal{N}_{\gamma'}}|\langle (\widehat \Sigma_{\pi_t} - \Sigma_{\pi_t})v,v\rangle|,$$ 
where we use $\mathcal{N}_{\gamma'}$ to denote the $\ell_2$-cover of the unit ball  $\{v\in \RR^d: \|v\|_2\le 1\}$ at scale $\gamma'$ and its size is $(\frac{2}{\gamma'})^d$.

For given $\pi_t$ and $v$, we consider the reward function $R^v$ with $R^v_{\tilde{H}}(x,a) = (v^\top \hat \phi_{\tilde{H}}(x,a))^2= v^\top \hat \phi_{\tilde{H}}(x,a)\hat \phi_{\tilde{H}}(x,a)^\top v$ $(=\sum_{i=1}^d\sum_{j=1}^d v[i]\hat\phi_{\tilde{H}}(x,a)[i]\hat \phi_{\tilde{H}}(x,a)[j]v[j])$ and $R^v_h=\zero,\forall h\in[\tilde H]$. In addition, we use $v^{\pi_t}_{R^{v}}$ to denote the expected return of policy $\pi_t$ under reward function $R^{v}$ and use $\hat v^{\pi_t}_{R^{v}}$ to denote $\textsc{FQE}\rbr{\Dcal,R^{v},\pi_t}$.

One thing to notice is that here we do not really run \fqe with reward $R^{v}$ for all $v\in\Ncal_{\gamma'}$ to get the estimate $\hat v^{\pi_t}_{R^{v}}$ in \pref{alg:FQI_ell_planner}. Instead, it only appears in the analysis. Once we calculate $\widehat \Sigma_{\pi_t}$ by estimating the covariance matrix (\pref{line:est_cov_mat}-12 in \pref{alg:FQI_ell_planner}), we can immediately get $v^\top \widehat \Sigma_{\pi_t} v$, which it is equivalent to the real output of \fqe (i.e., $\hat v^{\pi_t}_{R^{v}}$). More specifically, from the equivalence between \fqe and a model-based plug-in formulation \citep[Theorem 1]{duan2020minimax} and noticing the definition of $R^v_{\tilde{H}},R^{\fqe,ij}$, we have the following crucial equation
$$\hat v^{\pi_t}_{R^{v}}=\sum_{i=1}^d\sum_{j=1}^d v[i](2\hat v^{\pi_t}_{R^{\fqe,ij}}-1)v[j]=\sum_{i=1}^d\sum_{j=1}^d v[i]\widehat{\Sigma}_\pi[i,j]v[j]= v^\top \widehat \Sigma_{\pi_t} v.$$ 

This further implies that $v^{\pi_t}_{R^v} = v^\top \Sigma_{\pi_t} v$ and $|\langle (\widehat \Sigma_{\pi} - \Sigma_{\pi})v,v\rangle|=|\hat v^{\pi_t}_{R^v}-v^{\pi_t}_{R^v}|$ is indeed the \fqe error. Therefore, bounding the operator norm $\|\widehat{\Sigma}_{\pi_t} - \Sigma_{\pi_t}\|_{\mathrm{op}}$ can be reduced to bounding the \fqe error for a class of rewards.

Applying \pref{lem:fqe} with $\Rcal=\{R^v:v\in\Ncal_{\gamma'}\}$ and $\gamma'=\frac{\beta}{32d}$, we get that if $n\ge \tilde O(\frac{H^4d^6 \kappa K^2 \log(|\Phi|/\delta)}{\beta^2})$, then (assuming $\frac{\beta}{32d}\le \frac{1}{4}$) with probability at least $1-\delta/3$, we have for any $t\in[T]$
\begin{align*}
\nbr{\widehat{\Sigma}_{\pi_t} - \Sigma_{\pi_t}}_{\mathrm{op}}\le(1-2\gamma')^{-1}\sup_{v\in\mathcal{N}_{\gamma'}}|\langle (\widehat \Sigma_{\pi_t} - \Sigma_{\pi_t})v,v\rangle|
=(1-2\gamma')^{-1}\sup_{v\in\mathcal{N}_{\gamma'}}|\hat v^{\pi_t}_{R^v}-v^{\pi_t}_{R^v}|\le\frac{\beta}{8d}.
\end{align*}
In summary, the required number of samples for \Cref{eq:fqi_sample_2} is 
\[
n\ge \tilde O\rbr{\frac{H^4d^6 \kappa K^2 \log(|\Phi|/\delta)}{\beta^2}}.
\]
Now, if we terminate in iteration $T$, we know that $\hat v_T^{\pi_T} \leq 3\beta/4$. This implies
\begin{align}
\label{eq:fqi_sample_3}
\max_\pi v_T^{\pi} \leq v_T^{\pi_T} + \beta/8  \leq \hat{v}_T^{\pi_T} + \beta/4 \leq \beta.
\end{align}

Therefore \Cref{eq:fqi_term_1} holds. From \Cref{eq:fqi_sample_2}, it is also easy to see that \Cref{eq:fqi_term_2} holds.

Then, we turn to the iteration complexity. Similarly to \Cref{eq:fqi_sample_3}, we have
\begin{align*}
&~T\rbr{3\beta/4 - \beta/4} \leq \sum_{t=1}^T\rbr{ \hat{v}_t^{\pi_t} - \beta/4} \leq \sum_{t=1}^T \rbr{v_t^{\pi_t} - \beta/8} = \sum_{t=1}^T \rbr{\EE_{\pi_t}\sbr{\hat \phi_{\tilde H}^\top \Gamma_{t-1}^{-1}\hat\phi_{\tilde H}} - \beta/8} 
\\= &~\sum_{t=1}^T \rbr{ \tr(\Sigma_{\pi_t}\Gamma_{t-1}^{-1}) - \beta/8}
\leq \sum_{t=1}^T \tr(\widehat{\Sigma}_{\pi_t} \Gamma_{t-1}^{-1}) \leq 2d \log\rbr{1 + \frac{T}{d}}.
\end{align*}
In the last step, we apply elliptical potential lemma (e.g., Lemma 26 of \cite{agarwal2020flambe}).

Reorganizing the equation yields $T\le \frac{4d}{\beta}\log\rbr{1+\frac{T}{d}}$. Further, if $T\le \frac{8d}{\beta}\log\rbr{1+\frac{8}{\beta}}$, then we have
\begin{align*}
T &{} \le \frac{4d}{\beta}\log\rbr{1+T/d} \le \frac{4d}{\beta}\log\rbr{1+\frac{ 8\log\rbr{1+8/\beta}}{\beta}}\nonumber\\
&{} \le \frac{4d}{\beta}\log\rbr{1+\rbr{ 8/\beta}^2}\le\frac{8d}{\beta}\log\rbr{1+8/\beta}.
\end{align*}
Therefore, we obtain an upper bound on $T$ by this set and guess approach.
\end{proof}

\section{Supporting Sample Complexity Results for \pref{sec:proof_main}}
\label{app:supp_dev_bounds_main}
In this section, we state and prove the deviation bounds used for the proofs in \pref{sec:proof_main}.

\subsection{Deviation Bound for \textsc{Flo}} 
We start by stating the sample complexity result for the min-max-min oracle \textsc{flo}, defined in \cref{eq:flo_obj} in \pref{sec:flo_analysis}.
\begin{lemma}
\label{lem:dev_bound_flo}
If the feature learning objective \Cref{eq:emp_minmax_obj} is solved by the \textsc{Flo} for a sample of size $n$, then for $\Vcal \subseteq (\Xcal \rightarrow [0,L])$, $\Vcal\coloneqq \{v(x_{h+1}) = \mathrm{clip}_{[0,L]} $ $(\EE_{a_{h+1}
\sim \pi_{h+1}(x_{h+1})}$ $[R(x_{h+1},a_{h+1}) + \langle\phi_{h+1}(x_{h+1},a_{h+1}),\theta\rangle]): \phi_{h+1} \in \Phi_{h+1}, \|\theta\|_2 \le L\sqrt{d}, R \in \Rcal \}$,  where $\pi$ is the greedy policy or the uniform policy, with probability at least $1-
\delta$, we have
\begin{align*}
    \max_{v \in \Vcal} \be{\rhoxa, \hat \phi_h, v; L\sqrt{d}} \le \frac{16\cc L^2d^2 \log( 2nL\sqrt d|\Phi_h||\Phi_{h+1}||\Rcal|/\delta)}{n},
\end{align*}
where $\cc$ is the constant in \pref{lem:fqi_fastrate_unclipped}.
\end{lemma}

\begin{proof}
Firstly, note the term $\be{\rhoxa, \hat \phi_h, v; L\sqrt{d}}$ is a shorthand for
\begin{align*}
    \min_{\|w\|_2 \le L\sqrt{d}} \Lcal_{\rhoxa} (\hat \phi_h, w, v) - \Lcal_{\rhoxa}(\phi^*_h, \theta^*_v, v)
\end{align*}
where $\theta^*_v = \argmin_{\|\theta\|_2 \le L\sqrt{d}} \Lcal_{\rhoxa} (\phi^*_h, \theta, v)$.

Now, using the result in \pref{lem:fqi_fastrate_unclipped} from \pref{app:prob} and denoting $\Lcal_{\rhoxa}(\cdot)$ as $\Lcal(\cdot)$, with probability at least $1-\delta$, we have
\begin{align*}
    & \abr{\Lcal( \phi,w,v) - \Lcal(  \phi^*,\theta^*_v,v) - \rbr{ \Lcal_{\Dcal}( \phi,w,v) - \Lcal_{\Dcal}( \phi^*,\theta^*_v,v)}} \\
    \le~ &\frac{1}{2}\rbr{ \Lcal( \phi,w,v) - \Lcal( \phi^*,\theta^*_v,v)} + \frac{4\cc L^2d^2 \log( 2nL\sqrt d|\Phi_h||\Phi_{h+1}||\Rcal|/\delta)}{n}
\end{align*}
for all $\|w\|_2 \le L\sqrt{d}$ ($B = L\sqrt{d}$), $\phi \in \Phi_h$ and $v \in \Vcal$. 

By the definition of \textsc{flo}, for any $v \in \Vcal$, we know that there exists $\theta^*_v$ such that $(\phi^*, \theta^*_v)$ is the population minimizer $\argmin_{\tilde{\phi} \in \Phi_h, \|\tilde{w}\|_2 \le L\sqrt{d}} \Lcal(\tilde{\phi}, \tilde{w}, v)$. Using this and the concentration result, for all $v \in \Vcal$, with $(\tilde{\phi}_v, \tilde{w}_v)$ as the solution of the innermost min in \Cref{eq:flo_obj}, we have
\begin{align*}
    &\Lcal_{\Dcal}(\phi^*, \theta^*_v, v) -  \Lcal_{\Dcal}(\tilde{\phi}_v, \tilde{w}_v, v) \nonumber\\
    \le {} & \frac{3}{2}\rbr{\Lcal(\phi^*, \theta^*_v, v) -  \Lcal(\tilde{\phi}_v, \tilde{w}_v, v)} + \frac{4\cc L^2d^2 \log( 2nL\sqrt d|\Phi_h||\Phi_{h+1}||\Rcal|/\delta)}{n}\\
    \le {} & \frac{4\cc L^2d^2 \log( 2nL\sqrt d|\Phi_h||\Phi_{h+1}||\Rcal|/\delta)}{n}.
\end{align*}
Now, for the oracle solution $\hat \phi$, for all $v \in \Vcal$, we have
\begin{align*}
   &  {}   \Lcal_{\Dcal}(\hat \phi, \hat w_v, v) - \Lcal_{\Dcal}(\tilde{\phi}_v, \tilde{w}_v, v) \\
    =  {} &\Lcal_{\Dcal}(\hat \phi, \hat w_v, v) - \Lcal_{\Dcal}(\phi^*, \theta^*_v, v) + \Lcal_{\Dcal}(\phi^*, \theta^*_v, v) -  \Lcal_{\Dcal}(\tilde{\phi}_v, \tilde{w}_v, v) \\
  \ge   {}&   \frac{1}{2}  \rbr{\Lcal(\hat \phi, \hat w_v, v) - \Lcal(\phi^*, \theta^*_v, v)} - \frac{4\cc L^2d^2 \log( 2nL\sqrt d|\Phi_h||\Phi_{h+1}||\Rcal|/\delta)}{n}.
\end{align*}
Combining the two chains of inequalities, we get
\begin{align*}
    & \Lcal(\hat \phi, \hat w_v, v) - \Lcal(\phi^*, \theta^*_v, v)\nonumber\\ 
    \le {} & 2 \rbr{\Lcal_{\Dcal}(\hat \phi, \hat w_v, v) - \Lcal_{\Dcal}(\tilde{\phi}_v, \tilde{w}_v, v)} + \frac{8\cc L^2d^2 \log( 2nL\sqrt d|\Phi_h||\Phi_{h+1}||\Rcal|/\delta)}{n}\\
    \le {} & 2 \max_{g \in \Vcal}  \rbr{\Lcal_{\Dcal}(\hat \phi, \hat w_g, g) - \Lcal_{\Dcal}(\tilde{\phi}_g, \tilde{w}_g, g)} + \frac{8\cc L^2d^2 \log( 2nL\sqrt d|\Phi_h||\Phi_{h+1}||\Rcal|/\delta)}{n} \\
    \le {} & 2 \max_{g \in \Vcal}  \rbr{\Lcal_{\Dcal}(\phi^*, \theta^*_g, g) - \Lcal_{\Dcal}(\tilde{\phi}_g, \tilde{w}_g, g)} + \frac{8\cc L^2d^2 \log( 2nL\sqrt d|\Phi_h||\Phi_{h+1}||\Rcal|/\delta)}{n} \\
    \le {} & \frac{16\cc L^2d^2 \log( 2nL\sqrt d|\Phi_h||\Phi_{h+1}||\Rcal|/\delta)}{n}.
\end{align*}
Hence, we have proved the desired result.
\end{proof}
Here, we explicitly give a result for the discriminator function classes used by \alg. However, a similar result can be easily derived for a general discriminator class $\Vcal$ with the dependence $\log (N)$ where $N$ is either the cardinality or an appropriate complexity measure of $\Vcal$.

\subsection{Deviation Bounds for Greedy Selection} 
Below, we state the deviation bounds used in the proof of \pref{lem:dev_bound_greedy_final} in \pref{sec:greedy_analysis}.
\begin{lemma}
\label{lem:dev_bound_greedy}
Let 
\[\tilde{\veps} = \frac{2\cc d(B+L\sqrt{d})^2 \log (n(B+L\sqrt d)|\Phi||\Phi'||\Rcal|/\delta)}{n},\]
where $\cc$ is the constant in \pref{lem:fqi_fastrate_unclipped}. If \pref{alg:greedy_selection} is called with a dataset $\Dcal$ of size $n$ and termination loss cutoff $3\veps_1/2 + \tilde{\veps}$, then with probability at least $1-\delta$, for all $v \in \Vcal \subseteq (\Xcal \rightarrow [0,L])$, $\Vcal \coloneqq \{v(x_{h+1}) = \mathrm{clip}_{[0,L]} (\EE_{a_{h+1}
\sim \pi_{h+1}(x_{h+1})}[R(x_{h+1},a_{h+1}) +$ $\langle\phi_{h+1}(x_{h+1},a_{h+1}),\theta\rangle]): \phi_{h+1} \in \Phi_{h+1}, \|\theta\|_2 \le L\sqrt{d}, R \in \Rcal \}$, $\|w\|_2 \le B_t$, and $t \le T$, we have
\begin{gather*}
    \sum_{i \le t} \EE_{\rhoxa} \sbr{\rbr{ \hat \phi_{t,h}(x_h,a_h)^{\top} w_{t,i} - \phi^*_h(x_h,a_h)^{\top} \theta^*_i}^2} \le {}  t \tilde{\veps} 
    \\
    \EE_{\rhoxa} \sbr{\rbr{ \hat \phi_{t,h}(x_h,a_h)^{\top} w - \phi^*_h(x_h,a_h)^{\top} \theta^*_{t+1}}^2} \ge \veps_1
\end{gather*}
where $w_{t,i} = \argmin_{\|\tilde w\|_2 \le B_t} \Lcal_{\Dcal}(\hat \phi_{t,h}, \tilde w, v_i)$, $\theta^*_i = \argmin_{\|\tilde{w}\|_2 \le L\sqrt{d}} \Lcal_{\rhoxa}($ $\phi^*_h, \tilde{w}, v_i)$, $B\ge B_t = \frac{L\sqrt{dt}}{2}$ in \pref{alg:greedy_selection}, and policy $\pi$ is  the greedy policy or the uniform policy.

Further, at termination, the learned feature $\hat \phi_{T,h}$ satisfies
\begin{align*}
    \max_{v \in \Vcal} \be{\rhoxa, \hat \phi_{T,h}, v; B} \le {} & 3\veps_1 + 4 \tilde{\veps}.
\end{align*}
\end{lemma}
\begin{proof}
We again denote $\Lcal_{\rhoxa}(\cdot)$ as $\Lcal(\cdot)$ and set $\tilde{\veps} = \frac{2\cc d(B+L\sqrt{d})^2 \log (n(B+L\sqrt d)|\Phi||\Phi'||\Rcal|/\delta)}{n}$. Further, we remove the subscript $h$ for simplicity unless not clear by context. We begin by using the result in \pref{lem:fqi_fastrate_unclipped} such that, with probability at least $1-\delta$, for all $\|w\|_2 \le B$ ($B \ge L\sqrt{d}$), $\phi \in \Phi_h$ and $v \in \Vcal$, we have
\begin{align*}
    &\abr{\Lcal( \phi,w,v) - \Lcal(  \phi^*,\theta^*_v,v) - \rbr{ \Lcal_{\Dcal}( \phi,w,v) - \Lcal_{\Dcal}( \phi^*,\theta^*_v,v)}} \nonumber\\
    \le {} &~  \frac{1}{2} \rbr{ \Lcal( \phi,w,v) - \Lcal( \phi^*,\theta^*_v,v)} + \tilde{\veps}/2.
\end{align*}

Thus, for the feature fitting step in \pref{line:min} of \pref{alg:greedy_selection} in iteration $t$, with probability at least $1-\delta$ we have
\begin{align*}
    &~\sum_{v_i \in \Vcal^t} \EE \sbr{\rbr{ \hat \phi_{t}^\top w_{t,i} - \phi^{*\top} \theta^*_i}^2} = \sum_{v_i \in \Vcal^t}\rbr{ \Lcal(\hat \phi_t, w_{t,i}, v_i) - \Lcal(\phi^*, \theta^*_i, v_i)} 
    \\
    \le&~ \sum_{v_i \in \Vcal^t} 2  \rbr{ \Lcal_{\Dcal}( \hat \phi_{t},w_{t,i},v_i) - \Lcal_{\Dcal}( \phi^*,\theta^*_i,v_i)} + |\Vcal^t|\tilde{\veps} 
    \le~ t\tilde{\veps},
\end{align*}
which means the first inequality in the lemma statement holds.

For the adversarial test function at iteration $t$ with $B_t \le B$, let $\bar{w} \coloneqq \argmin_{\|w\|_2 \le B_t} \Lcal_{\Dcal} (\hat \phi_t,$  $w, v_{t+1})$. Using the same sample size for the adversarial test function at each non-terminal iteration with loss cutoff $c_{\textrm{cutoff}}$, for any vector $w \in \R^d$ with $\|w\|_2 \le B_t$ we get
\begin{align*}
    &\EE \sbr{\rbr{ \hat \phi_t^\top w - \phi^{*\top} \theta^*_{t+1}}^2} \nonumber= \Lcal (\hat \phi_t, w, v_{t+1}) - \Lcal (\phi^*, \theta^*_{t+1}, v_{t+1}) \\
    \ge {} & \frac{2}{3} \rbr{\Lcal_{\Dcal} (\hat \phi_t, w, v_{t+1}) - \Lcal_{\Dcal} (\phi^*, \theta^*_{t+1}, v_{t+1})} - \frac{\tilde{\veps}}{3}\\
    \ge {} & \frac{2}{3} \rbr{\Lcal_{\Dcal} (\hat \phi_t, \bar{w}, v_{t+1}) - \Lcal_{\Dcal} (\phi^*, \theta^*_{t+1}, v_{t+1})} - \frac{\tilde{\veps}}{3}\\
    \ge {} & \frac{2c_{\textrm{cutoff}}}{3} + \frac{2}{3} \rbr{\min_{\tilde{\phi} \in \Phi_h, \|\tilde{w}\|_2 \le L\sqrt{d}} \Lcal_{\Dcal} ( \tilde \phi, \tilde w, v_{t+1}) - \Lcal_{\Dcal} (\phi^*, \theta^*_{t+1}, v_{t+1})} - \frac{\tilde{\veps}}{3}\\
    \ge {} & \frac{2c_{\textrm{cutoff}}}{3} + \frac{1}{3} \rbr{\Lcal (\tilde \phi_{t+1}, \tilde w_{t+1}, v_{t+1}) - \Lcal (\phi^*, \theta^*_{t+1}, v_{t+1})} - \frac{2\tilde{\veps}}{3} 
    \ge {} \frac{2c_{\textrm{cutoff}}}{3} - \frac{2\tilde{\veps}}{3}.
\end{align*}
In the first inequality, we invoke \pref{lem:fqi_fastrate_unclipped} to move to empirical losses. In the third inequality, we add and subtract the ERM loss over $(\phi, w)$ pairs along with the fact that the termination condition is not satisfied for $v_{t+1}$. In the next step, we again use \pref{lem:fqi_fastrate_unclipped} for the ERM pair $(\tilde \phi_{t+1}, \tilde w_{t+1})$ for $v_{t+1}$.

Thus, if we set the cutoff $c_{\textrm{cutoff}}$ for test loss to $3\veps_1/2 + \tilde{\veps}$, for a non-terminal iteration $t$, for any $w \in \R^d$ with $\|w\|_2 \le B_t$, we have
\begin{align}
    \EE \sbr{\rbr{ \hat \phi_t^\top w - \phi^{*\top} \theta^*_{t+1}}^2} \ge \veps_1,
\end{align}
which implies the second inequality in the lemma statement holds.

At the same time, for the last iteration, for all $v \in \Vcal$, the feature $\hat \phi_{T}$ satisfies
\begin{align*}
    &\min_{\|w\|_2 \le B} \EE \sbr{\rbr{ \hat \phi_T^\top w - \phi^{*\top} \theta^*_v}^2} 
    \le {} 2\rbr{\Lcal_{\Dcal}(\hat \phi_T, \hat w_v, v) - \Lcal_{\Dcal}(\phi^*, \theta^*_v, v)} + \tilde{\veps}\\
    \le {} & 2 \rbr{\Lcal_{\Dcal}(\hat \phi_T, \hat w_v, v) - \Lcal_{\Dcal}(\tilde{\phi}_v, \tilde{w}_v, v) + \Lcal_{\Dcal}(\tilde{\phi}_v, \tilde{w}_v, v) - \Lcal_{\Dcal}(\phi^*, \theta^*_v, v)} + \tilde{\veps}\\
    \le {} & 2 \rbr{\Lcal_{\Dcal}(\hat \phi_T, \hat w_v, v) - \Lcal_{\Dcal}(\tilde{\phi}_v, \tilde{w}_v, v)} + \tilde{\veps} 
    \le {}  3\veps_1 + 4\tilde{\veps}.
\end{align*}
This gives us the third inequality in the lemma, thus completing the proof.
\end{proof}

\subsection{Deviation Bound for Enumerable Feature Setting}
We now show that using the ridge estimator for an enumerable feature class, discriminator class $\Fcal_{h+1}$ and an appropriately set value of $\lambda$ still allows us to establish a feature approximation result similar to \textsc{Flo} (\Cref{eq:emp_minmax_obj}) and iterative greedy feature selection (\pref{alg:greedy_selection}).

\begin{lemma}
\label{lem:ridge_error}
For the feature $\hat \phi_h$ learned via the ridge estimator \Cref{eq:ev-enum_pf}, with $B=\frac{1}{\lambda}$ and $\lambda = \tilde{\Theta}\rbr{\frac{1}{n^{1/3}}}$ for any function $f \in \Fcal_{h+1}$, with probability at least $1-\delta$, we have
\begin{align*}
   \max_{f\in\Fcal_{h+1}}\be{\rhoxa, \hat \phi_h, f; B} \le \tilde{O}\rbr{\frac{d^2 \log(2n|\Phi_h||\Phi_{h+1}|/\delta)}{n^{1/3}}}
\end{align*}
where the discriminator function class is $\Fcal_{h+1}\defeq\{f(x_{h+1}) =\EE_{\unif(\Acal)}\sbr{\inner{\phi_{h+1}(x_{h+1},a)}{\theta}}:\phi_{h+1}$ $\in\Phi_{h+1},\|\theta\|_2\le\sqrt d\}$.
\end{lemma}

\begin{proof}
Firstly, from the definition of $\Fcal_{h+1}$, we note that note the term $\be{\rhoxa, \hat \phi_h, f; 1/\lambda}$ can be written as
\begin{align*}
    \min_{\|\hat w_f\|_2 \le 1/\lambda} \Lcal_{\rhoxa} (\hat \phi_h, \hat w_f, f) - \Lcal_{\rhoxa}(\phi^*_h, \theta^*_f, f)
\end{align*}
where $\theta^*_f = \argmin_{\|\theta\|_2 \le \sqrt d} \Lcal_{\rhoxa} (\phi^*_h, \theta, f)$.

We again denote $\Lcal_{\rhoxa}(\cdot)$ as $\Lcal(\cdot)$, $\Fcal_{h+1}$ as $\Fcal$. Firstly, as the discriminator function class $\Fcal$ is defined without clipping, we now have: $\EE[f(x')|x,a] = \langle \phi^*(x,a), \theta^*_f\rangle$ with $\|\theta^*_f\|_2 \le d$. Also, the scale of the ridge estimator $w_f = \rbr{\frac{1}{n}X^\top X + \lambda I_{d \times d}}^{-1}\rbr{\frac{1}{n} X^\top f}$ now scales as $\frac{1}{\lambda}$. Now, applying \pref{lem:fqi_fastrate_unclipped}, for all $\phi \in \Phi_h$, $\|w\|_2 \le 1/\lambda$ and $f \in \Fcal$, we have
\begin{align*}
    & \abr{\Lcal( \phi,w,f) - \Lcal(  \phi^*,\theta^*_f,f) - \rbr{ \Lcal_{\Dcal}( \phi,w,f) - \Lcal_{\Dcal}( \phi^*,\theta^*_f,f)}} \\
    \le~&  \frac{1}{2} \rbr{ \Lcal( \phi,w,f) - \Lcal( \phi^*,\theta^*_f,f)} + 
     \frac{\cc d(1/\lambda+\sqrt{d})^2 \log (n(1/\lambda+\sqrt d)|\Phi_h||\Phi_{h+1}|/\delta)}{n},
\end{align*}
where $\cc$ is the constant in \pref{lem:fqi_fastrate_unclipped}. Now, let $w^*_f$ denote the population ridge regression estimator for target $f \in \Fcal$ for features $\phi^*$. Assume $\lambda \le 1/d$ Then an upper bound on the second term in the RHS is $\gamma\coloneqq\frac{4\cc d \log (2n(1/\lambda)|\Phi_h||\Phi_{h+1}|/\delta)}{\lambda^2n}$. For the selected feature $\hat \phi$, we have
\begin{align*}
    &\Lcal_{\Dcal}(\hat \phi, \hat w_f, f) - \Lcal_{\Dcal}(\tilde{\phi}_f, \tilde{w}_f, f) \nonumber\\
    = {} & \Lcal_{\Dcal}(\hat \phi, \hat w_f, f) - \Lcal_{\Dcal}(\phi^*, \theta^*_f, f) + \Lcal_{\Dcal}(\phi^*, \theta^*_f, f) -  \Lcal_{\Dcal}(\tilde{\phi}_f, \tilde{w}_f, f) \\
    \ge {} & \frac{1}{2}  \rbr{\Lcal(\hat \phi, \hat w_f, f) - \Lcal(\phi^*, \theta^*_f, f)} + \Lcal_{\Dcal}(\phi^*, \theta^*_f, f) -  \Lcal_{\Dcal}(\tilde{\phi}_f, \tilde{w}_f, f) - \gamma \\
    \ge {} & \frac{1}{2}  \rbr{\Lcal(\hat \phi, \hat w_f, f) - \Lcal(\phi^*, \theta^*_f, f)} + \Lcal_{\Dcal}(\phi^*, \theta^*_f, f) -  \Lcal_{\Dcal}(\phi^*, w^*_f, f) - \gamma.
\end{align*}

Thus, with the feature selection output, we have
\begin{align*}
    & \Lcal( \hat \phi,\hat w_f,f) - \Lcal(  \phi^*,\theta^*_f,f)\\ \le {} & 2 \rbr{\Lcal_{\Dcal}(\hat \phi, \hat w_f, f) - \Lcal_{\Dcal}(\tilde{\phi}_f, \tilde{w}_f, f)} + 2\rbr{\Lcal_{\Dcal}(\phi^*, w^*_f, f) - \Lcal_{\Dcal}(\phi^*, \theta^*_f, f)}  + 2\gamma\\
    \le {} & 2\max_{g \in \Fcal} \rbr{\Lcal_{\Dcal}(\hat \phi, \hat w_g, g) - \Lcal_{\Dcal}(\tilde{\phi}_g, \tilde{w}_g, g)}  + 3\rbr{\Lcal(\phi^*, w^*_f, f) - \Lcal(\phi^*, \theta^*_f, f)}  + 4\gamma\\
    \le {} & 2\max_{g \in \Fcal} \rbr{\Lcal_{\Dcal}(\phi^*, w^*_g, g) - \Lcal_{\Dcal}(\tilde{\phi}_g, \tilde{w}_g, g)} +  3\rbr{\Lcal(\phi^*, w^*_f, f) - \Lcal(\phi^*, \theta^*_f, f)}  + 4\gamma\\
    \le {} & 2\max_{g \in \Fcal} \rbr{\Lcal_{\Dcal}(\phi^*, w^*_g, g) - \Lcal_{\Dcal}(\phi^*, \theta^*_g, g)}+ 2\max_{g \in \Fcal} \rbr{\Lcal_{\Dcal}(\phi^*, \theta^*_g, g) - \Lcal_{\Dcal}(\tilde{\phi}_g, \tilde{w}_g, g)} \\
    & {} +  3\rbr{\Lcal(\phi^*, w^*_f, f) - \Lcal(\phi^*, \theta^*_f, f)}  + 4\gamma\\
    \le {} & 2 \max_{g \in \Fcal} \rbr{\Lcal_{\Dcal}(\phi^*, w^*_g, g) - \Lcal_{\Dcal}(\phi^*, \theta^*_g, g)}  + 3\rbr{\Lcal(\phi^*, w^*_f, f) - \Lcal(\phi^*, \theta^*_f, f)} + 6\gamma\\
    \le {} & 6\max_{g \in \Fcal} \rbr{\Lcal(\phi^*, w^*_g, g) - \Lcal(\phi^*, \theta^*_g, g)} + 8\gamma.
\end{align*}
The third inequality uses the fact that $\hat \phi$ is the solution of the ridge regression based feature selection objective. Further, in all steps, we repeatedly apply the deviation bound from \pref{lem:fqi_fastrate_unclipped} to move from $\Lcal_{\Dcal}(\cdot)$ to $\Lcal(\cdot)$.

Now, for ridge regression estimate $w^*_g$, we can bound the bias term on the RHS as follows:
\begin{align*}
    {} &\Lcal(\phi^*, w^*_g, g) - \Lcal(\phi^*, \theta^*_g, g) = {}  \EE\sbr{\rbr{\inner{\phi^*}{w^*_g} - \inner{\phi^*}{\theta^*_g}}^2} 
    \\
    = {} & \|w^*_g - \theta^*_g\|_{\Sigma^*}^2 = \sum_{i=1}^d \lambda_i \inner{v_i}{w^*_g - \theta^*_g}^2 
    = {}  \sum_{i=1}^d \lambda_i \rbr{\frac{\lambda_i}{\lambda + \lambda_i}\inner{v_i}{\theta^*_g} - \inner{v_i}{\theta^*_g}}^2 \\
    = {} & \sum_{i=1}^d \frac{\lambda_i \lambda^2 \inner{v_i}{\theta^*_g}^2}{(\lambda_i + \lambda)^2} \le \frac{\lambda}{4}\|\theta^*_g\|_2^2 \le \frac{\lambda d^2}{4},
\end{align*}
where $(\lambda_i, v_i)$ denote the $i$-th eigenvalue-eigenvector pair of the population covariance matrix $\Sigma^*$ for feature $\phi^*$. In the derivation above, we use the fact that $w^*_g = (\Sigma + \lambda I)^{-1}\EE[\phi^* g] = \frac{\lambda_i}{\lambda + \lambda_i}\inner{v_i}{\theta^*_g}$.

Therefore, the final deviation bound for $\hat \phi$ is
\begin{align*}
    \Lcal( \hat \phi,\hat w_f,f) - \Lcal(  \phi^*,\theta^*_f,f) \le {} & \frac{3\lambda d^2}{2} + \frac{32\cc d \log (2n(1/\lambda)|\Phi_h||\Phi_{h+1}|/\delta)}{\lambda^2n}.
\end{align*}

Setting $\lambda = \tilde{O}\rbr{\frac{1}{n^{1/3}}}$ gives us the final result.
\end{proof}

\section{FQI Planning Results}
\label{app:fqi}
In this section, we provide various FQI (Fitted Q-iteration) planning results. In \pref{app:fqi_framework}, we provide the general framework of FQI algorithms. In \pref{app:fqi_full_class}, we show the sample complexity of \textsc{fqi-full-class} that handles the offline planning for a class of rewards. In \pref{app:fqi_repr}, we discuss the sample complexity result of planning with the learned feature $\bar\phi$. In \pref{app:fqi_ellip}, we provide the sample complexity guarantee for planning for the elliptical reward class. We want to mention that we abuse some notations in this section. For example, $\Fcal_h,\Gcal_h,\Vcal_h$ may have different meanings from the main text. However, they should be clear within the context. 

For simplicity, we use the horizon $H$ in \pref{alg:linear_fqi} and all statements. When \textsc{FQI} is called with a smaller horizon $\tilde H\le H$, we can just replace all $H$ by $\tilde H$. The analyses still go through and the sample complexity will not exceed the one instantiated with $H$.

\subsection{FQI Planning Algorithm}\label{app:fqi_framework}
In this part, we present the general framework of FQI planner in \pref{alg:linear_fqi}, which subsumes three different algorithms: \textsc{fqi-full-class}, \textsc{fqi-representation}, and \textsc{fqi-elliptical}. \textsc{fqi-full-class} and \textsc{fqi-representation} will be used to plan for a finite deterministic reward class, while \textsc{fqi-elliptical} is specialized in planning for the elliptical reward class. 

\begin{algorithm}[tbh]
\begin{algorithmic}[1]
\STATE \textbf{input:} (1) exploratory dataset $\{\Dcal\}_{0:H-1}$ sampled from $\rho_{h-3}^{+3}$ with size $n$ at each level $h\in[H]$, (2) reward function $R=R_{0:H-1}$ with $R_h:\Xcal\times\Acal \to [0,1],\forall h\in[H]$, (3) function class: (i) for \textsc{FQI-full-class}, $\Fcal_h(R_h)\coloneqq \{R_h+\langle \phi_h,w_h\rangle: \|w_h\|_2 \le H\sqrt{d},\phi_h \in \Phi_h\}, h\in[H]$; (ii) for \textsc{FQI-representation}, $\Fcal_h(R_h)\coloneqq \{R_h+\mathrm{clip}_{[0,H]}\rbr{\langle \bar \phi_h,w_h\rangle}: \|w_h\|_2 \le B\}, h\in[H]$; (iii) for \textsc{FQI-elliptical}, $\Fcal_h(R_h)\coloneqq \{R_h+\langle  \phi_h,w_h\rangle:\|w_h\|_2 \le \sqrt{d},\phi_h\in\Phi_h\}, h\in[H]$.
\STATE Set $\hat{V}_H(x) = 0$.
\FOR{$h= H-1,\ldots,0$}
\STATE Pick $n$ samples $\cbr{\rbr{x^{(i)}_h,a^{(i)}_h,x^{(i)}_{h+1}}}_{i=1}^n$ from the exploratory dataset $\Dcal_h$.
\STATE Solve least squares problem:
\begin{align}\label{eq:ls_fqi}
\hat{f}_{h,R_h} \gets \argmin_{f_h\in \Fcal_h(R_h)} \Lcal_{\Dcal_h,R_h}(f_h,\hat V_{h+1}),
\end{align}
where
$\Lcal_{\Dcal_h,R_h}(f_h,\hat V_{h+1}):=\sum_{i=1}^n\rbr{ f_h\rbr{x^{(i)}_h,a^{(i)}_h} - R_h\rbr{x_h^{(i)},a_h^{(i)}} - \hat V_{h+1}\rbr{x_{h+1}^{(i)}}}^2$.
\IF {\textsc{FQI-full-class} or \textsc{FQI-representation}}
\STATE Define $\hat{\pi}_h(x) = \argmax_{a} \hat{f}_{h,R_h}(x,a)$ and $\hat{V}_{h}(x) = \mathrm{clip}_{[0,H]}\rbr{\max_a \hat{f}_{h,R_h}(x,a)}$.
\ELSIF {\textsc{FQI-elliptical}}
\STATE Define $\hat{\pi}_h(x) = \argmax_{a} \hat{f}_{h,R_h}(x,a)$ and $\hat{V}_h(x) = \mathrm{clip}_{[0,1]}\rbr{\max_a \hat{f}_{h,R_h}(x,a)}$. 
\ENDIF
\ENDFOR
\STATE \textbf{return} $\hat{\pi} = (\hat{\pi}_0,\ldots,\hat{\pi}_{H-1})$.
\end{algorithmic}
\caption{\textsc{FQI}: Fitted Q-Iteration}
\label{alg:linear_fqi}
\end{algorithm}

This leads to the different bounds of parameters in the Q-value function classes and different clipping thresholds of the state-value functions. In addition, for \textsc{fqi-full-class} and \textsc{fqi-elliptical}, we use all features in $\Phi$ to construct the Q-value function classes, while in \textsc{fqi-representation} we only use the the learned representation $\bar \phi$. The details can be found below. When calling \pref{alg:linear_fqi} and there is no confusion, we sometimes drop the input (3) function class for simplicity.

Unlike the regression problem in the main text, the objective here includes an additional reward function component. Therefore, we define a new loss function $\Lcal_{D_h,R_h}$, and will use $\Lcal_{\rho_{h-3}^{+3},R_h}$ to denote its population version. Notice that the function class $\Fcal_h(R_h)$ in \pref{alg:linear_fqi} also depends on the reward function $R$. If we pull out the reward term from $\hat f_{h,R_h}$, we can obtain an equivalent solution of the least squares problem \Cref{eq:ls_fqi} as below 
\begin{align*}
\hat{f}_{h,R_h} =R_h + \argmin_{f_h\in \Fcal_h(\mathbf{0})} \Lcal_{\Dcal_h}(f_h,\hat V_{h+1}),
\end{align*}
where
\begin{align*}
\Lcal_{\Dcal_h}(f_h,\hat V_{h+1}):=\sum_{i=1}^n\left( f_h\rbr{x^{(i)}_h,a^{(i)}_h}- \hat V_{h+1}\rbr{x_{h+1}^{(i)}}\right)^2.
\end{align*}
Intuitively, the reward function $R_h$ only makes the current least squares solution offset the original (reward-independent) least squares solution by $R_h$.

\subsection{Planning for a Reward Class with the Full Representation Class}\label{app:fqi_full_class}
In this part, we first establish the sample complexity of planning for a prespecified deterministic reward function $R$ in \pref{lem:real_world_optimization}. We will choose \textsc{fqi-full-class} as the planner, where the Q-value function class consists of linear function of all features in the feature class with reward appended. Specifically, we have $\Fcal_h(R_h)\coloneqq \{R_h+\langle \phi_h,w_h\rangle: \|w_h\|_2 \le H\sqrt{d},\phi_h \in \Phi_h\}, h\in[H]$. Equipped with this lemma, we also provide the sample complexity of planning for a finite deterministic reward class $\Rcal$ in \pref{corr:fqi_full_class}. The analysis is similar to that of \citet{chen2019information,agarwal2020flambe}.

\begin{lemma}[Planning for a prespecified reward with the entire representation class]
\label{lem:real_world_optimization}
Assume that we have the exploratory dataset $\{\Dcal\}_{0:H-1}$, which is collected from $\rho_{h-3}^{+3}$ and satisfies~\Cref{eq:dist_shift} for all $h \in [H]$. For a prespecified deterministic reward function $R=R_{0:H-1}, R_h:
\Xcal \times \Acal \to [0,1],\forall h\in[H]$ and $\delta \in (0,1)$, if we set
\begin{align*}
n\geq \frac{32\cc H^6 d^2\kappa K }{\beta^2} \log \rbr{\frac{16\cc H^6 d^2\kappa K }{\beta^2}} + \frac{32\cc H^6 d^2\kappa K }{\beta^2} \log\rbr{ \tfrac{4|\Phi|H^3d}{\delta}},
\end{align*}
where $\cc$ is the constant in \pref{lem:fqi_fastrate_unclipped}, then with probability at least $1-\delta$, the policy $\hat{\pi}$ returned
by \textsc{FQI-full-class} satisfies
\begin{align*}
\EE_{\hat{\pi}}\sbr{\sum_{h=0}^{H-1} R_h(x_h,a_h)} \geq \max_{\pi} \EE_\pi\sbr{\sum_{h=0}^{H-1}R_h(x_h,a_h)} - \beta.
\end{align*}
\end{lemma}

\begin{proof}
We first bound the difference in cumulative rewards between $\hat{\pi}
\defeq \hat{\pi}_{0:H-1}$ and the optimal policy $\pi^*$ for the given
reward function. Recall that $\hat{\pi}_0$ is greedy w.r.t. $\hat f_{0,R_0}$, which
implies that $\hat{f}_{0,R_0}(x_0,\hat{\pi}_0(x_0)) \geq
\hat{f}_0(x_0,\pi^*(x_0))$ for all $x_0$. Hence, we have
\begin{align*}
&~v^*_R - v^{\hat{\pi}}_R \\
=&~ \EE_{\pi^*}\sbr{R_0(x_0,a_0) + V^{*}_1(x_1)} - \EE_{\hat{\pi}}\sbr{R_0(x_0,a_0) + V^{\hat{\pi}}_1(x_1)}\\
\leq&~ \EE_{\pi^*}\sbr{R_0(x_0,a_0) + V^*_1(x_1) - \hat{f}_0(x_0,a_0)} - \EE_{\hat{\pi}}\sbr{R_0(x_0,a_0) + V^{\hat{\pi}}_1(x_1) - \hat{f}_{0,R_0}(x_0,a_0)}\\
=&~ \EE_{\pi^*}\sbr{R_0(x_0,a_0)+V^*_1(x_1) - \hat{f}_{0,R_0}(x_0,a_0)} - \EE_{\hat\pi}\sbr{R_0(x_0,a_0)+V^{*}_1(x_1) - \hat{f}_{0,R_0}(x_0,a_0) } \\
& ~~~~ + \EE_{\hat{\pi}}\sbr{V^*_1(x_1) - V^{\hat{\pi}}_1(x_1) }\\
=&~  \EE_{\pi^*}\sbr{Q_0^*(x_0,a_0) - \hat{f}_{0,R_0}(x_0,a_0)} - \EE_{\hat\pi}\sbr{Q_0^*(x_0,a_0) - \hat{f}_{0,R_0}(x_0,a_0)  } \\
& ~~~~+ \EE_{\hat{\pi}}\sbr{V^*_1(x_1) - V^{\hat{\pi}}_1(x_1)}.
\end{align*}

Continuing unrolling to $h=H - 1$, we get
\begin{align*}
&~v^*_R - v^{\hat{\pi}}_R \\
\le &~ \sum_{h=0}^{H-1}\EE_{\hat{\pi}_{0:h-1}\circ\pi^*}\sbr{Q_h^*(x_h,a_h) - \hat f_{h,R_h}(x_h,a_h)} - \sum_{h=0}^{H-1} \EE_{\hat{\pi}_{0:h}}\sbr{Q_h^*(x_h,a_h) - \hat f_{h,R_h}(x_h,a_h) }.
\end{align*}
Now we bound each of these terms. The two terms only differ in the policies that generate the data and can be handled similarly. Therefore, in the following, we focus on just one of them. For any function $V_{h+1}:\Xcal \rightarrow \RR$, we introduce Bellman backup operator $(\Tcal_h V_{h+1})(x_h,a_h)\defeq R_h(x_h,a_h)+\EE\sbr{V_{h+1}(x_{h+1})\mid x_h,a_h}$.
Let's call the roll-in policy $\pi$ and drop the dependence on $h$. This gives us 
\begin{align*}
&\abr{\EE_\pi\sbr{Q^*(x,a)- \hat f_R(x,a) }}=\abr{\EE_\pi\sbr{R(x,a) + \EE\sbr{V^*(x')\mid x,a} - \hat f_R(x,a)}}\\
\le{}&\EE_\pi\sbr{\abr{R(x,a) + \EE\sbr{V^*(x')\mid x,a} - \hat f_R(x,a)}} \\
\leq{} &\EE_\pi\sbr{\abr{\EE\sbr{V^*(x')\mid x,a} - \EE\sbr{\hat{V}(x') \mid x,a}} + \abr{R(x,a) + \EE\sbr{\hat{V}(x') \mid x,a} - \hat f_R(x,a)}}\\
\leq{}&\EE_\pi\sbr{\abr{V^*(x') - \hat{V}(x')}+ \abr{(\Tcal \hat V)(x,a) - \hat f_R(x,a)}},
\end{align*}
where the last inequality is due to Jensen's inequality.

From the definition of $\hat V(x')$, we have
\begin{align*}
\EE_\pi\sbr{\abr{V^*(x') - \hat{V}(x')}} 
&\leq  \EE_\pi\sbr{\abr{\max_{a} Q^*(x',a) - \max_{a'}\hat f_R(x',a')}}\\
&\leq \EE_{\pi \circ \tilde{\pi}} \sbr{ \abr{Q^*(x',a') - \hat f_R(x',a')}}.
\end{align*}
In the last inequality, we define $\tilde{\pi}$ to be the greedy one between two actions, that is we set $\tilde{\pi}(x') =
\argmax_{a'} \max\{Q^*(x',a'),\hat f (x',a')\}$. This expression has
the same form as the initial one, while at the next timestep. Keep unrolling yields
\begin{align*}
&\abr{\EE_\pi \sbr{Q^*_h(x_h, a_h)- \hat f_{h,R_h}(x_h,a_h)}}\\
\leq&~   \sum_{\tau=h}^{H-1} \max_{\pi_\tau} \EE_{\pi_\tau}\sbr{\abr{(\Tcal_\tau \hat V_{\tau+1})(x_\tau,a_\tau) - \hat f_{\tau,R_\tau}(x_\tau,a_\tau)}}\notag\\
\leq&~\sum_{\tau=h}^{H-1} \max_{\pi_\tau}\sqrt{ \EE_{\pi_\tau}\sbr{\sbr{(\Tcal_\tau \hat V_{\tau+1})(x_\tau,a_\tau) - \hat f_{\tau,R_\tau}(x_\tau,a_\tau)}^2}} \notag\\
\leq &~ \sum_{\tau=h}^{H-1} \sqrt{\kappa K \EE_{\rho_{\tau-3}^{
+3}}\sbr{\sbr{(\Tcal_\tau \hat V_{\tau+1})(x_\tau,a_\tau) - \hat f_{\tau,R_\tau}(x_\tau,a_\tau)}^2}}, \label{eq:dist_fqi}
\end{align*}
where the last inequality is due to condition~\Cref{eq:dist_shift}.

Further, we have that with probability at least $1-\delta$,
\begin{align*}
&\EE_{\rho_{\tau-3}^{+3}}\sbr{\rbr{(\Tcal_\tau \hat V_{\tau+1})(x_\tau,a_\tau) - \hat f_{\tau,R_\tau}(x_\tau,a_\tau)}^2} \nonumber\\
= {} & \EE_{\rho_{\tau-3}^{+3}}\Big[\big(R_\tau(x_\tau,a_\tau)+\hat{V}_{\tau+1}(x_{\tau +1}) - \hat f_{\tau,R_\tau}(x_\tau,a_\tau)\big)^2\\
&\quad -\big( R_\tau(x_\tau,a_\tau) + \hat V_{\tau+1}(x_{\tau+1})- (\Tcal_\tau \hat V_{\tau+1})(x_\tau,a_\tau)\big)^2\Big]\\
={}&\EE\sbr{\Lcal_{\Dcal_\tau,R_\tau}(\hat f_{\tau,R_\tau},\hat V_{\tau+1})} - \EE\sbr{\Lcal_{\Dcal_\tau,R_\tau}( \Tcal_\tau \hat V_{\tau+1},\hat V_{\tau+1})}\\
\le{}&\frac{16\cc H^2d^2 \log (4nH^3d|\Phi|/\delta)}{n}. \tag{Step (*), \pref{lem:fqi_arbitrary}}
\end{align*}

Plugging this back into the overall value performance difference, the bound is
\begin{align*}
v^*_R - v^{\hat{\pi}}_R \leq H^2\sqrt{\kappa K}\sqrt{\frac{16\cc H^2d^2 \log (4nH^3d|\Phi|/\delta)}{n}}.
\end{align*}
Setting RHS to be less than $\beta$ and reorganize, we get  $$n\geq \frac{16\cc H^6 d^2\kappa K \log (4nH^3d|\Phi|/\delta)}{\beta^2}.$$

A sufficient condition for the inequality above is 
$$n\geq \frac{32\cc H^6 d^2\kappa K }{\beta^2} \log \rbr{\frac{16\cc H^6 d^2\kappa K }{\beta^2}} + \frac{32\cc H^6 d^2\kappa K }{\beta^2} \log\rbr{ \tfrac{4|\Phi|H^3d}{\delta}},$$
which completes the proof.
\end{proof}

\begin{corollary}[Planning for a reward class with a full representation class] \label{corr:fqi_full_class}\ \\
Assume that we have the exploratory dataset $\{\Dcal\}_{0:H-1}$, which is collected from $\rho_{h-3}^{+3}$ and satisfies~\Cref{eq:dist_shift} for all $h \in [H]$, and we are given a finite deterministic reward class $\Rcal =\Rcal_0\times\ldots\times\Rcal_{H-1},\Rcal_h\subseteq(
\Xcal \times \Acal \to [0,1]),\forall h\in[H]$. For $\delta \in (0,1)$ and any reward function $R
\in \Rcal$, if we set
\begin{align*}
n\geq \frac{32\cc H^6 d^2\kappa K }{\beta^2} \log \rbr{\frac{16\cc H^6 d^2\kappa K }{\beta^2}} + \frac{32\cc H^6 d^2\kappa K }{\beta^2} \log\rbr{ \tfrac{4|\Phi|\Rcal||H^3d}{\delta}},
\end{align*}
where $\cc$ is the constant in \pref{lem:fqi_fastrate_unclipped}, then with probability at least $1-\delta$, the policy $\hat{\pi}$ returned
by \textsc{FQI-full-class} satisfies
\begin{align*}
\EE_{\hat{\pi}}\sbr{\sum_{h=0}^{H-1} R_h(x_h,a_h)} \geq \max_{\pi} \EE_\pi\sbr{\sum_{h=0}^{H-1}R_h(x_h,a_h)}  - \beta.
\end{align*}
\end{corollary}
\begin{proof}
For any fixed reward $R\in\Rcal$, applying \pref{lem:real_world_optimization} yields that with probability $1-\delta'$,  
\begin{align*}
\EE_{\hat{\pi}}\sbr{\sum_{h=0}^{H-1} R_h(x_h,a_h)} \geq \max_{\pi} \EE_\pi\sbr{\sum_{h=0}^{H-1}R_h(x_h,a_h)}  - \beta,
\end{align*}
if we set $n$ according to the lemma statement.

Then union bounding over $R\in\Rcal$ and setting $\delta=\delta'/|\Rcal|$ gives us the desired result.
\end{proof}

\begin{lemma}[Deviation bound for \pref{lem:real_world_optimization}]
\label{lem:fqi_arbitrary} Given a deterministic reward function $R=R_{0:H-1}, R_h:\Xcal\times \Acal \rightarrow [0,1],\forall h\in[H]$ and a dataset $\{\Dcal\}_{0:H-1}$ collected from $\rho_{h-3}^{+3}$, where $\Dcal_{h}$ is  $\cbr{\rbr{x_h^{(i)},a_h^{(i)},x_{h+1}^{(i)}}}_{i=1}^n$. With probability at least $1-\delta$,  $\forall\,h\in[H],V_{h+1} \in \Vcal_{h+1}(R_{h+1})$, we have
\begin{align*}
\abr{\EE\sbr{\Lcal_{\Dcal_h,R_h}(\hat f_{h,R_h}, V_{h+1})} - \EE\sbr{\Lcal_{\Dcal_h,R_h}( \Tcal_h V_{h+1}, V_{h+1})}} \le \frac{16\cc H^2d^2 \log (4nH^3d|\Phi|/\delta)}{n},
\end{align*}
Here $\cc$ is the constant in \pref{lem:fqi_fastrate_unclipped}, $\Vcal_{h+1}(R_{h+1}):=\{\mathrm{clip}_{[0,H]}\rbr{\max_a f_{h+1,R_{h+1}}(x_{h+1},a)}:f_{h+1,R_{h+1}}\in \Fcal_{h+1}(R_{h+1})\}$ for $h\in[H-1]$ and $\Vcal_{H}=\{\mathbf{0}\}$ is the  state-value function class, and $\Fcal_h(R_h)\coloneqq \{R_h+\langle \phi_h,w_h\rangle: \|w_h\|_2 \le H\sqrt{d},\phi_h \in \Phi_h\}$ for $h\in[H]$ is the reward dependent Q-value function class.
\end{lemma}

\begin{proof}
Recall the definition, we have $\hat f_{h,R_h}=R_h+\hat f_h$, $\hat{f}_h = \argmin_{f_h\in \Fcal_h(\mathbf{0})} \Lcal_{\Dcal_h}(f_h, V_{h+1})$, $\Fcal_h(\mathbf{0})\coloneqq\{\langle \phi_h,w_h\rangle: \|w_h\|_2 \le H\sqrt{d},\phi_h \in \Phi_h\}$, and $\Lcal_{\Dcal_h}(f_h, V_{h+1}):=\sum_{i=1}^n\left( f_h\rbr{x^{(i)}_h,a^{(i)}_h}-\right.$ $\left. V_{h+1}\rbr{x_{h+1}^{(i)}}\right)^2$. Therefore we have $\Lcal_{\Dcal_h,R_h}(\hat f_{h,R_h}, V_{h+1}) =\Lcal_{\Dcal_h}(\hat f_h, V_{h+1})$ and  
$\Lcal_{\Dcal_h,R_h}( \Tcal_h V_{h+1},$ $ V_{h+1})=\Lcal_{\Dcal_h}( \Tcal_h V_{h+1}-R_h, V_{h+1})=\Lcal_{\Dcal_h}(\langle\phi^*_h,\theta_{V_{h+1}}^*\rangle, V_{h+1}).$
Then it suffices to show the bound between $\Lcal_{\Dcal_h}(\hat f_h, V_{h+1})$ and $\Lcal_{\Dcal_h}(\langle\phi^*_h,\theta_{V_{h+1}}^*\rangle, V_{h+1})$.

Firstly, we fix $h\in[H]$. Noticing the structure of $\Fcal_h(\mathbf{0})$, for any $f_h\in\Fcal_h(\mathbf{0})$, we can associate it with some $\phi_h\in\Phi_h$ and $w_h$ that satisfies $\|w_h\|_2 \le H\sqrt d$. Therefore, we can equivalently write $\Lcal_{D_h}(f_h,V_{h+1})$ as $\Lcal_{D_h}(\phi_h,w_h,V_{h+1})$. Also noticing the structure of $\Vcal_{h+1}(R_{h+1})$, we can directly apply \pref{lem:fqi_fastrate_unclipped} with $\rho_{h-3}^{+3}$, $\Phi_h$, $\Phi_{h+1}$, $B=H\sqrt d$, $L=H$, $\Rcal=\{R\}$, and $\pi'$ be the greedy policy. 

This implies that for all $\|w_h\|_2\le H\sqrt d$, $\phi_h \in \Phi_h$, and $V_{h+1} \in \Vcal_{h+1}(R_{h+1})$, with probability at least $1-\delta/H$, we have
\begin{align}
&\Lcal_{\rho_{h-3}^{+3}}( \phi_h,w_h,V_{h+1}) - \Lcal_{\rho_{h-3}^{+3}}(  \phi^*_h,\theta^*_{V_{h+1}},V_{h+1})\nonumber 
\\
\le{}&
2\rbr{  \Lcal_{\Dcal_h}( \phi_h,w_h,V_{h+1}) - \Lcal_{\Dcal_h}( \phi^*_h,\theta^*_{V_{h+1}},V_{h+1})}+\frac{16\cc H^2d^2 \log (4nH^2d|\Phi|H/\delta)}{n}.
\label{eq:used_in_fqe}
\end{align}
Notice that here we use the property that $\phi_h^*\in\Phi_h$ and $\|\theta^*_{V_{h+1}}\|_2\le H\sqrt d$ from  \pref{lem:linMDP_expn}.

From the definition, we have $\hat f_h=\argmin_{f_h\in\Fcal_h(\mathbf{0})} \Lcal_{\Dcal_h}(f_h,V_{h+1})$. Noticing the structure of $\Fcal_h(\mathbf{0})$, we can write $\hat f_h=\langle \hat \phi_h, \hat w_h\rangle$, where $\hat \phi_h,\hat w_h=\argmin_{\phi_h\in\Phi_h,\|w_h\|_2\le H\sqrt d} \Lcal_{\Dcal_h}(\phi_h,w_h,V_{h+1})$ (here we abuse the notation of $\hat \phi_h$, which is reserved for the learned feature).

This implies $\Lcal_{\Dcal_h}(\hat \phi_h,\hat w_h,V_{h+1}) - \Lcal_{\Dcal_h}( \phi^*_h,\theta^*_{V_{h+1}},V_{h+1})\le 0$. Therefore, with probability at least $1-\delta/H$
\begin{align*}
\abr{ \Lcal_{\rho_{h-3}^{+3}}( \hat \phi_h,\hat w_h,V_{h+1}) - \Lcal_{\rho_{h-3}^{+3}}(  \phi^*_h,\theta^*_{V_{h+1}},V_{h+1})}\le \frac{16\cc H^2d^2 \log (4nH^3d|\Phi|/\delta)}{n}. 
\end{align*}


Finally, by definition $\EE\sbr{\Lcal_{\Dcal_h}(\hat f_h, V_{h+1})}$ $=\Lcal_{\rho_{h-3}^{+3}}( \hat \phi_h,\hat w_h,V_{h+1})$ and $\EE[\Lcal_{\Dcal_h,R_h}(\Tcal_h V_{h+1},$ $ V_{h+1})]=\Lcal_{\rho_{h-3}^{+3}}(  \phi^*_h,\theta^*_{V_{h+1}},V_{h+1})$, union bounding over $h\in[H]$, we complete the proof.
\end{proof}

\subsection{Planning for a Reward Class with the Learned Representation Function}\label{app:fqi_repr}
In this part, we will show that the learned feature $\bar \phi$ enables the downstream policy optimization for a finite deterministic reward class $\Rcal$. The sample complexity is shown in \pref{lem:fqi_reward_class}. We will choose \textsc{fqi-representation} as the planner, where the Q-value function class only consists of linear function of learned feature $\bar \phi$ with reward appended. Specifically, we have $\Fcal_h(R_h)\coloneqq \{R_h+\mathrm{clip}_{[0,H]}\rbr{\langle \bar \phi_h,w_h\rangle}: \|w_h\|_2 \le B\}, h\in[H]$. In addition to constructing the function class with learned feature itself, we also perform clipping in $\Fcal_h(R_h)$. This clipping variant helps us avoid the poly($B$) dependence in the sample complexity. Notice that clipped Q-value function classes also work for \textsc{fqi-full-class} and \textsc{fqi-elliptical}, and would save $d$ factor. We only introduce this variant here because $B$ is much larger than $H\sqrt d$ or $d$. 

\begin{lemma}[Planning for a reward class with a learned representation] \label{lem:fqi_reward_class}
Assume that we have the exploratory dataset $\{\Dcal\}_{0:H-1}$ (collected from $\rho_{h-3}^{+3}$ and satisfies~\Cref{eq:dist_shift} for all $h \in [H]$), a learned feature $\bar \phi_h$ that satisfies the condition in \Cref{eq:fqi_rep_error}, and a finite deterministic reward class $\Rcal =\Rcal_0\times\ldots\times\Rcal_{H-1},\Rcal_h\subseteq(
\Xcal \times \Acal \to [0,1]),\forall h\in[H]$. For $\delta\in(0,1)$ and any reward function $R\in \Rcal$, if we set
\begin{align*}
n\geq \frac{2\cd H^6 d\kappa K} {\beta^2}\log\rbr{\frac{\cd H^6 d\kappa K} {\beta^2}}+\frac{2\cd H^6 d\kappa K} {\beta^2}\log \rbr{\tfrac{|\Rcal|BH}{\delta}},
\end{align*}
where $\cd$ is the constant in \pref{lem:dev_fqi_reward_class}, then with probability at least $1-\delta$, the policy $\hat{\pi}$ returned
by \textsc{FQI-Representation} satisfies
\begin{align*}
\EE_{\hat{\pi}}\sbr{\sum_{h=0}^{H-1} R_h(x_h,a_h)} \geq \max_{\pi} \EE_\pi\sbr{\sum_{h=0}^{H-1}R_h(x_h,a_h)}  - \beta - H^2\sqrt{\kappa K\veps_{\mathrm{apx}}}.
\end{align*}
\end{lemma}

\begin{proof}
Following similar steps in the proof of \pref{lem:real_world_optimization} and replacing \pref{lem:fqi_arbitrary} with \pref{lem:dev_fqi_reward_class} in Step(*), with probability at least $1-\delta$, we have
\begin{align*}
v^*_R - v^{\hat{\pi}}_R \leq & H^2\sqrt{\kappa K}\sqrt{\frac{\cd dH^2 \log (nB|\Rcal|H/\delta)}{n}  +\veps_{\mathrm{apx}}}\\
\le& H^2\sqrt{\kappa K}\sqrt{\frac{\cd dH^2 \log (nB|\Rcal|H/\delta)}{n}  }+H^2\sqrt{\kappa K\veps_{\mathrm{apx}}}.
\end{align*}

Setting RHS to be less than $\beta+H^2\sqrt{\kappa K\veps_{\mathrm{apx}}}$ and reorganize, we get the condition $$n\geq \frac{\cd H^6 d\kappa K \log(n|\Rcal|BH/\delta)}{\beta^2}.$$

A sufficient condition for the inequality above is 
$$n\geq \frac{2\cd H^6 d\kappa K} {\beta^2}\log\rbr{\frac{\cd H^6 d\kappa K} {\beta^2}}+\frac{2\cd H^6 d\kappa K} {\beta^2}\log \rbr{\tfrac{|\Rcal|BH}{\delta}},$$
which completes the proof.
\end{proof}

\begin{lemma}[Deviation bound for \pref{lem:fqi_reward_class}]\label{lem:dev_fqi_reward_class}
Assume that we have an exploratory dataset  $\Dcal_h\coloneqq \cbr{\rbr{x_h^{(i)},a_h^{(i)},x_{h+1}^{(i)}}}_{i=1}^n$ collected from $\rho_{h-3}^{+3}$, $h\in[H]$, a learned feature $\bar \phi_h$ that satisfies the condition in \Cref{eq:fqi_rep_error}, and a finite deterministic reward class $\Rcal =\Rcal_0\times\ldots\times\Rcal_{H-1},\Rcal_h\subseteq(
\Xcal \times \Acal \to [0,1]),\forall h\in[H]$. Then, with probability at least $1-\delta$, $\forall R\in \Rcal,h\in[H],V_{h+1} \in\Vcal_{h+1}(R_{h+1})$, we have
\begin{align*}
\abr{\EE\sbr{\Lcal_{\Dcal_h,R_h}(\hat f_{h,R_h}, V_{h+1})} - \EE\sbr{\Lcal_{\Dcal_h,R_h}( \Tcal_h V_{h+1}, V_{h+1})}}\le \frac{\cd dH^2 \log (nB|\Rcal|H/\delta)}{n} + \veps_{\mathrm{apx}},
\end{align*}
Here $\cd$ is some universal constant, $\Vcal_{h+1}(R_{h+1}):=\{\mathrm{clip}_{[0,H]}\rbr{\max_a f_{h+1,R_{h+1}}(x_{h+1},a)}:f_{h+1,R_{h+1}}\in \Fcal_{h+1}(R_{h+1})\}$ for $h\in[H-1]$ and $\Vcal_{H}=\{\mathbf{0}\}$ are the state-value function class, $\Fcal_h(R_h)\coloneqq\{R_h+\mathrm{clip}_{[0,H]}\rbr{\langle \bar \phi_h,w_h\rangle}: \|w_h\|_2 \le B\}$  for $h\in[H]$ is the reward dependent Q-value function class.
\end{lemma}

\begin{proof}
We start with any fixed $h\in[H]$. Recall the definition $\hat f_{h,R_h}=R_h+\hat f_h$, where 
$\hat{f}_h = \argmin_{f_h\in \Fcal_h(\mathbf{0})}\Lcal_{\Dcal_h}(f_h, V_{h+1})$, $\Lcal_{\Dcal_h}(f_h, V_{h+1}):=\sum_{i=1}^n\left( f_h\rbr{x^{(i)}_h,a^{(i)}_h} - V_{h+1}\rbr{x_{h+1}^{(i)}}\right)^2$, and $\Fcal_h(\mathbf{0})\coloneqq\{\mathrm{clip}_{[0,H]}\rbr{\langle \bar \phi_h,w_h\rangle}: \|w_h\|_2 \le B\}$. Therefore we get $\Lcal_{\Dcal_h,R_h}(\hat f_{h,R_h}, V_{h+1}) =\Lcal_{\Dcal_h}(\hat f_h, V_{h+1})$ and
$\Lcal_{\Dcal_h,R_h}( \Tcal_h V_{h+1},$ $ V_{h+1})=\Lcal_{\Dcal_h}( \Tcal_h V_{h+1}-R_h, V_{h+1})=\Lcal_{\Dcal_h}(\langle\phi^*_h,\theta_{V_{h+1}}^*\rangle, V_{h+1}).$
Then it suffices to show the bound between $\Lcal_{\Dcal_h}(\hat f_h, V_{h+1})$ and $\Lcal_{\Dcal_h}(\langle\phi^*_h,\theta_{V_{h+1}}^*\rangle, V_{h+1})$.

Firstly, we show that condition \Cref{eq:fqi_rep_error} implies the following approximation guarantee for $\bar\phi_h$. For all $h ,V_{h+1}\in\{ \mathrm{clip}_{[0,H]}$ $(\max_{a}(R_{h+1}(x_{h+1},a) +\mathrm{clip}_{[0,H]}\rbr{\inner{\phi_{h+1}(x_{h+1},a)}{\theta_{h+1}}})):\phi_{h+1} \in \Phi_{h+1}, \|\theta_{h+1}\|_2 \le B,R\in\Rcal\}$, let $\bar w_{V_{h+1}} = \argmin_{\|w_h\|_2 \le B} \Lcal_{\Dcal_h}(\bar \phi_h, $ $w_h, V_{h+1})$ with $B\geq H\sqrt d$, then we have 
\begin{align*}
     \EE_{\rho_{h-3}^{+3}}\sbr{ \rbr{\inner{\bar \phi_h(x_h,a_h)}{\bar w_{V_{h+1}}} - \EE\sbr{V_{h+1}(x_{h+1}) | x_h,a_h} }^2 } \le \veps_{\mathrm{apx}}.
\end{align*}

This is because the order of taking max and clipping doesn't matter:
\begin{align*}
&\mathrm{clip}_{[0,H]}\rbr{\max_{a}\rbr{R_{h+1}(x_{h+1},a)+\mathrm{clip}_{[0,H]}\rbr{\inner{ \phi_{h+1}(x_{h+1},a)}{\theta_{h+1}}}}}\\
={}&\mathrm{clip}_{[0,H]}\rbr{\max_{a}\rbr{R_{h+1}(x_{h+1},a)+\inner{ \phi_{h+1}(x_{h+1},a)}{\theta_{h+1}}}}.
\end{align*} 

Since we have clipping, we now define $\Lcal_{\rho_{h-3}^{+3}}^{\mathrm{c}}(\phi_h,w_h,V_{h+1})\coloneqq \EE_{\rho_{h-3}^{+3}}[(\mathrm{clip}_{[0,H]}\rbr{\inner{\phi_h}{w_h}}-V_{h+1})^2]$ and $\Lcal_{\Dcal_h}^{\mathrm{c}}(\cdot)$ as its empirical version. Then we can follow the similar steps in \pref{lem:fqi_fastrate_unclipped} and get the concentration result. For any $R\in\Rcal,V_{h+1}\in\Vcal_{h+1}(R)$ and $\|w_h\|_2\le B$, we have that with probability at least $1-\delta'$
\begin{align*}
    &     \Big| \Lcal_{\rho_{h-3}^{+3}}^{\mathrm{c}}(\bar \phi_h,w_h,V_{h+1}) - \Lcal_{\rho_{h-3}^{+3}}^{\mathrm{c}}(  \phi_h^*,\theta^*_{V_{h+1}},V_{h+1})
    \\
    &~~~~~~~~~~~~- \rbr{ \Lcal_{\Dcal_h}^{\mathrm{c}}(\bar \phi_h,w_h,V_{h+1}) - \Lcal_{\Dcal_h}^{\mathrm{c}}( \phi^*_h,\theta^*_{V_{h+1}},V_{h+1})}\Big| \\  
    \le &~\frac{1}{2} \rbr{ \Lcal_{\rho_{h-3}^{+3}}^{\mathrm{c}}(\bar \phi_h,w_h,V_{h+1}) - \Lcal_{\rho_{h-3}^{+3}}^{\mathrm{c}}(  \phi_h^*,\theta^*_{V_{h+1}},V_{h+1})} + \frac{\cd' dH^2 \log (nB|\Rcal|/\delta')}{n},
\end{align*}
where $\cd'$ is some universal constant. Note that the slight difference is that here we will set the feature class to be $\{\bar\phi_h\}$ and $\{\bar\phi_{h+1}\}$ when calling \pref{lem:fqi_fastrate_unclipped}, change the norm constraints of the value function classes, and add add the clipping on the corresponding function classes there. This changes the range of the hypothesis functions to $16H^2$ and gets rid of the union bound over the feature classes. Then we have
\begin{align*}
    &~\Lcal_{\rho_{h-3}^{+3}}^{\mathrm{c}}(\bar \phi_h,w_h,V_{h+1}) - \Lcal_{\rho_{h-3}^{+3}}^{\mathrm{c}}(  \phi_h^*,\theta^*_{V_{h+1}},V_{h+1})
    \\
    \le&~    \rbr{\Lcal_{\rho_{h-3}^{+3}}^{\mathrm{c}}(\bar \phi_h,w_h,V_{h+1}) - \Lcal_{\rho_{h-3}^{+3}}^{\mathrm{c}}(\bar \phi_h,\bar w_{V_{h+1}},V_{h+1})}
    \\
    &\quad\quad +    \abr{\Lcal_{\rho_{h-3}^{+3}}^{\mathrm{c}}(\bar \phi_h,\bar w_{V_{h+1}},V_{h+1}) - \Lcal_{\rho_{h-3}^{+3}}^{\mathrm{c}}(  \phi_h^*,\theta^*_{V_{h+1}},V_{h+1})}
    \\
    \le&~ \rbr{\Lcal_{\rho_{h-3}^{+3}}^{\mathrm{c}}(\bar \phi_h,w_h,V_{h+1}) - \Lcal_{\rho_{h-3}^{+3}}^{\mathrm{c}}(\bar \phi_h,\bar w_{V_{h+1}},V_{h+1})} + \veps_{\mathrm{apx}}
    \\
    \le&~2\rbr{\Lcal_{\Dcal_h}^{\mathrm{c}}(\bar \phi_h,w_h,V_{h+1}) - \Lcal_{\Dcal_h}^{\mathrm{c}}(\bar \phi_h,\bar w_{V_{h+1}},V_{h+1})} + \frac{2\cd' dH^2 \log (nB|\Rcal|/\delta')}{n} + \veps_{\mathrm{apx}}
\end{align*}
The second inequality above is due to
\begin{align*}
&~\abr{\Lcal_{\rho_{h-3}^{+3}}^{\mathrm{c}}(\bar \phi_h,\bar w_{V_{h+1}},V_{h+1}) - \Lcal_{\rho_{h-3}^{+3}}^{\mathrm{c}}(  \phi_h^*,\theta^*_{V_{h+1}},V_{h+1})}
\\
=&~ \EE_{\rho_{h-3}^{+3}}\sbr{\rbr{\mathrm{clip}_{[0,H]}\rbr{\inner{\bar \phi_h(x_h,a_h)}{\bar w_{V_{h+1}}}} - \EE\sbr{V_{h+1}(x_{h+1}) | x_h,a_h} }^2 } 
\\
\le&~ \EE_{\rho_{h-3}^{+3}}\sbr{ \rbr{\inner{\bar \phi_h(x_h,a_h)}{\bar w_{V_{h+1}}} - \EE\sbr{V_{h+1}(x_{h+1}) | x_h,a_h} }^2 } \le \veps_{\mathrm{apx}}.
\end{align*}

From the definition, we have $\hat f_h=\argmin_{f_h\in\Fcal_h(\mathbf{0})} \Lcal_{\Dcal_h}^{\mathrm{c}}(f_h,V_{h+1})$. Noticing the structure of $\Fcal_h(\mathbf{0})$, we can write $\hat f_h=\mathrm{clip}_{[0,H]}\rbr{\langle \bar \phi_h, \hat w_h\rangle}$, where $\hat w_h=\argmin_{\|w_h\|_2\le B} \Lcal_{\Dcal_h}^{\mathrm{c}}(\bar \phi_h,w_h,V_{h+1})$. This implies $\Lcal_{\Dcal_h}^{\mathrm{c}}( \bar \phi_h,\hat w_h,V_{h+1}) - \Lcal_{\Dcal_h}^{\mathrm{c}}(  \bar \phi_h,\tilde w_{V_{h+1}},V_{h+1})\le 0.$

Union bounding over $h\in[H]$, and setting $\delta = \delta'/H$, we have that with probability at least $1-\delta$,
\begin{align*}
\abr{ \Lcal_{\rho_{h-3}^{+3}}^{\mathrm{c}}( \bar \phi_h,\hat w_h,V_{h+1}) - \Lcal_{\rho_{h-3}^{+3}}^{\mathrm{c}}(  \phi^*_h,\theta^*_{V_{h+1}},V_{h+1})}\le\frac{2\cd' dH^2 \log (nB|\Rcal|H/\delta)}{n} + \veps_{\mathrm{apx}}.
\end{align*}
Finally, noticing the property that $\EE\sbr{\Lcal_{\Dcal_h}(\hat f_h, V_{h+1})}=\Lcal_{\rho_{h-3}^{+3}}^{\mathrm{c}}(\bar\phi_h,\hat w_h,V_{h+1})$ and $\EE[\Lcal_{\Dcal_h,R_h}($ $\Tcal_h V_{h+1}, V_{h+1})]=\Lcal_{\rho_{h-3}^{+3}}(  \phi^*_h,\theta^*_{V_{h+1}},V_{h+1})=\Lcal_{\rho_{h-3}^{+3}}^{\mathrm{c}}(  \phi^*_h,\theta^*_{V_{h+1}},V_{h+1})$, we complete the proof.
\end{proof}

\subsection{Planning for Elliptical Reward Functions}\label{app:fqi_ellip}
In this part, we show the sample complexity of \textsc{fqi-elliptical}, which is specialized in planning for the elliptical reward class defined in \pref{lem:planning_sparse_reward}. The Q-value function class consists of linear function of all features in the feature class with reward appended. Specifically, we have $\Fcal_h(R_h)\coloneqq \{R_h+\langle \phi_h,w_h\rangle: \|w_h\|_2 \le \sqrt{d},\phi_h \in \Phi_h\}, h\in[H]$. We still use the full representation class, but compared with \textsc{fqi-full-class}, we use a different bound on the norm of the parameters. 

\begin{lemma}[Planning for elliptical reward functions]
\label{lem:planning_sparse_reward}
Assume that we have the exploratory dataset $\{\Dcal\}_{0:H-1}$ (collected from $\rho_{h-3}^{+3}$ and satisfies~\Cref{eq:dist_shift} for all $h \in [H]$). For $\delta \in (0,1)$ and any deterministic elliptical reward function $R\in\Rcal$, where $\Rcal\coloneqq \{R_{0:H-1} :R_{0:H-2}=\mathbf{0},R_{H-1} \in  \{\phi_{H-1}^\top \Gamma^{-1} \phi_{H-1}:\phi_{H-1}\in\Phi_{H-1}, \Gamma\in\RR^{d\times d}, \lambda_{\min} (\Gamma)\ge 1\}\}$, if we set
\begin{align*}
n\geq \frac{136\cc H^4d^3 \kappa K}{\beta^2}\log\rbr{\frac{68\cc  H^4d^3 \kappa K}{\beta^2}}+\frac{136\cc  H^4d^3 \kappa K}{\beta^2}\log\rbr{ \tfrac{2|\Phi|H}{\delta}},
\end{align*}
where $\cc$ is the constant in \pref{lem:fqi_fastrate_unclipped}, 
then with probability at least $1-\delta$, the policy $\hat{\pi}$ returned
by \textsc{FQI-elliptical} satisfies
\begin{align*}
\EE_{\hat{\pi}}\sbr{\sum_{h=0}^{H-1} R_h(x_h,a_h)} \geq \max_{\pi} \EE_\pi\sbr{\sum_{h=0}^{H-1} R_h(x_h,a_h)} - \beta.
\end{align*}
\end{lemma}
\begin{remark}
Notice that the elliptical reward function only has a non-zero value at timestep $H-1$.
\end{remark}

\begin{proof}
The proof mostly follows the steps in \pref{lem:real_world_optimization}. Since we consider the deterministic  elliptical reward function class, we apply \pref{lem:fqi_sparse} instead of \pref{lem:fqi_arbitrary} in Step (*). Then following a similar calculation gives us the result immediately.
\end{proof}

\begin{lemma}[Deviation bound for \pref{lem:planning_sparse_reward}]
\label{lem:fqi_sparse}
Consider the deterministic elliptical reward function classes $\Rcal\coloneqq \{R_{0:H-1} :R_{0:H-2}=\mathbf{0},R_{H-1} \in \Rcal_{H-1} \coloneqq \{\phi_{H-1}^\top \Gamma^{-1} \phi_{H-1}:\phi_{H-1}\in\Phi_{H-1}, \Gamma\in\RR^{d\times d}, \lambda_{\min} (\Gamma)\ge 1\}\}$, an exploratory dataset  $\Dcal_h\coloneqq \cbr{\rbr{x_h^{(i)},a_h^{(i)},x_{h+1}^{(i)}}}_{i=1}^n$ collected from $\rho_{h-3}^{+3}$, $h\in[H]$. With probability at least $1-\delta$,  $\forall\,R\in \Rcal, h\in[H],V_{h+1} \in \Vcal_{h+1}(R_{h+1})$, we have
\begin{align*}
\abr{\EE\sbr{\Lcal_{\Dcal_h,R_h}(\hat f_{h,R_h}, V_{h+1})} - \EE\sbr{\Lcal_{\Dcal_h,R_h}( \Tcal_h V_{h+1}, V_{h+1})}} \le\frac{68\cc d^3 \log (2n|\Phi|H/\delta)}{n},
\end{align*}
where $\cc$ is the constant in \pref{lem:fqi_fastrate_unclipped}, $\Vcal_{h+1}(R_{h+1}):=\{\mathrm{clip}_{[0,1]}\rbr{\max_a f_{h+1,R_{h+1}}(x_{h+1},a)}:f_{h+1,R_{h+1}}\in \Fcal_{h+1}(R_{h+1})\}$ for $h\in[H]$ and $\Vcal_{H}=\{\mathbf{0}\}$ is the state-value function class, and $\Fcal_h(R_h)\coloneqq \{R_h+\langle \phi_h,w_h\rangle: \|w_h\|_2 \le \sqrt{d},\phi_h \in \Phi_h\}$ for $h\in[H]$ is the reward dependent Q-value function class.

\end{lemma}

\begin{proof}
Firstly, from \pref{lem:covering_reward_cls}, we know that there exists a $\gamma$-cover $\Ccal_{\Rcal_{H-1}^\textsc{ell},\gamma}$ for the reward class $\Rcal$. For any fixed $\tilde R\in\Ccal_{\Rcal_{H-1}^\textsc{ell},\gamma}$, we can follow the similar steps in \pref{lem:fqi_arbitrary} to get a concentration result. The differences are that the norm of $w_h$ is now bounded by $\sqrt d$ instead of $H\sqrt d$, and we clip to $[0,1]$. Therefore, for this fixed $\tilde R$, with probability at least $1-\delta'$, we have that $\forall h\in[H],\tilde V_{h+1}\in\Vcal_{h+1}(\tilde R_{h+1})$,
\begin{align*}
\abr{\EE\sbr{\Lcal_{\Dcal_h,\tilde R_h}(\hat f_{h,\tilde R_h}, \tilde V_{h+1})} - \EE\sbr{\Lcal_{\Dcal_h,\tilde R_h}( \Tcal_h \tilde V_{h+1}, \tilde V_{h+1})}} \le  \frac{16\cc d \log (4nH|\Phi|/\delta)}{n}.
\end{align*}

Union bounding over all $\tilde R\in\Ccal_{\Rcal_{H-1}^\textsc{ell},\gamma}$ and setting $\delta=\delta'/|\Ccal_{\Rcal_{H-1}^\textsc{ell},\gamma}|$, with probability at least $1-\delta$, we have
\begin{align*}
\abr{\EE\sbr{\Lcal_{\Dcal_h,\tilde R_h}(\hat f_{h,\tilde R_h},\tilde V_{h+1})}\hspace{-.2em}  -\hspace{-.1em}  \EE\sbr{\Lcal_{\Dcal_h,\tilde R_h}( \Tcal_h \tilde V_{h+1}, \tilde V_{h+1})}}\hspace{-.2em} \le \hspace{-.2em} 
\frac{16\cc d \log (4nH|\Phi||\Ccal_{\Rcal_{H-1}^\textsc{ell},\gamma}|/\delta)}{n}.
\end{align*}

Notice that $\Ccal_{\Rcal_{H-1}^\textsc{ell},\gamma}$ is a $\gamma$-cover of $\Rcal$, for any $R\in\Rcal$, there exists $\tilde R\in \Ccal_{\Rcal_{H-1}^\textsc{ell},\gamma}$, such that $\|R_h-\tilde R_h\|_\infty\le \gamma$. Therefore for any $f_{h,R_h}\in \Fcal_{h}(R_h)$ and $V_{h}\in\Vcal_h(R_h)$, there exists some $f_{h,\tilde R_h}\in \Fcal(\tilde R_h)$ and $ V_h\in\Vcal(\tilde R_h)$ that satisfy $ \| f_{h,R_h}- f_{h,\tilde R_h}\|_\infty\le \gamma$ and $\|\tilde V_h-V_h\|_\infty\le \gamma$. Hence, for any $R\in\Rcal, h\in[H],V_{h+1} \in \Vcal_{h+1}(R_{h+1})$, with probability at least $1-\delta$,
\begin{align*}
&\abr{\EE\sbr{\Lcal_{\Dcal_h,R_h}(\hat f_{h,R_h}, V_{h+1})} - \EE\sbr{\Lcal_{\Dcal_h,R_h}( \Tcal_h V_{h+1}, V_{h+1})}} \\
\le{} &\abr{\EE\sbr{\Lcal_{\Dcal_h,\tilde R_h}(\hat f_{h,\tilde R_h}, \tilde V_{h+1})} - \EE\sbr{\Lcal_{\Dcal_h,\tilde R_h}( \Tcal_h \tilde V_{h+1}, \tilde V_{h+1})}}+36\sqrt d \gamma\\
\le{} & \frac{16\cc d \log (4nH|\Phi||\Ccal_{\Rcal_{H-1}^\textsc{ell},\gamma}|/\delta)}{n} + 36\sqrt d \gamma
\le {} \frac{68\cc d^3 \log (2n|\Phi|H/\delta)}{n}.
\end{align*}
The last inequality is obtained by choosing $\gamma=\frac{\sqrt d}{n}$ and noticing $|\Ccal_{\Rcal_{H-1}^\textsc{ell},\gamma}| = |\Phi_{H-1}|(2\sqrt d/\gamma)^{d^2}$ $\le |\Phi|(2n)^{d^2}$. This completes the proof.
\end{proof}

\section{\fqe Result} \label{app:fqe}
In this section, we provide \fqe \citep{le2019batch} analysis. In \pref{app:fqe_framework}, we discuss the \fqe algorithm. For simplicity, we use the horizon $H$ in \pref{alg:linear_fqe} and all statements. When \fqe is called with a smaller horizon $\tilde H\le H$, we can just replace all $H$ by $\tilde H$. The analyses still go through and the sample complexity will not exceed the one instantiated with $H$. We want to mention that we abuse some notations in this section. For example, $\Fcal_h,\Vcal_h$ may have different meanings from the main text. However, they should be clear within the context.

\subsection{\fqe Algorithm}
\label{app:fqe_framework}
In this part, we present FQE algorithm (\pref{alg:linear_fqe}), which is similar to \pref{alg:linear_fqi}. The difference is that we now approximate the Bellman equation instead of the Bellman optimality equation. More specifically, $\hat V$ is defined by evaluating $\hat f$ with $\pi$ rather than maximizing over all actions in $\hat f$. In addition, we return the estimated expected return instead of the greedy policy. The details can be found below. When calling \pref{alg:linear_fqe} and there is no confusion, we sometimes drop the input function class (see \pref{line:fqe_input} in \pref{alg:linear_fqe}) for simplicity. 

\begin{algorithm}[htp]
\begin{algorithmic}[1]
\STATE \textbf{input:} (1) exploratory dataset $\{\Dcal\}_{0:H-1}$ sampled from $\rho_{h-3}^{+3}$ with size $n$ at each level $h\in[H]$, (2) reward function $R=R_{0:H-1}$ with $R_h:\Xcal\times\Acal \to [0,1],\forall h\in[H]$, (3) function class: $\Fcal_h(R_h)\coloneqq \{R_h+\langle  \phi_h,w_h\rangle:\|w_h\|_2 \le \sqrt{d},\phi_h\in\Phi_h\}, h\in[H]$, (4) evaluated policy $\pi$. \label{line:fqe_input}
\STATE Set $\hat{V}_H(x) = 0$.
\FOR{$h= H-1,\ldots,0$}
\STATE Pick $n$ samples $\left\{\rbr{x^{(i)}_h,a^{(i)}_h,x^{(i)}_{h+1}}\right\}_{i=1}^n$ from the exploratory dataset $\Dcal_h$.
\STATE Solve least squares problem:
\begin{align*}
\hat{f}_{h,R_h} \gets \argmin_{f_h\in \Fcal_h(R_h)} \Lcal_{\Dcal_h,R_h}(f_h,\hat V_{h+1}),
\end{align*}
where
$\Lcal_{\Dcal_h,R_h}(f_h,\hat V_{h+1}):=\sum_{i=1}^n\rbr{ f_h\rbr{x^{(i)}_h,a^{(i)}_h} - R_h\rbr{x_h^{(i)},a_h^{(i)}} - \hat V_{h+1}\rbr{x_{h+1}^{(i)}}}^2$.
\STATE Define $\hat{V}_{h}(x) = \mathrm{clip}_{[0,1]}\rbr{\hat{f}_{h,R_h}(x,\pi)}$.
\ENDFOR
\STATE \textbf{return} $\hat v^\pi=\hat V_0(x_0)$.
\end{algorithmic}
\caption{\textsc{FQE}: Fitted Q-Evaluation}
\label{alg:linear_fqe}
\end{algorithm}

\subsection{\fqe Analysis}
\begin{lemma}[\fqe for a reward class]
\label{lem:fqe}
Assume that we have the exploratory dataset $\{\Dcal\}_{0:H-1}$ (collected from $\rho_{h-3}^{+3}$ and satisfies~\Cref{eq:dist_shift} for all $h \in [H]$), and we are given a finite deterministic reward class $\Rcal =\Rcal_0\times\ldots\times\Rcal_{H-1}$ with $\Rcal_{H-1}\subseteq(
\Xcal \times \Acal \to [0,1])$ and $\Rcal_h=\{\zero\},\forall h\in[H-1]$. For $\delta \in (0,1)$, any policy $\pi$ encountered when running \pref{alg:FQI_ell_planner} (see the statement of \pref{lem:dev_fqe} for details), and any reward function $R\in\Rcal$, if we set
\begin{align*}
n\geq \frac{2\ce H^4 d^3\kappa K^2 \log^2(Kd)}{\beta^2} \rbr{\log \rbr{\frac{\ce H^4 d^3\kappa K^2 \log^2(Kd)}{\beta^2}} +  \log\rbr{ \tfrac{2|\Phi||\Rcal|H}{\delta}}},
\end{align*}
where $\ce$ is the constant in \pref{lem:dev_fqe}, then with probability at least $1-\delta$, the value $\hat v^\pi$ returned
by \textsc{\fqe} (\pref{alg:linear_fqe}) satisfies
\begin{align*}
\abr{\EE_{\pi}\sbr{\sum_{h=0}^{H-1} R_h(x_h,a_h)}-\hat v^\pi}=\abr{v^\pi_R-\hat v^\pi} \leq  \beta.
\end{align*}
\end{lemma}

\begin{proof}
This proof shares some similarity as the proof of \pref{lem:real_world_optimization}, so we will only show the different steps. For any fixed $R\in\Rcal$, we have 
\begin{align*}
{}&\abr{v^\pi_R - \hat v^{\pi}}={}\abr{V^\pi_0(x_0)-\hat f_{0,R_0}(x_0)} 
={}\abr{\EE_{\pi}\sbr{Q^\pi_0(x_0,a_0)-\hat f_{0,R_0}(x_0,a_0)}}
\\
={}&\abr{\EE_\pi\sbr{R_0(x_0,a_0) + \EE\sbr{V^\pi_1(x_1)\mid x_0,a_0} - \hat f_{0,R_0}(x_0,a_0)}}\\
\le{}&\EE_\pi\sbr{\abr{R_0(x_0,a_0) + \EE\sbr{V^\pi_1(x_1)\mid x_0,a_0} - \hat f_{0,R_0}(x_0,a_0)}} \\
\leq{}& \EE_\pi\sbr{\abr{\EE\sbr{V^\pi_1(x_1)\mid x_0,a_0} - \EE\sbr{\hat{V}_1(x_1) \mid x_0,a_0}}}\\
&\quad + \EE_\pi\sbr{\abr{R_0(x_0,a_0) + \EE\sbr{\hat{V}_1(x_1) \mid x_0,a_0} - \hat f_{0,R_0}(x_0,a_0)}}\\
\leq{}&\EE_\pi\sbr{\abr{V^\pi_1(x_1) - \hat{V}_1(x_1)}+ \abr{(\Tcal_0 \hat V_1)(x_0,a_0) - \hat f_{0,R_0}(x_0,a_0)}},
\end{align*}
where the last inequality is due to Jensen's inequality.

Continuing unrolling to $h=H-1$, we get
\begin{align*}
\abr{v^\pi_R - \hat v^{\pi}} \leq{}& \sum_{h=0}^{H-1}\EE_{\pi}\abr{(\Tcal_h \hat V_h)(x_h,a_h) - \hat f_{h,R_h}(x_h,a_h)}\\
\leq {} & \sum_{h=0}^{H-1} \sqrt{\kappa K \EE_{\rho_{h-3}^{
+3}}\sbr{\sbr{(\Tcal_h \hat V_{h+1})(x_h,a_h) - \hat f_{h,R_h}(x_h,a_h)}^2}}, 
\end{align*}
where the last inequality is due to condition~\Cref{eq:dist_shift}.

Further, we have that with probability at least $1-\delta$,
\begin{align*}
&\EE_{\rho_{h-3}^{+3}}\sbr{\rbr{(\Tcal_h \hat V_{h+1})(x_h,a_h) - \hat f_{h,R_h}(x_h,a_h)}^2} \nonumber\\
={}&\EE\sbr{\Lcal_{\Dcal_h,R_h}(\hat f_{h,R_h},\hat V_{h+1})} - \EE\sbr{\Lcal_{\Dcal_h,R_h}( \Tcal_h \hat V_{h+1},\hat V_{h+1})}\\
\le{}& \frac{\ce K d^3\log^2(Kd)\log\rbr{n|\Phi||\Rcal|H/\delta}}{n}. \tag{\pref{lem:dev_fqe}}
\end{align*}

Plugging this back into the overall value performance difference, the bound is
\begin{align*}
\abr{v^\pi_R - \hat v^{\pi}} \leq H^2\sqrt{\kappa K}\sqrt{ \frac{\ce K d^3\log^2(Kd)\log\rbr{n|\Phi||\Rcal|H/\delta}}{n}}.
\end{align*}
Setting RHS to be less than $\beta$ and reorganize, we get  $$n\geq \frac{\ce H^4\kappa K^2 d^3\log^2(Kd)\log\rbr{n|\Phi||\Rcal|H/\delta}}{\beta^2}.$$

A sufficient condition for the inequality above is 
$$n\geq \frac{2\ce H^4 d^3\kappa K^2 \log^2(Kd)}{\beta^2} \rbr{\log \rbr{\frac{\ce H^4 d^3\kappa K^2 \log^2(Kd)}{\beta^2}} +  \log\rbr{ \tfrac{2|\Phi||\Rcal|H}{\delta}}}.$$
Notice that all the analyses above and \pref{lem:dev_fqe} hold for any $R\in\Rcal$, we complete the proof.
\end{proof}

\begin{lemma}[Deviation bound for \pref{lem:fqe}]\label{lem:dev_fqe}
Assume that we have an exploratory dataset  $\Dcal_h\coloneqq \cbr{\rbr{x_h^{(i)},a_h^{(i)},x_{h+1}^{(i)}}}_{i=1}^n$ collected from $\rho_{h-3}^{+3}$, $h\in[H]$, and a finite deterministic reward class $\Rcal =\Rcal_0\times\ldots\times\Rcal_{H-1}$ with $\Rcal_{H-1}\subseteq(
\Xcal \times \Acal \to [0,1])$ and $\Rcal_h=\{\zero\},\forall h\in[H-1]$. Then, with probability at least $1-\delta$, $\forall R\in \Rcal,\pi\in\Pi,h\in[H],V_{h+1} \in\Vcal_{h+1}(R_{h+1})$, we have
\begin{align*}
\abr{\EE\sbr{\Lcal_{\Dcal_h,R_h}(\hat f_{h,R_h}, V_{h+1})} - \EE\sbr{\Lcal_{\Dcal_h,R_h}( \Tcal_h V_{h+1}, V_{h+1})}}\le \frac{\ce K d^3\log^2(Kd)\log\rbr{n|\Phi||\Rcal|H/\delta}}{n}.
\end{align*}
Here $\ce$ is some universal constant. The policy class is defined as $\Pi=\Pi_0\times\ldots\Pi_{H-1}$, where $\Pi_h=\{\argmax_a (\langle\phi_h(x_h,a),w_h\rangle:\phi_h\in\Phi_h,\|w_h\|_2\le\sqrt d \}$ for $h\in[H-1]$ and $\Pi_{H-1}=\{\argmax_a (R_{H-1}(x_{H-1},a)):R\in \Rcal_{ H-1}^\textsc{ell}\}$ ($\Rcal_{ H-1}^\textsc{ell}$ is defined in \pref{lem:covering_reward_cls}). Additionally, $\Vcal_{h+1}(R_{h+1}):=\{\mathrm{clip}_{[0,1]}\rbr{f_{h+1,R_{h+1}}(x_{h+1},\pi_{h+1}(x_{h+1})}:f_{h+1,R_{h+1}}\in \Fcal_{h+1}(R_{h+1}),$ $\pi_{h+1}\in \Pi_{h+1}\}$ for $h\in[H-1]$, $\Vcal_{H}=\{\mathbf{0}\}$ are the state-value function classes and $\Fcal_h(R_h)\coloneqq\{R_h+\langle \phi_h,w_h\rangle: \phi_h\in\Phi_h,\|w_h\|_2 \le \sqrt d\}$  for $h\in[H]$ is the reward dependent Q-value function class.
\end{lemma}
\begin{remark} From \pref{alg:linear_fqi}, we know that the policy class defined in the lemma statement includes all possible policies that can be encountered when running \pref{alg:linear_fqi}.
\end{remark}

\begin{proof}
For the last level $h=H-1$, we know that $\hat f_{H-1}=\zero$ and $\hat f_{H-1,R_{H-1}}=R_{H-1}+\hat f_h=R_{H-1}$ because $\hat V_{H}=\zero$ and the reward $R_{H-1}$ is known. Thus, for $h=H-1$, we have
\[\abr{\EE\sbr{\Lcal_{\Dcal_h,R_h}(\hat f_{h,R_h}, V_{h+1})} - \EE\sbr{\Lcal_{\Dcal_h,R_h}( \Tcal_h V_{h+1}, V_{h+1})}}=0.\]

Then we consider the cases that $h\le H-3$ and $h=H-2$. Notice that any $\pi\in\Pi$ is a greedy policy, thus we have
\[
\mathrm{clip}_{[0,1]}\rbr{f_{h+1,R_{h+1}}(x_{h+1},\pi_{h+1}(x_{h+1})}=\EE_{a_{h+1}\sim\pi_{h+1}}\sbr{\mathrm{clip}_{[0,1]}\rbr{f_{h+1,R_{h+1}}(x_{h+1},a_{h+1})}},
\]
which implies that the $\Vcal$ class has the same format as that in \pref{lem:fastrate_sqloss_fqe_1} and \pref{lem:fastrate_sqloss_fqe_2}. 

Now we can invoke these lemmas by setting $L=1$ and $B=\sqrt d$ and get concentration results. The remaining steps follow similar ones starting at \Cref{eq:used_in_fqe} in the proof of \pref{lem:fqi_arbitrary}.
\end{proof}

\begin{corollary}[\fqe for the elliptical reward class]
\label{corr:fqe_entire_class}
Assume that we have the exploratory dataset $\{\Dcal\}_{0:H-1}$ (collected from $\rho_{h-3}^{+3}$ and satisfies~\Cref{eq:dist_shift} for all $h \in [H]$) and consider the elliptical reward class $\Rcal\coloneqq \{R_{0:H-1} :R_{0:H-2}=\mathbf{0},R_{H-1} \in  \{\phi_{H-1}^\top \Gamma^{-1} \phi_{H-1}:\phi_{H-1}\in\Phi_{H-1}, \Gamma\in\RR^{d\times d}, \lambda_{\min} (\Gamma)\ge 1\}\}$. For $\delta \in (0,1)$, any policy $\pi$ encountered when running \pref{alg:FQI_ell_planner} (see the statement of \pref{lem:dev_fqe} for details), and any reward function $R\in\Rcal$, if we set
\begin{align*}
n\geq \frac{4\cf  H^4 d^5\kappa K^2 \log^2(Kd)}{\beta^2} \rbr{\log \rbr{\frac{2\cf  H^4 d^5\kappa K^2 \log^2(Kd)}{\beta^2}} +  \log\rbr{ \tfrac{2|\Phi|H}{\delta}}}.
\end{align*}
then with probability at least $1-\delta$, the value $\hat v^\pi$ returned
by \textsc{\fqe} (\pref{alg:linear_fqe}) satisfies
\begin{align*}
\abr{\EE_{\pi}\sbr{\sum_{h=0}^{H-1} R_h(x_h,a_h)}-\hat v^\pi}=\abr{v^\pi_R-\hat v^\pi} \leq  \beta.
\end{align*}
\end{corollary}

\begin{proof}
Similar as the proof of \pref{lem:fqi_sparse}, we first instantiate the result \pref{lem:dev_fqe} with the finite reward class $\Ccal_{\Rcal_{H-1}^\textsc{ell},\gamma}$, where $\gamma=\frac{\sqrt{d}}{n}$. Then we bound the difference between any state function and its closest function in the cover. This will give us a version of \pref{lem:dev_fqe} with the elliptical class $\Rcal\coloneqq \{R_{0:H-1} :R_{0:H-2}=\mathbf{0},R_{H-1} \in  \{\phi_{H-1}^\top \Gamma^{-1} \phi_{H-1}:\phi_{H-1}\in\Phi_{H-1}, \Gamma\in\RR^{d\times d}, \lambda_{\min} (\Gamma)\ge 1\}\}$ and 
\begin{align}
\label{eq:fqe_entire_ellip}
\abr{\EE\hspace{-.15em}\sbr{\Lcal_{\Dcal_h,R_h}(\hat f_{h,R_h},\hspace{-.25em} V_{h+1})} \hspace{-.25em}-\hspace{-.15em} \EE\hspace{-.15em}\sbr{\Lcal_{\Dcal_h,R_h}( \Tcal_h V_{h+1},\hspace{-.15em} V_{h+1})}}\hspace{-.25em}\le\hspace{-.15em} \frac{2\cf K d^5\log^2(Kd)\log\rbr{n|\Phi|H/\delta}}{n}.
\end{align}

The final result can be obtained by following the proof of \pref{lem:fqe}, while we use \Cref{eq:fqe_entire_ellip} instead of the result in \pref{lem:dev_fqe}.
\end{proof}

\section{Auxiliary Results}
\label{app:aux_lemmas}
In this section, we provide detailed proofs for auxiliary lemmas.
\subsection{Proof of \pref{lem:linMDP_expn}}\label{app:pf_lem1}
In this part, we provide the proof of \pref{lem:linMDP_expn} for completeness. This result is widely used throughout the paper.
\begin{lemma}[Restatement of \pref{lem:linMDP_expn}]
For a low-rank MDP $\Mcal$ with embedding dimension $d$, for any function $f : \Xcal \rightarrow [0,1]$, we have
\begin{align*}
    \EE \sbr{f(x_{h+1})|x_h,a_h } = \inner{\sphih}{\theta^*_f}
\end{align*}
where $\theta^*_f \in \R^d$ and we have $\|\theta^*_f\|_2 \le \sqrt{d}$. A similar linear representation is true for $\EE_{a \sim \pi_{h+1}}[f(x_{h+1}, a)|x_h,a_h]$ where $f: \Xcal \times \Acal \rightarrow [0,1]$ and a policy $\pi_{h+1}: \Xcal \rightarrow \Delta(\Acal)$.
\end{lemma}

\begin{proof}
For state-value function $f$, we have 
\begin{align*}
    \EE \sbr{f(x_{h+1})|x_h,a_h } &= \int f(x_{h+1})T_{h}(x_{h+1}|x_h,a_h) d(x_{h+1})
    \\
    &=\int f(x_{h+1})\inner{\sphih}{\mu_h^*(x_{h+1})} d(x_{h+1})
    \\
    &=\inner{\sphih}{\int f(x_{h+1})\mu_h^*(x_{h+1})d(x_{h+1})}=\inner{\sphih}{\theta_f^*}, 
\end{align*}
where $\theta_f^*\coloneqq \int f(x_{h+1})\mu_h^*(x_{h+1})d(x_{h+1})$ is a function of $f$. Additionally, we obtain $\|\theta_f^*\|_2\le \sqrt d$ from \pref{def:lowrank}.

For Q-value function $f$, we similarly have
\begin{align*}
    \EE_{a\sim\pi_{h+1}} \sbr{f(x_{h+1},a)|x_h,a_h } =\inner{\sphih}{\theta_f^*}, 
\end{align*}
where $\theta_f^*\coloneqq \iint f(x_{h+1},a_{h+1})\pi(a_{h+1}|x_{h+1}) \mu_h^*(x_{h+1})d(x_{h+1})d(a_{h+1})$ and $\|\theta_f^*\|_2 \le \sqrt d$.
\end{proof}

\subsection{Deviation Bounds for Regression with Squared Loss}
\label{app:dev_bound}
In this section, we derive a generalization error bound for squared loss for a class $\Fcal$ which subsumes the discriminator classes $\Fcal_{h}$ and $\Gcal_h$ in the main text. In this section, we prove the bounds for a prespecified $h\in[H]$ and drop $h$ and $h+1$ subscripts for simplicity. When we apply \pref{lem:fastrate_sqloss_fqe_1} in other parts of the paper, usually $\rbr{x^{(i)}, a^{(i)}, x'^{(i)}}$ tuples stands for $\rbr{x_{h}^{(i)}, a_h^{(i)}, x_{h+1}^{(i)}}$ tuples, function classes $\Phi,\Phi'$ refers to $\Phi_h,\Phi_{h+1}$, and $\Rcal$ refers to $\Rcal_h$. We abuse the notations of $\Fcal,\Gcal,\Phi,\Rcal,h,w, \Theta, \varepsilon$ and they have different meanings from the main text. 

\begin{lemma}
\label{lem:fastrate_sqloss_fqe_1}
For a dataset $\Dcal \coloneqq \cbr{\rbr{x^{(i)}, a^{(i)}, x'^{(i)}}}_{i=1}^n \sim \rho$, finite feature classes $\Phi$ and $\Phi'$, and a  finite reward function class $\Rcal\subseteq ( \Xcal \times \Acal \rightarrow [0,1])$, we can show that, with probability at least $1-\delta$,
\begin{align*}
    & \abr{\Lcal_{\rho}( \phi,w,V) - \Lcal_{\rho}(  \phi^*,\theta^*_V,V) - \rbr{ \Lcal_{\Dcal}( \phi,w,V) - \Lcal_{\Dcal}( \phi^*,\theta^*_V,V)}} \\
    \le{}& \frac{1}{2} \rbr{ \Lcal_{\rho}( \phi,w,V) - \Lcal_{\rho}( \phi^*,\theta^*_V,V)} + \frac{\ca K d\log^2(Kd)(B+L\sqrt d)^2\log\rbr{n|\Phi||\Phi'||\Rcal|/\delta}}{n}
\end{align*}
for all $\phi \in \Phi$, $\|w\|_2 \le B, V\in\Vcal \coloneqq \{\EE_{a'\sim \pi'(x')}[f(x',a')]: f\in\Fcal,\pi'\in\Pi'\}$, where $\Fcal \coloneqq \{f(x',a') = \mathrm{clip}_{[0,L]} (R(x',$ $a') + \langle\phi'(x',a'),\theta\rangle): \phi' \in \Phi', \|\theta\|_2 \le B, R \in \Rcal\}$, $\Pi'=\{\argmax_{a'}(\langle \phi'(x',a'),\theta \rangle),\phi'\in\Phi',\|\theta\|_2 \le B\}$, and $\ca$ is some universal constant.
\end{lemma}

\begin{proof}
Firstly recall that from \pref{lem:linMDP_expn}, for any $V\in\Vcal$, we have $\EE[V(x')|x,a]=\langle \phi^*(x,a),\theta^*_V \rangle$ with $\|\theta^*_V\|_2\le L\sqrt d$. From the definition, we can decompose the policy class $\Pi'$ and the value function class $\Fcal$ according to corresponding features and rewards. We have that $\Pi'=\bigcup_{\phi'_{\Pi'}\in\Phi'}\Pi'(\phi'_{\Pi'})$ and $\Fcal=\bigcup_{\phi'_{\Fcal}\in\Phi',R\in\Rcal}\Fcal(\phi'_{\Fcal},R)$, where $\Pi'(\phi'_{\Pi'})\coloneqq\{\argmax_{a'}(\langle\phi_{\Pi'}'(x',a'),\theta \rangle):\|\theta\|_2 \le B\}$ and $\Fcal(\phi'_{\Fcal},R)\coloneqq \{f(x',a') = \mathrm{clip}_{[0,L]} (R(x',$ $a')  +\langle\phi_{\Fcal}'(x',a'),\theta\rangle): \|\theta\|_2 \le B\}$.

Then, we start with fixed $\phi,\phi_{\Pi'}',\phi'_{\Fcal},R$ and show the concentration result related to $\phi,\Fcal(\phi'_{\Fcal},R)$ and $\Pi'(\phi'_{\Pi'})$. We also define function classes $\Gcal_1(\phi)=\{\langle\phi,w_1\rangle:\|w_1\|_2\le B\},$ and $\Gcal_2=\{\langle\phi^*,w_2\rangle:\|w_2\|_2\le L\sqrt d\}$. From the fact of pseudo dimension of linear function class and Natarajan dimension \citep[Theorem 21]{daniely2011multiclass} and notice that the usage of clipping in $\Fcal(\phi'_{\Fcal},R)$ will not increase the pseudo dimension, we know that 
\begin{align*}
\Pdim(\Fcal(\phi'_{\Fcal},R))\le  d+1,\quad  \Pdim(\Gcal_1)\le d, \quad \Pdim(\Gcal_2)\le   d, \quad \Ndim(\Pi'(\phi'_{\Pi'}))\le c _1 d\log (d),
\end{align*}
where $c_1>0$ is some universal constant.

Consider the hypothesis class $\Hcal(\phi,\phi_{\Pi'}',\phi'_{\Fcal},R)$
\begin{align}
    \Hcal(\phi,\phi_{\Pi'}',\phi'_{\Fcal},R)&:=\Big\{(x,a,x')\rightarrow (g_1(x,a)-f(x',\pi'))^2 - (g_2(x,a)-f(x',\pi'))^2:
    \nonumber\\
    & \quad\quad f\in\Fcal(\phi'_{\Fcal},R),g_1\in\Gcal_1,g_2\in\Gcal_2, \pi'\in\Pi'(\phi_{\Pi'}')\Big\}.\label{eq:formula_h}
\end{align}
For any $h\in\Hcal(\phi,\phi_{\Pi'}',\phi'_{\Fcal},R)$, we can write it as
\begin{align*}
h(x,a,x')=(g_1(x,a))^2-(g_2(x,a))^2-\rbr{ 2\sum_{a'\in\Acal}f(x',a')\one[a'=\pi'(x')]} \rbr{g_1(x,a) - g_2(x,a)},
\end{align*}
where $\one[\cdot]$ is the indicator function. 

Here $\Hcal(\phi,\phi_{\Pi'}',\phi'_{\Fcal},R)$ is a composited class of $\Fcal(\phi'_{\Fcal},R), \Gcal_1,\Gcal_2, \Pi'(\phi_{\Pi'}')$ and the compositions belong to \pref{lem:pseudodim_com}. Hence we obtain that 
\[
\Pdim(\Hcal(\phi,\phi_{\Pi'}',\phi'_{\Fcal},R))\le c_2 K d\log^2(Kd)=:d',
\]
where $c_2$ is some universal constant. We briefly discuss how to use the compositions to bound the pseudo dimension of the most complex term $\sum_{a'\in\Acal}f(x',a')\one[a'=\pi'(x')]$, since other compositions are easy to see. We first need to augment $a'$ to the domain (i.e., use domain $(x',a')$, whether keeping $(x,a)$ does not matter because it do not depend on $(x,a)$). Then notice that for any fixed $\tilde a'\in\Acal$, $f(x',\tilde a')\one[\tilde a'=\pi'(x')]$ is a map from $(x',a')$ to $\RR$, we can apply part 3 of \pref{lem:pseudodim_com}. Finally, by part 2 of \pref{lem:pseudodim_com}, we take the summation over $\Acal$ we get the final bound. The term $a'$ does not show up in the domain of $\Hcal$ as its dependence disappears after we take the summation over $\Acal$.


For any $h\in\Hcal(\phi,\phi_{\Pi'}',\phi'_{\Fcal},R)$, its range is bounded as $|h(\cdot)|\le 4(B+L\sqrt d)^2$. By \pref{corr:uni_bern_conf}, with probability at least $1-\delta'$, we have the concentration result 
\begin{align*}
	&~\abr{\EE[h(x,a,x')]-\frac{1}{n}\sum_{i=1}^n h\rbr{x^{(i)},a^{(i)},x'^{(i)}}} 
	\\
	\le&~\sqrt{\frac{768d'\VV[h(x,a,x')]\log\rbr{n/\delta'}}{n}}+\frac{6144d'(B+L\sqrt d)^2\log\rbr{n/\delta'}}{n}  \;\;\forall h\in\Hcal(\phi,\phi_{\Pi'}',\phi'_{\Fcal},R).
    \end{align*}

Union bounding over $\phi_{\Pi'}'\in\Phi',\phi'_{\Fcal}\in\Phi',R\in\Rcal$, we get that for any fixed $\phi$ and any $\|w_1\|_2\le B, V\in\Vcal$, with probability at least $1-|\Phi'|^2|\Rcal|\delta'$ we have 
\begin{align}
\label{eq:fqe_bern_1}
    & \abr{\Lcal_{\rho}( \phi,w_1,V) - \Lcal_{\rho}(  \phi^*,\theta^*_V,V) - \rbr{ \Lcal_{\Dcal}( \phi,w_1,V) - \Lcal_{\Dcal}( \phi^*,\theta^*_V,V)}} 
    \notag\\
    \le&~\sqrt{\frac{768d'\VV[(\langle\phi(x,a),w_1\rangle-V(x'))^2 - (\langle \phi^*(x,a),\theta^*_V\rangle -V(x'))^2]\log\rbr{n/\delta'}}{n}}
    \notag\\
    &\quad +\frac{6144d' (B+L\sqrt d)^2\log\rbr{n/\delta'}}{n}.
\end{align}

For the variance term, we can bound it as the following
\begin{align*}
&~\VV[(\langle\phi(x,a),w_1\rangle-V(x'))^2 - (\langle \phi^*(x,a),\theta^*_V\rangle -V(x'))^2]
\\
\le&~\EE\sbr{\rbr{(\langle\phi(x,a),w_1\rangle-V(x'))^2 - (\langle \phi^*(x,a),\theta^*_V\rangle -V(x'))^2}^2}
\\
\le&~4(B+L\sqrt d)^2\EE\sbr{\rbr{\langle\phi(x,a),w_1\rangle-\langle\phi^*(x,a),\theta^*_V\rangle}^2}\\
=&~4(B+L\sqrt d)^2\rbr{\Lcal_{\rho}( \phi,w_1,V) - \Lcal_{\rho}( \phi^*,\theta^*_V,V)}.
\end{align*}

Plugging this into \Cref{eq:fqe_bern_1} and invoking AM-GM inequality gives us that for any fixed $\phi$ and any $\|w_1\|_2\le B, V\in\Vcal$, with probability at least $1-|\Phi'|^2|\Rcal|\delta'$ 
\begin{align*}
&\abr{\Lcal_{\rho}( \phi,w_1,V) - \Lcal_{\rho}(  \phi^*,\theta^*_V,V) - \rbr{ \Lcal_{\Dcal}( \phi,w_1,V) - \Lcal_{\Dcal}( \phi^*,\theta^*_V,V)}} 
\notag\\
\le&~\frac{1}{2}\rbr{\Lcal_{\rho}( \phi,w_1,V) - \Lcal_{\rho}( \phi^*,\theta^*_V,V)}+\frac{(768+6144)d'(B+L\sqrt d)^2\log\rbr{n/\delta'}}{n}.
\end{align*}

The final bound is obtained by union bounding over $\phi\in\Phi$, setting $\delta'=\delta/(|\Phi||\Phi'|^2|\Rcal|)$, and noticing the definition of $d'$.
\end{proof}
\begin{lemma}
\label{lem:fastrate_sqloss_fqe_2}
For a dataset $\Dcal \coloneqq \cbr{\rbr{x^{(i)}, a^{(i)}, x'^{(i)}}}_{i=1}^n \sim \rho$, finite feature classes $\Phi$ and $\Phi'$, and a  finite reward function class $\Rcal\subseteq( \Xcal \times \Acal \rightarrow [0,1])$, we can show that, with probability at least $1-\delta$,
\begin{align*}
    & \abr{\Lcal_{\rho}( \phi,w,V) - \Lcal_{\rho}(  \phi^*,\theta^*_V,V) - \rbr{ \Lcal_{\Dcal}( \phi,w,V) - \Lcal_{\Dcal}( \phi^*,\theta^*_V,V)}} \\
    \le{}& \frac{1}{2} \rbr{ \Lcal_{\rho}( \phi,w,V) - \Lcal_{\rho}( \phi^*,\theta^*_V,V)} + \frac{\cb K d^2\log^2(Kd)(B+L\sqrt d)^2\log\rbr{n|\Phi||\Phi'||\Rcal|/\delta}}{n}
\end{align*}
for all $\phi \in \Phi$, $\|w\|_2 \le B, V\in\Vcal \coloneqq \{\EE_{a'\sim \pi'(x')}[f(x',a')]: f\in\Fcal,\pi'\in\Pi'\}$, where $\Fcal \coloneqq \{f(x',a') = \mathrm{clip}_{[0,L]} (R(x',$ $a') + \langle\phi'(x',a'),\theta\rangle): \phi' \in \Phi', \|\theta\|_2 \le B, R \in \Rcal\}$, $\Pi'=\{\argmax_{a'}(f^{\textsc{ell}}(x',a')):f^{\textsc{ell}}\in\{\phi^{'\top} \Gamma^{-1} \phi':\phi'\in\Phi', \Gamma\in\RR^{d\times d}, \lambda_{\min} (\Gamma)\ge 1\}\}$, and $\cb$ is some universal constant.
\end{lemma}
\begin{proof}
The proof is almost the same as the proof of \pref{lem:fastrate_sqloss_fqe_1}. The only difference is that we have a different policy class $\Pi'$. It can be decomposed as $\Pi'=\bigcup_{\phi'_{\Pi'}\in\Phi'}\Pi'(\phi'_{\Pi'})$, where $\Pi'(\phi'_{\Pi'})=\{\argmax_{a'}(f^{\textsc{ell}}(x',a')):f^{\textsc{ell}}\in\{\phi^{'\top}_{\Pi'}\Gamma^{-1} \phi'_{\Pi'}:\Gamma\in\RR^{d\times d}, \lambda_{\min} (\Gamma)\ge 1\}\}$. We can show that $\Pi'(\phi'_{\Pi'})$ can be represented as the greedy policy of a linear class with dimension $d^2$
\begin{align*}
\Pi'(\phi'_{\Pi'})\subseteq \{\argmax_{a'}\langle\varphi'_{\phi'_{\Pi'}}(x',a'),w\rangle:w\in\RR^{d^2}\},
\end{align*}
where $\varphi'_{\phi'_{\Pi'}}$ is a $d^2$ dimension function and for any $i,j\in[d]$, $\varphi'_{\phi'_{\Pi'}}(x',a')[i+jd]=\phi'_{\Pi'}(x',a')[i]$ $\phi'_{\Pi'}(x',a')[j]$.

Therefore, we have $\Ndim(\Pi'(\phi'_{\Pi'}))\le c_3 d^2\log(d^2)$. Following the steps in \pref{lem:fastrate_sqloss_fqe_1}, it is easy to see that we only need to pay such additional $d$ factor in the final result.
\end{proof}
\begin{lemma}
\label{lem:fqi_fastrate_unclipped}
For a dataset $\Dcal \coloneqq \cbr{\rbr{x^{(i)}, a^{(i)}, x'^{(i)}}}_{i=1}^n \sim \rho$, finite feature classes $\Phi$ and $\Phi'$ and finite reward function class $\Rcal\subseteq(\Xcal \times \Acal \rightarrow [0,1])$, we can show that, with probability at least $1-\delta$
\begin{align*}
    & \abr{\Lcal_{\rho}( \phi,w,V) - \Lcal_{\rho}(  \phi^*,\theta^*_f,V) - \rbr{ \Lcal_{\Dcal}( \phi,w,V) - \Lcal_{\Dcal}( \phi^*,\theta^*_V,V)}} \\
    \le{}& \frac{1}{2} \rbr{ \Lcal_{\rho}( \phi,w,V) - \Lcal_{\rho}( \phi^*,\theta^*_V,V)} + \frac{\cc d(B+L\sqrt{d})^2 \log (n(B+L\sqrt d)|\Phi||\Phi'||\Rcal|/\delta)}{n}
\end{align*}
for all $\phi \in \Phi$, $\|w\|_2 \le B, V\in\Vcal \coloneqq \{\mathrm{clip}_{[0,L]} (\EE_{a'\sim \pi_f'(x')}[R(x',a') + \langle\phi'(x',a')],\theta\rangle): \phi' \in \Phi', \|\theta\|_2 \le B, R \in \Rcal\}$, where $\pi_f'$ is a policy determined by $f$ (e.g.., its induced greedy policy or the uniform policy), and $\cc$ is some universal constant.
\end{lemma}
\begin{proof}
This proof mostly follows from \pref{lem:fastrate_sqloss_fqe_1}. The main difference is that we do not establish the bound of the covering number through the pseudo dimension, while here we directly show it. We first similarly define the hypothesis class $\Hcal$ as
\begin{align*}
    \Hcal:=\cbr{(x,a,x')\rightarrow (g_1(x,a)-f(x',\pi'))^2 - (g_2(x,a)-f(x',\pi'))^2:
     f\in\Fcal,g_1\in\Gcal_1,g_2\in\Gcal_2},
\end{align*}
where $\Gcal_1=\{\langle\phi,w_1\rangle:\|w_1\|_2\le B,\phi\in\Phi\},$ $\Gcal_2=\{\langle\phi^*,w_2\rangle:\|w_2\|_2\le L\sqrt d\}$, and $\Fcal\coloneqq \{\mathrm{clip}_{[0,L]}(\EE_{a'\sim\pi'_f(x')}[R(x',$ $a') 
+\langle\phi'(x',a'),\theta\rangle]): \phi'\in\Phi', \|\theta\|_2 \le B,R\in\Rcal\}$.

From the standard covering result, we know that there exists an $\ell_2$ cover $\overline{\Wcal}_1$ of $\Wcal_1=\{w_1:\|w_1\|_2\le B\}$ with scale $\gamma$ of size $\rbr{\frac{2B}{\gamma}}^d$. Similarly we construct an $\ell_2$ cover of $\overline{\Wcal}_2=\{w_2:\|w_2\|_2\le L\sqrt d\}$ and an $\ell_2$ cover $\overline\Theta$ of $\Theta=\{\theta:\|\theta\|_2\le B\}$. They both have scale $\gamma$ and but with different sizes $\rbr{\frac{2L\sqrt d}{\gamma}}^d$ and $\rbr{\frac{2B}{\gamma}}^d$ respectively.

Then we define $\overline\Gcal_1=\{\langle\phi,w_1\rangle:w_1\in\overline{\Wcal}_1,\phi\in\Phi\},$ $\overline\Gcal_2=\{\langle\phi^*,w_2\rangle:w_2\in \overline{\Wcal}_2\}$, and $\overline\Fcal\coloneqq \{ \mathrm{clip}_{[0,L]} (\EE_{a'\sim\pi'_f(x')}[R(x',$ $a') 
+\langle\phi'(x',a'),\theta\rangle]): \phi'\in\Phi', \theta\in\overline\Theta,R\in\Rcal\}$. 
Now we define a function class
\begin{align*}
    \overline\Hcal:=\Big\{(x,a,x')\rightarrow (g_1(x,a)-f(x',\pi'))^2 - (g_2(x,a)-f(x',\pi'))^2:
     f\in\overline\Fcal,g_1\in\overline\Gcal_1,g_2\in\overline\Gcal_2\Big\}.
\end{align*}

It is easy to see that for any $h\in\Hcal$, there exists $h'\in\overline\Hcal$ such that $\|h-h'\|_\infty\le 16\gamma(B+L\sqrt d)$ and $|\overline \Hcal|=|\Rcal||\Phi||\Phi'|\rbr{\frac{2B}{\gamma}}^{2d}\rbr{\frac{2L \sqrt d}{\gamma}}^d$. From the definition of $\ell_1$ cover (\pref{def:covnum}), we can see that $\overline\Hcal$ is a $16\gamma(B+L\sqrt d)$ resolution $\ell_1$ cover of $\Hcal$. By setting $\gamma=\frac{\veps}{16(B+L\sqrt d)}$, we get that for any $\varepsilon,n$, 
\begin{align}
\Ncal_1(\varepsilon,\Hcal,n)\le |\Rcal||\Phi||\Phi'|\rbr{2B/\gamma}^{2d}\rbr{2L \sqrt d/\gamma}^d \le |\Rcal||\Phi||\Phi'|\rbr{32(B+L\sqrt d)^2/\veps}^{3d}.
\label{eq:l1_cover_H}
\end{align}

In \pref{corr:uni_bern_conf}, instead of bounding the covering number from the pseudo dimension, we now directly substitute \Cref{eq:l1_cover_H} into \Cref{eq:bernstein_n_eq}. Following the remaining steps in \pref{corr:uni_bern_conf}, with probability at least $1-\delta$, for any $h\in\Hcal$ we have
\begin{align*}
	\Big|\EE[h(x,a,x')]-\frac{1}{n}\sum_{i=1}^n h\Big(x^{(i)},a^{(i)},&x^{'(i)}\Big)\Big|\le\sqrt{\frac{\cc' d \VV[h(z)]\log\rbr{n(B+L\sqrt d)|R||\Phi||\Phi'|/\delta}}{n}}
	\\
	&+\frac{2\cc' d (B+L\sqrt d)^2\log\rbr{n(B+L\sqrt d)|R||\Phi||\Phi'|/\delta}}{n},
\end{align*}
where $\cc'$ is some universal constant.

Following remaining steps in \pref{lem:fastrate_sqloss_fqe_1} (without the union bound) completes the proof.
\end{proof}

\subsection{Generalized Elliptic Potential Lemma}
\label{app:gen_elliptic_pot}
The following lemma is adapted from Proposition 1 of \citet{alex2020elliptical}.
\begin{lemma}[Generalized elliptic potential lemma]
\label{lem:gen_elliptic_pot}
For any sequence of vectors $\theta^*_1, \theta^*_2,$ $\ldots, \theta^*_T \in \R^{d \times T}$ where $\|\theta^*_i\| \le L\sqrt{d}$, and any $\lambda \ge L^2d$, we have
\begin{align*}
    \sum_{t=1}^T \|\Sigma_t^{-1}\theta^*_{t+1}\|_2 \le 2\sqrt{\frac{dT}{\lambda}}.
\end{align*}
\end{lemma}
\begin{proof}
Proposition 1 from \citet{alex2020elliptical} shows that for any bounded sequence of vectors $\theta^*_1, \theta^*_2, \ldots,\theta^*_T \in \R^{d \times T}$, we have
\begin{align*}
    \sum_{t=1}^T \|\Sigma_t^{-1}\theta^*_{t}\|_2 \le \sqrt{\frac{dT}{\lambda}}.
\end{align*}
Now, we have
\begin{align*}
    \Sigma_{t} = \Sigma_{t-1} + \theta^*_t \theta^{*\top}_t \preceq \Sigma_{t-1} + \lambda I_{d\times d} \preceq 2\Sigma_{t-1},
\end{align*}
where we use the fact that $\|\theta^*_t \theta^{*\top}_t\|_2 \le L^2d \le \lambda$. Using this relation, we can show the property that for any vector $x\in \RR^d$, $4x^\top \Sigma_t^{-2} x \ge  x^\top \Sigma_{t-1}^{-2} x$.

First noticing the above p.s.d. dominance inequality, we have $I_{d\times d}/2 \preceq \Sigma_t^{-1/2}\Sigma_{t-1}\Sigma_t^{-1/2}$. Therefore, all the eigenvalues of $\Sigma_t^{-1/2}\Sigma_{t-1}\Sigma_t^{-1/2}$ (thus $\Sigma_t^{-1}\Sigma_{t-1}$) are no less than 1/2. Applying SVD decomposition, we can get all eigenvalues of matrix $\Sigma_{t-1}\Sigma_{t}^{-2}\Sigma_{t-1}=(\Sigma_t^{-1}\Sigma_{t-1})^\top(\Sigma_t^{-1}$ $\Sigma_{t-1})$ are no less than 1/4. Then consider any vector $y\in\RR^d$, we have $4y^\top \Sigma_{t-1}\Sigma_{t}^{-2}\Sigma_{t-1} y \ge y^\top y$. Let $x=\Sigma^{-1}_{y-1}y$, we get this property.

Applying the above result, we finally have
\begin{align*}
    \sum_{t=1}^T \|\Sigma_t^{-1}\theta^*_{t+1}\|_2 \le 2 \sum_{t=1}^T \|\Sigma_t^{-1}\theta^*_{t}\|_2 \le 2 \sqrt{\frac{dT}{\lambda}}.
\end{align*}
which completes the proof.
\end{proof}

\subsection{Covering Lemma for the Elliptical Reward Class}
In this part, we provide the statistical complexity of the  elliptical reward class. The result is used when we analyze the elliptical planner.
\begin{lemma}[Covering lemma for the elliptical reward class]
\label{lem:covering_reward_cls}
For any $h\in[H]$ and the elliptical reward class $\Rcal_h \coloneqq \{\phi_h^\top \Gamma^{-1} \phi_h:\phi_h\in\Phi_h, \Gamma\in\RR^{d\times d}, \lambda_{\min} (\Gamma)\ge 1\}$, there exists a $\gamma$-cover (in $\ell_\infty$ norm) $\Ccal_{\Rcal_h,\gamma}$ of size $|\Phi_h|(2\sqrt d/\gamma)^{d^2}$.

Moreover, for any $h\in[H]$, there exists a $\gamma$-cover (in $\ell_\infty$ norm) $\Ccal_{\Rcal_h^{\textsc{ell}},\gamma}$ of the reward class $\Rcal_h^{\textsc{ell}}\coloneqq \{R_{0:h} :R_{0:h-1}=\mathbf{0},R_{h} \in \Rcal_{h} \coloneqq \{\phi_{h}^\top \Gamma^{-1} \phi_{h}:\phi_{h}\in\Phi_{h}, \Gamma\in\RR^{d\times d}, \lambda_{\min} (\Gamma)\ge 1\}\}$ and $|\Ccal_{\Rcal_h^{\textsc{ell}},\gamma}|=|\Ccal_{\Rcal_{h},\gamma}|$.
\end{lemma}

\begin{proof}
Firstly, for any $\Gamma\in\RR^{d\times d}$ with $\lambda_{\min}(\Gamma)\ge 1$, applying matrix norm inequality yields $\|\Gamma\|_{\text{F}} \ge \sqrt d$. Further, we have $$\|\Gamma^{-1}\|_{\text{F}} \le \frac{\|I_{d\times d}\|_{\text{F}}} {\|\Gamma\|_{\text{F}}} \le \sqrt d.$$
Next, consider the matrix class $\bar A\coloneqq\{A\in\RR^{d\times d}:\|A\|_{\text{F}}\le \sqrt d\}$. From the definition of the Frobenius norm, for any $A\in\bar A$ and any  $(i,j)$-th element, we have $|A_{ij}|\le \sqrt d$. Applying the standard covering argument for each of the $d^2$ elements, there exists a $\gamma$-cover of $\bar A$, whose size is upper bounded by $(2\sqrt d / \gamma)^{d^2}$. Denote this $\gamma$-cover as $\bar A_\gamma$. For any $\Gamma\in\RR^{d\times d}$ with $\lambda_{\min}(\Gamma)\ge 1$, we can pick some $A\in\bar A_\gamma$ so that $\|\Gamma^{-1}-A_\Gamma\|_{\text F} \le \gamma$. Then for any $h\in[H], \phi_h\in\Phi_h, (x,a)\in\Xcal\times\Acal$, we have
\[|\phi_h(x,a)^\top \Gamma^{-1}\phi_h(x,a) - \phi_h(x,a)^\top A_\Gamma \phi_h(x,a)|\le \sup_{v:\|v\|_2\le 1}|v^\top (\Gamma^{-1} -  A_\Gamma )v|\le \|\Gamma^{-1}-A_\Gamma \|_{\text F} \le \gamma.\]
This implies that $\|\phi_h^\top \Gamma^{-1}\phi_h - \phi_h^\top A_\Gamma \phi_h\|_\infty\le \gamma$ and thus
$\Ccal_{\Rcal_h,\gamma}\coloneqq \{\phi_h^\top A\phi_h: \phi_h\in\Phi_h, A\in \bar A_\gamma\}$ is a $\gamma$-cover of $\Rcal_h$. 

Finally, for any $h\in[H]$, from the definition of $\Rcal_h^{\textsc{ell}}$ we know that at level $0,\ldots,h-1$ the reward is $\zero$ for any reward $R\in \Rcal_h^{\textsc{ell}}$. Therefore, $\Ccal_{\Rcal_h^{\textsc{ell}},\gamma}$ is directly a $\gamma$-cover of reward class $\Rcal_h^{\textsc{ell}}$ and $|\Ccal_{\Rcal_h^{\textsc{ell}},\gamma}|=|\Ccal_{\Rcal_{h},\gamma}|$. This completes the proof.
\end{proof}

\subsection{Probabilistic Tools}
\label{app:prob}
In this part, we abuse some notations (e.g., $\Zcal,\Gcal,\Pi,n,\varepsilon,f,g,h,z$) and they have the different meaning from other parts of the paper. 
\begin{definition}[$\ell_1$ covering number]
	\label{def:covnum}
	Given a hypothesis class $\Hcal \subseteq (\Zcal \to \RR)$, $\varepsilon > 0$, and $Z^n = (z_1,\dotsc,z_n) \in \Zcal^n$. We define the $\ell_1$ covering number $\Ncal_1(\varepsilon,\Hcal,Z^n)$ as the minimal cardinality of a set $\Ccal \subseteq \Hcal$, such that for any $h \in \Hcal$, there exists $h'\in \Hcal$ such that  $\frac{1}{n}\sum_{i = 1}^n |h(z_i) - h'(z_i)| \leq \varepsilon.$
	We also define $\Ncal_1(\varepsilon,\Hcal,n):= \max_{Z^n\in\Zcal^n}\Ncal_1(\varepsilon,\Hcal,Z^n)$.
\end{definition}

\begin{lemma}[Uniform deviation bound using covering number (Hoeffding's version), Theorem 29.1 of \cite{devroye2013probabilistic}]\label{lem:covering}
	Let $\Hcal\subseteq (\Zcal\rightarrow[0,b])$ be a hypothesis class and $Z^n=(z_1,\ldots,z_n)\in\Zcal^n$, where $z_i$ are i.i.d. samples drawn from some distribution $\PP(z)$ supported on $\Zcal$. Then for any $n$ and $\varepsilon>0$, we have
	\begin{equation*}
		\PP\sbr{\sup_{h\in \Hcal}\left|\frac{1}{n}\sum_{i=1}^{n}h(z_i)-\EE[h(z)]\right|>\varepsilon}\le 8\Ncal_1(\varepsilon/8,\Hcal,n)\exp\left(-\frac{n\varepsilon^2}{128b^2}\right).
	\end{equation*}
\end{lemma}

\begin{lemma}[An extension of the classical Bernstein's inequality, Lemma F.2 of \citet{dong2019sqrt}; Lemma 3.1 of \cite{massart1986rates}]
\label{lem:Bernstein_without_replacement}
For any $N\ge n \ge 1$, let $w$ be a uniformly random permutation over $1,\ldots,N$. For any $\xi\in \R^{N}$, we define 
$$\hat S_N=\sum_{i=1}^{N}\xi_i,\quad \hat S_{w,n}=\sum_{i=1}^{n}\xi_{w(i)},\quad \hat \sigma_N^2=\left(\frac{1}{N}\sum_{i=1}^{N}\xi_i^2\right)-\left(\frac{1}{N}\sum_{i=1}^{N}\xi_i\right)^2,$$ and $\hat U_N=\max_{1\le i\le N}\xi_i-\min_{1\le i\le N}\xi_i.$ Then for any $\varepsilon>0$, we have
\begin{equation*}
\PP\sbr{\left|\frac{\hat S_{w,n}}{n}-\frac{\hat S_N}{N}\right|>\varepsilon}\le 2\exp\left(-\frac{n\varepsilon^2}{2\hat \sigma_N^2+\varepsilon \hat U_N}\right).
\end{equation*}
\end{lemma}

\begin{lemma}[Uniform deviation bound using covering number (Bernstein's version), adapted from Lemma F.3 of \citet{dong2019sqrt}]
	\label{lem:uni_bern}
	Let $\Hcal\subseteq (\Zcal\rightarrow[0,b])$ be a hypothesis class and $Z^n=(z_1,\ldots,z_n)\in\Zcal^n$, where $z_i$ are i.i.d. samples drawn from some distribution $\PP(z)$ supported on $\Zcal$. Then for any $h\in\Hcal$, we have
	\begin{align*}
		\PP_{}\left[\abr{\EE[h(z)]-\frac{1}{n}\sum_{i=1}^n h(z_i)}>\varepsilon\right]
		\le&~\inf_{N\ge 2n+8b^2/\varepsilon^2}\left(32\Ncal_1(\varepsilon/32,\Hcal,N)\exp\rbr{-\frac{N\varepsilon^2}{2048b^2}}\right.\\
		&\left.\quad\quad+4\Ncal_1\rbr{\frac{\varepsilon n}{16N},\Hcal,N}\exp\rbr{-\frac{n\varepsilon^2}{128\VV[h(z)] +256\varepsilon b}}\right).
	\end{align*}
\end{lemma}
\begin{remark}
The major difference Lemma F.3 of \citet{dong2019sqrt} is that we have $\VV[h(z)]$ instead of its uniform upper bound on RHS.
\end{remark}
\begin{proof}
	The proof mostly follows from Lemma F.3 of \citet{dong2019sqrt}. Firstly, we define similar notations for the empirical sums. Let $n'=N-n$ and $Z^N=(z_1,\cdots,z_N)$ be $N$ i.i.d.~random samples. For any $h\in \Hcal$, we define the following: 
	\[
	\hat{\EE}_{n}[h(z)] := \frac{1}{n}\sum_{i=1}^{n}h(z_i),\quad \hat{\EE}_{n'}[h(z)]:= \frac{1}{n'}\sum_{i=n+1}^{N}h(z_i),\quad \hat{\EE}_{N}[h(z)]:= \frac{1}{N}\sum_{i=1}^{N}h(z_i).
	\]
	
	In addition, let $w$ be a random permutation over $1,\ldots,N$, which is independent of the choice of $(z_1,\cdots,z_{N}).$ The empirical sums of the permutation are defined as
	$$\hat{\EE}_{w,n}[h(z)]:= \frac{1}{n}\sum_{i=1}^{n}h(z_{w(i)}),\quad \hat{\EE}_{w,n'}[h(z)]:= \frac{1}{n'}\sum_{i=n+1}^{N}h(z_{w(i)}).$$ 
	
	Following the first two steps of the proof in \citet{dong2019sqrt}, for $N\ge 2n+8b^2/\varepsilon^2$ and any $h\in\Hcal$, we can easily get that
	\begin{align}
	\label{eq:cover_5}
		\PP\left[\left|\hat{\EE}_n[h(z)]-\EE [h(z)]\right|> \varepsilon\right]\le&~ 2\PP\left[\left|\hat{\EE}_n[h(z)]-\hat{\EE}_{n'}[h(z)]\right|> \varepsilon/2\right]\notag\\
		\le&~ 2\PP\left[\left|\hat{\EE}_{w,n}[h(z)]-\hat{\EE}_{N}[h(z)]\right|>\varepsilon/4\right].
	\end{align}
	The main difference in our proof is that we move $\sup_{h\in\Hcal}$ from inside $\PP[\cdot]$ to the outside (and change it to the argument ``for any $h\in\Hcal$''), and it can be easily verified.
	
	For any fixed $Z^N=(z_1,\cdots,z_N)\in\Zcal^N$, we use $\Ccal=\{h'_1,\cdots,h'_{|\Ccal|}\}$ to denote the minimal $(n\varepsilon/16N)$-cover over $\Hcal\rvert_{Z^N}$. Then we have $|\Ccal|\le \Ncal_1(n\varepsilon/16N,\Hcal,N),$ and there exists a mapping $\alpha:\Hcal\to \{1,\ldots,|\Ccal|\}$ such that, 
	\begin{align}
	\label{eq:cover_1}
    \frac{1}{N}\sum_{i=1}^{N}\abr{h(z_i)-h'_{\alpha(h)}(z_i)}\le n\varepsilon/16N,\quad\forall h\in \Hcal.
	\end{align}

	Following the analysis in \citet{dong2019sqrt}, we have that for any $h\in\Hcal$,
	\begin{align*}
		\left|\hat{\EE}_{w,n}[h(z)]-\hat{\EE}_{N}[h(z)]\right|\le\left|\hat{\EE}_{w,n}[h'_{\alpha(h)}(z)]-\hat{\EE}_{N}[h'_{\alpha(h)}(z)]\right|+\varepsilon/8.
	\end{align*}
	This implies that for fixed $Z^N$ and any $h\in\Hcal$, we have 
	\begin{align}
	\label{eq:cover_2}
	\PP\hspace{-.2em}\sbr{\left|\hat{\EE}_{w,n}[h(z)]-\hat{\EE}_{N}[h(z)]\right|\hspace{-.2em}>\hspace{-.2em}\varepsilon/4\hspace{-.2em}\mid\hspace{-.2em} Z^N}\hspace{-.2em}\le \PP\hspace{-.2em}\sbr{\left|\hat{\EE}_{w,n}[h'_{\alpha(h)}(z)]-\hat{\EE}_N[h'_{\alpha(h)}(z)]\right|\hspace{-.2em}>\hspace{-.2em}\varepsilon/8\hspace{-.2em}\mid\hspace{-.2em} Z^N}.
	\end{align}

    Now we define the empirical variance as
	\[
	\hat{\VV}_{N}[h(z)]:= \frac{1}{N}\sum_{i=1}^{N}h(z_i)^2-\left(\frac{1}{N}\sum_{i=1}^{N}h(z_i)\right)^2.
	\]
	Applying \pref{lem:Bernstein_without_replacement} and union bounding over $\Ccal$ yields that for any $h'_i\in\Ccal,i\in \{1,2,\ldots,|\Ccal|\}$, 
	\begin{align}
	\label{eq:cover_3}
		\PP\sbr{\left|\hat{\EE}_{w,n}[h'_i(z)]-\hat{\EE}_N[h'_i(z)]\right|> \varepsilon/8\mid Z^N}&\le~ 2|\Ccal|\exp\left(-\frac{n\varepsilon^2/64}{2\hat{\VV}_N[h'_i(z)]+\varepsilon b/8}\right).
	\end{align}
	
	Note that for $h$ and $h'_{\alpha(h)}$, we have
	\begin{align}
	\label{eq:cover_4}
	   &~\abr{\hat\VV_N[h(z)]-\hat\VV_N[h'_{\alpha(h)}(z)]}\notag\\
	   =&~\frac{1}{N}\sum_{i=1}^{N}\rbr{\rbr{h(z_i)-h'_{\alpha(h)}(z_i)}\rbr{h(z_i)+h'_{\alpha(h)}(z_i)}}\notag\\
	   &~\quad-\left(\frac{1}{N}\sum_{i=1}^{N}h(z_i)-\frac{1}{N}\sum_{i=1}^{N}h'_{\alpha(h)}(z_i)\right)\left(\frac{1}{N}\sum_{i=1}^{N}h(z_i)+\frac{1}{N}\sum_{i=1}^{N}h'_{\alpha(h)}(z_i)\right)\notag\\
	   \le&~2b\frac{1}{N}\sum_{i=1}^{N}\abr{h(z_i)-h'_{\alpha(h)}(z_i)}+2b\frac{1}{N}\sum_{i=1}^{N}\abr{h(z_i)-h'_{\alpha(h)}(z_i)}.
	\end{align}
	
    Combining \Cref{eq:cover_1,eq:cover_2,eq:cover_3,eq:cover_4}, for any fixed $Z^N$ and any $h\in\Hcal$, we have
    \begin{align}
    \label{eq:cover_9}
        \PP\sbr{\left|\hat{\EE}_{w,n}[h(z)]-\hat{\EE}_{N}[h(z)]\right|>\varepsilon/4\mid Z^N}\le2|\Ccal|\exp\left(-\frac{n\varepsilon^2/64}{2\hat{\VV}_N[h(z)]+\varepsilon b+n\varepsilon b/(2N)}\right).
    \end{align}
	
	For the empirical variance and population variance, we note that
	\begin{align}
	\label{eq:cover_7}
	&~\abr{\hat \VV_N[h(z)]-\VV[h(z)]}\notag\\
	=&~\abr{\rbr{\frac{1}{N}\sum_{i=1}^{N}h(z_i)^2-\EE[h(z)^2]}-\rbr{\left(\frac{1}{N}\sum_{i=1}^{N}h(z_i)\right)^2-(\EE[h(z)])^2}}\notag\\
	\le&~\abr{\frac{1}{N}\sum_{i=1}^{N}h(z_i)^2-\EE[h(z)^2]}+\abr{\rbr{\frac{1}{N}\sum_{i=1}^{N}h(z_i)-\EE[h(z)]}\rbr{\frac{1}{N}\sum_{i=1}^{N}h(z_i)+\EE[h(z)]}}\notag\\
    \le&~\abr{\frac{1}{N}\sum_{i=1}^{N}h(z_i)^2-\EE[h(z)^2]}+2b\abr{\frac{1}{N}\sum_{i=1}^{N}h(z_i)-\EE[h(z)]}.
	\end{align}
	Consequently,
    \begin{align}
    \label{eq:cover_8}
        &~\PP\sbr{\sup_{h\in \Hcal}\left|\hat \VV_N[h(z)]-\VV[h(z)]\right|>\varepsilon b}\notag\\
        \le&~   \PP\sbr{\sup_{h\in \Hcal}\left|\frac{1}{N}\sum_{i=1}^{N}h(z_i)^2-\EE[h(z)^2]\right|>\varepsilon b/2} +  \PP\sbr{\sup_{h\in \Hcal}\abr{\frac{1}{N}\sum_{i=1}^{N}h(z_i)-\EE[h(z)]}>\varepsilon /4}\notag\\
        \le&~8\Ncal_1\left(\varepsilon b/16,\Hcal^2,N\right)\exp\left(-\frac{N \varepsilon^2 }{512b^2}\right)+8\Ncal_1\left(\varepsilon /32,\Hcal,N\right)\exp\left(-\frac{N \varepsilon^2 }{2048b^2}\right)\notag
        \\
        \le&~16\Ncal_1\left(\varepsilon /32,\Hcal,N\right)\exp\left(-\frac{N \varepsilon^2 }{2048b^2}\right).
    \end{align}
Here, we define $\Hcal^2=\{h^2: h\in \Hcal\}$. The first inequality is due to \Cref{eq:cover_7}. The second inequality is due to \pref{lem:covering}. Notice that for any $h_1,h_2\in \Hcal,z\in \Zcal$, we have $|h_1(z)^2-h_2(z)^2|=|(h_1(z)-h_2(z))(h_1(z)+h_2(z))|\le 2b\left|h_1(z)-h_2(z)\right|$. This implies that $\Ncal_1\left(\varepsilon b/16,\Hcal^2,N\right)\le \Ncal_1\left(\varepsilon /32,\Hcal,N\right)$. Then it is easy to see the third inequality holds. 

Combining \Cref{eq:cover_9,eq:cover_8}, for $N\ge 2n+8b^2/\varepsilon^2$, any fixed $Z^N$ and any $h\in\Hcal$, we get that 
    \begin{align}
    \label{eq:cover_6}
        &~\PP\sbr{\left|\hat{\EE}_{w,n}[h(z)]-\hat{\EE}_{N}[h(z)]\right|>\varepsilon/4}\notag
        \\
        \le&~2|\Ccal|\exp\left(-\frac{n\varepsilon^2/64}{2\VV[h(z)]+4\varepsilon b}\right)+16\Ncal_1\left(\varepsilon /32,\Hcal,N\right)\exp\left(-\frac{N \varepsilon^2 }{2048b^2}\right).
    \end{align}

Finally, combining \Cref{eq:cover_5,eq:cover_6} and noticing the range of $N$ and the upper bound of $|\Ccal|$ completes the proof.
\end{proof}
\begin{corollary}[Uniform deviation bound using covering number (Bernstein's version, tail bound), adapted from Lemma F.4 of \citet{dong2019sqrt}]
\label{corr:uni_bern_tail}
For $b\ge 1$, let $\Hcal\subseteq (\Zcal\rightarrow[0,b])$ be a hypothesis class and $Z^n=(z_1,\ldots,z_n)\in\Zcal^n$, where $z_i$ are i.i.d. samples drawn from some distribution $\PP(z)$ supported on $\Zcal$. Then for any $h\in\Hcal$, we have
\begin{align*}
	\PP	\left[\abr{\EE[h(z)]-\frac{1}{n}\sum_{i=1}^n h(z_i)}>\varepsilon\right]
	\le 36\Ncal_1\rbr{\frac{\varepsilon^3 }{160b^2},\Hcal,\frac{10nb^2}{\varepsilon^2}}\exp\rbr{-\frac{n\varepsilon^2}{128\VV[h(z)] +256\varepsilon b}}	.
\end{align*}
\end{corollary}

\begin{proof}
In \pref{lem:uni_bern}, we can set $N=10nb^2/\varepsilon^2\ge 2n+8b^2/\varepsilon^2$, which indicates $\frac{N\varepsilon^2}{2048b^2}\ge\frac{n\varepsilon^2}{256\varepsilon b}\ge\frac{n\varepsilon^2}{128\VV[h(z)] +256\varepsilon b}$ and $\varepsilon/32\ge\varepsilon n/(16 N)$. Then noticing the monotonicity of covering number $\Ncal_1(\cdot,\Hcal,N)$ and $\exp(\cdot)$, we complete the proof.
\end{proof}

\begin{corollary}[Uniform deviation bound using covering number (Bernstein's version, confidence interval bound)]
	\label{corr:uni_bern_conf}
	For $b\ge 1$, let $\Hcal\subseteq (\Zcal\rightarrow[-b,b])$ be a hypothesis class with $\Pdim(\Hcal) \leq d_\Hcal$ and $Z^n=(z_1,\ldots,z_n)$ be i.i.d. samples drawn from some distribution $\PP(z)$ supported on $\Zcal$. Then with probability at least $1-\delta$, we have that for any $h\in\Hcal$,
	\begin{align*}
		\abr{\EE[h(z)]-\frac{1}{n}\sum_{i=1}^n h(z_i)} \le\sqrt{\frac{384d_\Hcal\VV[h(z)]\log\rbr{n/\delta}}{n}}+\frac{768d_\Hcal b\log\rbr{n/\delta}}{n}.
	\end{align*}
\end{corollary}
\begin{remark}
The slight difference from the standard Bernstein's inequality is that we have $\log (n)$ term on the numerator. It is mostly due to the $\veps$ dependence in the covering number.
\end{remark}
\begin{proof}
	Let $\Hcal' = \{(h(\cdot) + b: h \in \Hcal \}\subseteq (\Zcal \to [0, 2b])$. We know that shifting only changes the range of the function, but does not change the variance of the function. Therefore, applying \pref{corr:uni_bern_tail} gives us that for any $h'\in\Hcal'$,
\begin{align*}
	\PP	\left[\abr{\EE[h'(z)]-\frac{1}{n}\sum_{i=1}^n h'(z_i)} >\varepsilon\right]\hspace{-.15em}
	\le 36\Ncal_1\rbr{\frac{\varepsilon^3 }{640b^2},\Hcal',\frac{40nb^2}{\varepsilon^2}}\exp\rbr{-\frac{n\varepsilon^2}{128\VV[h(z)] +512\varepsilon b}} \hspace{-.15em}	.
\end{align*}
	
	From \pref{def:covnum}, we know that $\Hcal$ and $\Hcal'$ have the same covering number, i.e., $\Ncal_1(\varepsilon',\Hcal',m)=\Ncal_1(\varepsilon',\Hcal,m)$ for any $\varepsilon'\in\RR_+$ and $m\in\NN_+$. In addition, we have for any $h'\in\Hcal'$, we have
	$$\abr{\frac{1}{n}\sum_{i=1}^n h'(z_i)-\EE[h'(z)]}=\abr{\frac{1}{n}\sum_{i=1}^n h(z_i)+b-\EE[h(z)+b]}=\abr{\frac{1}{n}\sum_{i=1}^n h(z_i)-\EE[h(z)]}.$$
	This implies that for any $h\in\Hcal$,
	\begin{align}
	\label{eq:bernstein}
	\PP	\sbr{\abr{\EE[h(z)]\hspace{-.15em}-\hspace{-.15em}\frac{1}{n}\sum_{i=1}^n h(z_i)}\hspace{-.15em}>\hspace{-.15em}\varepsilon\hspace{-.15em}}\hspace{-.15em}\le 36\Ncal_1\hspace{-.15em}\rbr{\hspace{-.15em}\frac{\varepsilon^3 }{640b^2},\Hcal,\frac{40nb^2}{\varepsilon^2}\hspace{-.15em}}\exp\rbr{\hspace{-.25em}\frac{-n\varepsilon^2}{128\VV[h(z)] +512\varepsilon b}\hspace{-.25em}}	.
\end{align}

	Setting RHS of \Cref{eq:bernstein} to be $\delta$, we get 
	\begin{align}
	\label{eq:bernstein_n_eq}
	n=&~\frac{(128\VV[h(z)] +512\varepsilon b)\log\rbr{36\Ncal_1\rbr{\frac{\varepsilon^3 }{640b^2},\Hcal,\frac{40nb^2}{\varepsilon^2}}/\delta}}{\varepsilon^2}.
	\end{align}
	This implies the following inequality
	\[
	\varepsilon \le \sqrt{\frac{128\VV[h(z)]\log\rbr{36\Ncal_1\rbr{\frac{\varepsilon^3 }{640b^2},\Hcal,\frac{40nb^2}{\varepsilon^2}}/\delta}}{n}}+\frac{512b\log\rbr{36\Ncal_1\rbr{\frac{\varepsilon^3 }{640b^2},\Hcal,\frac{40nb^2}{\varepsilon^2}}/\delta}}{n},
	\]
	which can be verified by substituting $n$ in \Cref{eq:bernstein_n_eq}.
	
    Applying \pref{corr:l1cvnb} and simplifying the expression yields
    \begin{align*}
	\log\rbr{36\Ncal_1\rbr{\frac{\varepsilon^3 }{640b^2},\Hcal,\frac{72nb^2}{\varepsilon^2}}/\delta}\le&3d_\Hcal\log\rbr{\frac{54b}{\varepsilon\delta}}.
	\end{align*}
	
	From \Cref{eq:bernstein_n_eq}, we also have $n\ge \frac{512b}{\varepsilon}$. Consequently, we have
	\begin{align*}
   	\varepsilon \le \sqrt{\frac{384d_\Hcal\VV[h(z)]\log\rbr{n/\delta}}{n}}+\frac{768d_\Hcal b\log\rbr{n/\delta}}{n}.
	\end{align*}
	Plugging this into \Cref{eq:bernstein} completes the proof.
\end{proof}
\begin{lemma}[Bounding covering number by pseudo dimension \citep{haussler1995sphere}]
	\label{lem:l1cvnb}
	Given a hypothesis class $\Hcal \subseteq (\Zcal \to [0,1])$ with $\Pdim(\Hcal) \leq d_\Hcal$, we have for any $Z^n \in \Zcal^n$,
	\[
	\Ncal_1(\varepsilon,\Hcal,Z^n) \leq e \left(d_\Hcal + 1\right) \left( \frac{2e}{\varepsilon} \right)^{d_\Hcal}.\]
\end{lemma}

\begin{corollary}[Bounding covering number by pseudo dimension]
	\label{corr:l1cvnb}
	Given a hypothesis class $\Hcal \subseteq (\Zcal \to [a,b])$ with $\Pdim(\Hcal) \leq d_\Hcal$, for any $Z^n \in \Zcal^n$, we have
	\[
	\Ncal_1(\varepsilon,\Hcal,Z^n) \leq e \left(d_\Hcal + 1\right) \left( \frac{2e(b-a)}{\varepsilon} \right)^{d_\Hcal} \quad  \text{ and } \quad
	\Ncal_1(\varepsilon,\Hcal,n) \leq  \left( \frac{4e^2(b-a)}{\varepsilon} \right)^{d_\Hcal}.
	\]
\end{corollary}
\begin{proof}
	Let $\Hcal' = \{(h(\cdot) - a)/(b - a): h \in \Hcal \}\subseteq (\Zcal \to [0, 1])$. From the definition of pseudo dimension, it is easy to see $d_{\Hcal'} = d_{\Hcal}$. Noticing \pref{def:covnum} and applying \pref{lem:l1cvnb}, we get
	\[
	\Ncal_1(\varepsilon,\Hcal,Z^n) = \Ncal_1(\varepsilon / (b - a),\Hcal',Z^n) = \left(d_{\Hcal} + 1\right) \left( \frac{2e (b - a)}{\varepsilon} \right)^{d_{\Hcal}}.\] 
	Noticing \pref{def:covnum} and following simple algebra, we can show the second part.
\end{proof}
\begin{definition}[VC-dimension]
For hypothesis class $\Hcal \subseteq (\Xcal \to \{0,1\})$, we define its VC-dimension $\textrm{VC-dim}(\Hcal)$ as the maximal cardinality of a set $X = \{x_1, \ldots, x_{|X|}\} \subseteq \Xcal$ that satisfies $|\Hcal_{X}| = 2^{|X|}$ (or $X$ is \emph{shattered} by $\Hcal$), where $\Hcal_{X}$ is the restriction of $\Hcal$ to $X$, i.e., $\{(h(x_1), \ldots, h(x_{|X|})): h\in\Hcal \}$. 
\end{definition}

\begin{definition}[Pseudo dimension \citep{haussler2018decision}]\label{def:pdim}
For hypothesis class $\Hcal \subseteq (\Xcal \to \RR)$, we define its pseudo dimension $\Pdim(\Hcal)$ as $\Pdim(\Hcal) = \VCdim(\Hcal^+)$, where $\Hcal^+ = \{(x, \xi) \mapsto \one[h(x) > \xi] : h \in \Hcal\} \subseteq (\Xcal \times \mathbb{R} \to \{0, 1\})$.
\end{definition}

\begin{lemma}[Sauer's lemma]
\label{lem:sauerlemma}
For the hypothesis class $\Hcal \subseteq (\Xcal \to \{0,1\})$ with $\VCdim(\Hcal) $ $= d_{\mathrm{VC}} < \infty$ and any $X = (x_1,x_2,\dotsc,x_n) \in \Xcal^n$, we have
\begin{align}
|\Hcal_X| \leq \left(n+1\right)^{d_\mathrm{VC}},
\end{align}
where $\Hcal_X \coloneqq \{(h(x_1), h(x_2), \dotsc, h(x_n)):h \in \Hcal\}$ is the restriction of $\Hcal$ to $X$.
\end{lemma}
\begin{lemma}
	\label{lem:pseudodim_com}
	Let $\Zcal \coloneqq \Xcal \times \Acal$ with $|\Acal| = K$. Let $\Pi \subseteq (\Xcal \to \Acal)$ be a policy class with Natarajan dimension $\Ndim(\Pi) = d_\Pi\ge 6$, $\Fcal \subseteq (\Zcal \to [0,L])$ with pseudo dimension $\Pdim(\Fcal) = d_{\Fcal}\ge 6$, and $\Gcal_1,\Gcal_2 \subseteq (\Xcal \to [0,L])$ with pseudo dimension $\Pdim(\Gcal_1) = d_{\Gcal_1} \ge 6$ and $\Pdim(\Gcal_2) = d_{\Gcal_2} \ge 6$. Then we have the following:
	\begin{enumerate}
		\item The hypothesis class $\Hcal_1 = \{x \to g_1(x)g_2(x): g_1\in\Gcal_1,g_2 \in \Gcal_2\}$ has pseudo dimension $\Pdim(\Hcal_1) \leq 32 (d_{\Gcal_1} \log(d_{\Gcal_1})+d_{\Gcal_2} \log(d_{\Gcal_2}))$.
		\item The hypothesis class $\Hcal_2 = \{x \to g_1(x) + g_2(x): g_1,\in\Gcal_1,g_2 \in \Gcal_2\}$ has pseudo dimension $\Pdim(\Hcal_2) \leq 32 (d_{\Gcal_1} \log(d_{\Gcal_1})+d_{\Gcal_2} \log(d_{\Gcal_2}))$.
		\item The hypothesis class $\Hcal_3 = \{(x,a) \to  f(x,a)\one[a = \pi(x)]: f \in \Fcal, \pi \in \Pi\}$ has pseudo dimension $\Pdim(\Hcal_3) \leq 6(d_\Pi + d_\Fcal) \log(2eK(d_\Pi + d_\Fcal))$.
	\end{enumerate}
\end{lemma}
\begin{proof} 
Firstly, w.l.o.g. we assume that $L=1$ since in the pseudo dimension we can just scale all $\xi$ in \pref{def:pdim} by $1/L$.

\textbf{Part 1.} Let $\Hcal_1^+ \coloneqq \{(x,\zeta) \to \one[g_1(x)g_2(x) > \zeta]: g_1\in\Gcal_1,g_2 \in \Gcal_2\} \subseteq (\Xcal \times \RR \to \{0,1\})$. From the fact that $\Pdim(\Hcal)=\VCdim(\Hcal_1^+)$, it suffices to prove that $\log|\Hcal_{1,X}^+| \le n$ for any $X = ((x_1,\zeta_1),(x_2,\zeta_2),\dotsc,(x_n,\zeta_n)) \in (\Xcal \times \RR)^n$, where $n = 32 (d_{\Gcal_1} \log(d_{\Gcal_1})+d_{\Gcal_2} \log(d_{\Gcal_2}))$ and $\Hcal_{1,X}^+$ refers to the restriction of $\Hcal_1^+$ to $X$.

Similarly we define $\Gcal_{1}^+ \coloneqq \{(x,\xi) \to \one[g_1(x) > \xi]: g_1 \in \Gcal_1\} \subseteq (\Xcal \times \RR \to \{0,1\})$ and $\Gcal_{2}^+ \coloneqq \{(x,\xi') \to \one[g_2(x) > \xi']: g_2 \in \Gcal_2\} \subseteq (\Xcal \times \RR \to \{0,1\})$. 
For $X' = ((x_1,\xi_1),(x_2,\xi_2),\dotsc,(x_n,\xi_n))$ $\in (\Xcal \times \RR)^n,X'' =  ((x_1,\xi_1'),(x_2,\xi_2'),\dotsc,(x_n,\xi_n')) \in (\Xcal \times \RR)^n$, we use $\Gcal_{1,X'}^+$ to denote the restriction of $\Gcal_{1}^+$ to $X'$ and use $\Gcal_{2,X''}^+$ to denote the restriction of $\Gcal_{2}^+$ to $X''$. Since $\one[g_1(x)g_2(x) > \zeta]$ can be decomposed as $\one[g_1(x)> \xi]\one[g_2(x) > \zeta/\xi]$ for some $\xi\in\RR$, by setting $\xi'=\zeta/\xi$ we can see that 
$\left|\Hcal_{1,X}^+\right| \le \left|\Gcal_{1,X'}^+\right|\left|\Gcal_{2,X''}^+\right|.$

Applying Lemma~\ref{lem:sauerlemma}, we have
$\left|\Gcal_{1,X'}^+\right| \leq \left(n + 1\right)^{d_{\Gcal_1}}, \left|\Gcal_{2,X''}^+\right| \leq \left(n + 1\right)^{d_{\Gcal_2}}.$
Therefore,
\begin{align*}
\log\left|\Hcal_{1,X}^+\right| \leq &~  d_{\Gcal_1}\log(n+1)+d_{\Gcal_2}\log(n+1)
\\
\leq &~ (d_{\Gcal_1}+d_{\Gcal_2})\log(32 (d_{\Gcal_1} \log(d_{\Gcal_1})+d_{\Gcal_2} \log(d_{\Gcal_2}))+1)\\
\leq &~ (d_{\Gcal_1}+d_{\Gcal_2})\log(64 (d_{\Gcal_1}^2 +d_{\Gcal_2}^2))
\leq 2(d_{\Gcal_1}+d_{\Gcal_2})\log(64 (d_{\Gcal_1} +d_{\Gcal_2}))
\\
\le&~ 32 (d_{\Gcal_1} \log(d_{\Gcal_1}) + d_{\Gcal_2} \log(d_{\Gcal_2}))=n.
\end{align*}

\textbf{Part 2.} Note that we have the decomposition $\one[g_1(x)+g_2(x) > \zeta] = \one[g_1(x) > \xi] + \one[g_2(x) > \zeta - \xi]$. The remaining steps can be similarly followed from Part 1.

\textbf{Part 3.} We use $\Hcal_{3,X}^+$ to denote the restriction of $\Hcal_3^+\coloneqq\{(x,a,\zeta) \to \one[ f(x,a)\one[a = \pi(x)]]\ge \zeta: f \in \Fcal, \pi \in \Pi\}$ to $X=((x_1,a_1,\zeta_1),(x_2,a_2,\zeta_2),\ldots,(x_n,a_n,\zeta_n))$. It suffices to show an upper bound of $|\Hcal_{3,X}^+|$. 

For $h\in\Hcal_3^+$, we can decompose it as $\one[a=\pi(x)] \one[f(x,a)\ge \zeta]$. Following the proof of Lemma 21 in \citet{jiang2017contextual}, we can also get $|\Hcal_{3,X}^+|\le |\Pi_X||\Fcal_X^+|$, where $\Pi_X$ denotes the restriction of $\Pi$ to $(x_1,x_2,\ldots,x_n)$ and $\Fcal_X^+$ denotes the restriction of $\Fcal^+$ to $((x_1,a_1,\zeta_1),(x_2,a_2,\zeta_2),\ldots,(x_n,a_n,\zeta_n))$. Notice that here $\Hcal_{3,X}^+$ can not be produced by the Cartesian product of $\Pi_X$ and $\Fcal_X^+$, but the upper bound still holds. The remaining steps similarly follow from \citet{jiang2017contextual}.
\end{proof}

\end{document}